\documentclass[sts]{imsart}

\RequirePackage{amsthm,amsmath,amsfonts,amssymb,booktabs}
\RequirePackage[numbers,compress]{natbib}
\RequirePackage[colorlinks,citecolor=blue,urlcolor=blue]{hyperref}
\usepackage{multirow}
\usepackage{array} 

\RequirePackage{graphicx}
\usepackage[capitalise,nameinlink]{cleveref}  %

\startlocaldefs
\theoremstyle{plain}
\newtheorem{theorem}{Theorem}
\newtheorem*{theorem*}{Theorem}
\newtheorem{remark}{Remark}
\newtheorem{assumption}{Assumption}
\newtheorem*{assumption*}{Assumption}
\newtheorem{lemma}{Lemma}
\newtheorem*{lemma*}{Lemma}

\newtheorem{definition}{definition}
\newtheorem*{definition*}{Definition}

\Crefname{theorem}{Thm.}{Thms.}
\Crefname{algorithm}{Alg.}{Algs.}
\Crefname{section}{Sec.}{Secs.}
\Crefname{equation}{Eq.}{Eqs.}
\Crefname{assumption}{Asm.}{Asms.}
\Crefname{lemma}{Lem.}{Lems.}
\Crefname{remark}{Remark}{Remarks}
\Crefname{proposition}{Prop.}{Props.}
\Crefname{definition}{Def.}{Defs.}

\endlocaldefs

\usepackage{shortcuts}

\long\def\edit#1{{#1}}
\newcommand{\EluDim}{\op{EluDim}}

\let\cref\Cref

\begin{document}

\begin{frontmatter}
\title{The Central Role of the Loss Function in Reinforcement Learning}
\runtitle{Central Role of Loss Functions in RL}

\begin{aug}
\author[A]{\fnms{Kaiwen}~\snm{Wang}\ead[label=e1]{kw437@cornell.edu}},
\author[B]{\fnms{Nathan}~\snm{Kallus}\ead[label=e2]{kallus@cornell.edu}}
\and
\author[C]{\fnms{Wen}~\snm{Sun}\ead[label=e3]{ws455@cornell.edu}\ead[label=u1,url]{www.foo.com}}

\address[A]{The authors are from Cornell University. Correspondence to Kaiwen Wang (\href{https://kaiwenw.github.io/}{\nolinkurl{https://kaiwenw.github.io}}).}

\end{aug}

\begin{abstract}
This paper illustrates the central role of loss functions in data-driven decision making, providing a comprehensive survey on their influence in cost-sensitive classification (CSC) and reinforcement learning (RL). We demonstrate how different regression loss functions affect the sample efficiency and adaptivity of value-based decision making algorithms. Across multiple settings, we prove that algorithms using the binary cross-entropy loss achieve first-order bounds scaling with the optimal policy's cost and are much more efficient than the commonly used squared loss. Moreover, we prove that distributional algorithms using the maximum likelihood loss achieve second-order bounds scaling with the policy variance and are even sharper than first-order bounds. This in particular proves the benefits of distributional RL. We hope that this paper serves as a guide analyzing decision making algorithms with varying loss functions, and can inspire the reader to seek out better loss functions to improve any decision making algorithm.
\end{abstract}


\end{frontmatter}

\section{Introduction}

The value-based approach to reinforcement learning (RL) reduces the decision making problem to regression: first predict the expected rewards to go under the optimal policy, given state and action, and then one can simply choose the action that maximizes the prediction at every state. This regression, called $Q$-learning \citep{watkins1992q}, combined with recent decades' advances in deep learning, plays a central role in the empirical successes of deep RL.
Groundbreaking examples are DeepMind's use of deep $Q$-networks to play Atari with no feature engineering \citep{mnih2015human} and OpenAI's use of deep reward models to align large language models with human preferences via RL fine-tuning \citep{ouyang2022training}.

In prediction, we often say a good model is one with low mean-squared error out of sample. Correspondingly, regression is usually done by minimizing the average squared loss between predictions and targets in the training data. However, low mean-squared error may translate loosely to high-quality, downstream decision making. Thus, the natural question arises: is squared loss always the best choice for learning $Q$-functions?

\begin{table}
\caption{The decision-making regret of value-based RL per loss function, where $n$ is the number of samples. In this paper, we'll see that squared loss cannot adapt to small-cost or small-variance settings while binary-cross-entropy (bce) and maximum likelihood estimation (mle) can. \edit{We remark that mle is used with distributional regression and thus requires stronger realizability or completeness conditions. Moreover, all three losses use slightly different eluder dimensions in the online setting. We summarize these nuances in \cref{table:summary-of-rl-results}.}}
\label{table:rates}
\begin{tabular}{@{}cccc@{}}%
\toprule
Loss $\backslash$ Setting & Worst-case & Small cost & Small variance \\ \midrule
$\ell_{\op{sq}}$ & $\Theta(1/\sqrt{n})$ & $\Theta(1/\sqrt{n})$ & $\Theta(1/\sqrt{n})$ \\ \midrule
$\ell_{\op{bce}}$ & $\Theta(1/\sqrt{n})$ & $\color{red}{\Ocal(1/n)}$ & $\Theta(1/\sqrt{n})$ \\ \midrule
$\ell_{\op{mle}}$ & $\Theta(1/\sqrt{n})$ & $\color{red}{\Ocal(1/n)}$ & $\color{red}{\Ocal(1/n)}$ \\ \bottomrule
\end{tabular}
\end{table}

In this survey article, we highlight that the answer to this question is a resounding ``\emph{no}!'' 
\edit{We focus on the theoretical question: when and how do alternative loss functions attain better guarantees for decision making?
We start with the simple setting of cost-sensitive classification (CSC) and show that the binary cross-entropy (bce) loss leads to an improved convergence guarantee that adapts to problem instances with low optimal cost; such a result is called a first-order bound. This result is due to \citet{foster2021efficient}, who first observed that the bce loss can have benefits over square loss when outcomes are heteroskedastic and used this observation to derive first-order regret bounds for CSC and contextual bandits (CB). 
Then, we prove that the maximum likelihood (mle) loss yields an even better convergence guarantee that additionally adapts to problem instances with low variance; such a result is called a second-order bound and is strictly stronger than first-order. This result is due to \citet{wang2024more}, who first observed that the mle loss can have benefits over the bce loss when outcomes have low variance and used this observation to derive second-order bounds for CB and RL. 
We also provide lower bounds to show that the separation between these loss functions is indeed real, with one result new to this paper separating bce and mle in their ability to attain second-order bounds.
The decision making performance that each loss function attains in each type of problem instance is outlined in \cref{table:rates}.

We then turn to extending these observations and intuitions to RL, following the results of \citet{wang2023benefits,wang2024more,ayoub2024switching}.
We systematically derive bounds for both online and offline RL, revealing that the trends in \cref{table:rates} continue to hold for the more challenging RL settings.
Finally, we discuss issues surrounding computational complexity and provide a solution in the hybrid RL setting \citep{song2023hybrid}. 
In summary, we provide a detailed survey on the decision performance of different loss functions, with the aim of elucidating their central role in RL and data-driven decision making generally. }
The technical material is largely based on \citet{foster2021efficient,wang2023benefits,wang2024more,ayoub2024switching} with a couple new results along the way.

\section{Cost-Sensitive Classification}

To best illuminate the phenomenon, we start with the simplest setting of contextual decision making: cost-sensitive classification (CSC), where learning is done offline, decisions have no impact on future contexts, and full feedback is given for all actions.
To make it the simplest CSC setting, we even assume that the action space is finite (an assumption we shed in later sections). An instance of the CSC problem is then characterized by a context space $\Xcal$, a finite number of actions $A$, and a distribution $d$ on $\Xcal\times[0,1]^A$. The value of a policy $\pi:\Xcal\to\{1,\dots,A\}$ is its average cost under this distribution: $V(\pi)=\mathbb E[c(\pi(x))]$, where $x,c(1),\dots,c(A)\sim d$. The optimal value is $V^\star=\min_{\pi:\Xcal\to\{1,\dots,A\}}V(\pi)$.
We are given $n$ draws of $x_i,c_i(1),\dots,c_i(A)\sim d$, sampled independently and identically distributed (i.i.d.), based on which we output a policy $\hat\pi$ with the aim of it having low $V(\hat \pi)$.

Let $C:\Xcal\times\Acal\to\Delta([0,1])$ map $x,a$ to the conditional distribution of $c(a)$ given $x$ under $d$. Here $\Delta([0,1])$ denotes the set of distributions on $[0,1]$ that are absolutely continuous with respect to (w.r.t.) a base measure $\lambda$, such as Lesbesgue measure for continuous distributions or a counting measure for discrete distributions. We identify such distributions by its density function w.r.t. $\lambda$ and we write $C(y\mid x,a)$ for the density of $C(x,a)$ at $y$. We assume that $\lambda$ is common across $x,a$ and is known. We can then write value as an expectation w.r.t. $x$ alone:
$$
V(\pi)=\mathbb E[\bar C(x,\pi(x))],
$$
where the bar notation on a distribution denotes the mean: $\bar p=\int_y y p(y)\diff \lambda(y)$ for any $p\in\Delta([0,1])$.

\subsection{Solving CSC with Squared-Loss Regression}

A value-based approach to CSC is to learn a cost prediction $f(x,a)\approx \bar C(x,a)$ by regressing costs on contexts and then use an induced greedy policy: $\pi_f(x)\in\argmin_a f(x,a)$. \edit{A standard way to learn such a cost prediction is to minimize the squared error \citep{audibert2007fast,luedtke2017faster,hu2022fast}.}

Define the squared loss and the excess squared-loss risk for a prediction function $f$ as:
\begin{align*}
&\textstyle\ell_{\op{sq}}(\hat y,y):=(\hat y-y)^2,
\\&\textstyle\Ecal_{\op{sq}}(f):=\sum_{a}\mathbb E[\ell_{\op{sq}}(f(x,a),c(a))-\ell_{\op{sq}}(\bar C(x,a),c(a))].
\end{align*}
This can be used to expediently bound the sub-optimality of the policy induced by $f$:
\begin{align}
V(\pi_f)-V^\star&=\mathbb E[\bar C(x,\pi_f(x))-\bar C(x,\pi^\star(x))]\notag
\\&\notag\leq \mathbb E[\bar C(x,\pi_f(x))-f(x,\pi_f(x))\\&\notag\phantom{\leq \mathbb E[}+f(x,\pi^\star(x))-\bar C(x,\pi^\star(x))]
\\&\label{eq: bound2}\textstyle\lesssim \left(\sum_{a}\EE(f(x,a)-\bar C(x,a))^2\right)^{1/2}
\\&\notag\textstyle= \left(\Ecal_{\op{sq}}(f)\right)^{1/2},
\end{align}
where $\lesssim$ means $\leq$ up to a universal constant factor (\eg, above in \cref{eq: bound2}, it is $2$).

How do we learn a predictor with low excess squared-loss risk? We minimize the empirical squared-loss risk over a hypothesis class $\mathcal F$ of functions $\mathcal X\times\mathcal A\to[0,1]$:
\begin{equation*}
    \textstyle\hat f^{\op{sq}}_{\mathcal F}\in\argmin_{f\in \mathcal F}\sum_{i=1}^n\sum_{a=1}^A\ell_{\op{sq}}(f(x_i,a),c_i(a)).
\end{equation*}
This procedure is termed nonparametric least squares (since $\Fcal$ is general), and standard results control the
excess risk of $\hat f^{\op{sq}}_{\mathcal F}$.
Here we give a version for finite hypothesis classes, while for infinite classes the excess risk convergence depends on their complexity, such as given by the critical radius \citep{wainwright2019high}.
\begin{assumption}[Realizability]\label{ass:csc-realizability}
    $\bar C\in\Fcal$.
\end{assumption}
Under \cref{ass:csc-realizability}, for any $\delta\in(0,1)$, with probability at least (\wpal) $1-\delta$,
$$\Ecal_{\op{sq}}(\hat f^{\op{sq}}_\Fcal)\lesssim A\log(\edit{A}\abs{\mathcal F}/\delta)/n.$$

Together with \cref{eq: bound2}, we obtain the following probably approximately correct (PAC) bound:
\begin{theorem}
Under \cref{ass:csc-realizability}, for any $\delta\in(0,1)$, \wpal $1-\delta$, plug-in squared loss regression enjoys
$$V(\pi_{\hat f^{\op{sq}}_{\mathcal F}})-V^\star\lesssim \sqrt{A\log(\edit{A}\abs{\mathcal F}/\delta)/n}.$$
\end{theorem}
This PAC bound shrinks at a nice parametric rate of $\Ocal(n^{-1/2})$ as the number of samples $n$ grows, but can we do better?
The step in \cref{eq: bound2}, which translates error in predicted means to excess risk in the squared loss, was rather loose. 
\edit{\citet{foster2021efficient} thus investigated the bce loss and showed that it can achieve a first-order bound}, which we present next after introducing an important second-order lemma.

\subsection{The Second-Order Lemma}
We know that estimating a mean of a random variable is easier when the random variable has smaller variance. Our next result recovers this intuition \emph{as a completely deterministic statement about comparing bounded scalars}:
\begin{lemma}[Second-Order Mean Comparison]\label{lemma: second order}
Let $p,q$ be two densities on $[0,1]$ with respect to a common measure $\lambda'$. Then
$$
\abs{\bar p-\bar q}\leq
6\sigma(p)h(p,q)+8h^2(p,q),
$$
where the variance and the squared Hellinger distance are defined as
\begin{align*}
    &\textstyle\sigma^2(p)=\int_yy^2p(y)\diff \lambda'(y)-\bar p^2
    \\&\textstyle h^2(p,q)=\frac12\int_y(\sqrt{p(y)}-\sqrt{q(y)})^2\diff\lambda'(y).
\end{align*}
\end{lemma}
Here, $h^2(p,q)$ is the squared Hellinger distance, which is an $f$-divergence, and it is bounded in $[0,1]$.
This lemma is equivalent to Lemma 4.3 of \citep{wang2024more} and we provide a simplified proof in \cref{sec:proof-second-order-lemma}. 
Interpreting the inequality, which is a completely deterministic statement, in terms of estimating means, it says that estimation error can be bounded by two terms: one involves the standard deviation times a discrepancy and the other is a \emph{squared} discrepancy.
As variance shrinks, the first term vanishes and the second term dominates, which as a squared term we expect to decay quickly.
\edit{We note that \citep{foster2021statistical,foster2022complexity} obtain related first-moment and second-moment bounds, which are generally looser and do not directly imply the variance bound in \cref{lemma: second order}.}

\subsection{Regression with the Binary-Cross-Entropy Loss: First-Order PAC Bounds for CSC}
One way to instantiate \cref{lemma: second order} is to let $\lambda'$ be the counting measure on $\{0,1\}$ and, given any $f,g\in[0,1]$, set $p,q$ as the Bernoulli distributions with means $f,g$, respectively. Bounding $f(1-f)\leq f$, this leads to
\begin{align}\label{eq: first order}
\abs{f-g}\leq 8\sqrt fh_{\op{Ber}}(f,g)+20h^2_{\op{Ber}}(f,g),
\end{align}
where $h^2_{\op{Ber}}(f,g)=\frac12(\sqrt f-\sqrt g)^2+\frac12(\sqrt {1-f}-\sqrt {1-g})^2$ is the squared Hellinger distance between Bernoullis with means $f$ and $g$.
This recovers the key inequalities in \citep{foster2021efficient,wang2023benefits,ayoub2024switching}.

Replacing the bound in \cref{eq: bound2} with \cref{eq: first order} and using Cauchy-Schwartz, we obtain
\begin{align*}%
&\textstyle V(\pi_f)-V^\star\lesssim\sqrt{(V(\pi_f)+V^\star)\cdot\delta_{\op{Ber}}(f)}+\delta_{\op{Ber}}(f),
\\&\text{where}\,\,\textstyle \delta_{\op{Ber}}(f):=\sum_a\EE[h_{\op{Ber}}^2(\bar C(x,a),f(x,a))]. \nonumber
\end{align*}
Applying the inequality of arithmetic and geometric means (AM-GM), we see that this implies $V(\pi_f)\lesssim V^\star + \delta_{\op{Ber}}(f)$.
Plugging this implicit inequality back into the above, we have that
\begin{equation}\label{eq: first order csc pac}
V(\pi_f)-V^\star\lesssim
\sqrt{V^\star\cdot\delta_{\op{Ber}}(f)}+\delta_{\op{Ber}}(f).
\end{equation}
Since $V^\star\leq 1$, \cref{eq: first order csc pac} also implies $V(\pi_f)-V^\star\lesssim \sqrt{\delta_{\op{Ber}}(f)}$. That is, if we learn a predictor with low $\sqrt{\delta_{\op{Ber}}(f)}$, then its induced policy has correspondingly low suboptimality.
However, \cref{eq: first order csc pac} also crucially involves $V^\star$. Thus, if the optimal policy incurs little expected costs so that the first term in \cref{eq: first order csc pac} is negligible, we get to \emph{square} the rate of convergence.

How do we learn a predictor with low $\delta_{\op{Ber}}(f)$? Since $\delta_{\op{Ber}}(f)$ is an average divergence between Bernoulli distributions, we could try to fit Bernoullis to the costs. Define the binary-cross-entropy (bce) loss as
$$
\ell_{\op{bce}}(\hat y,y):=-y \ln\hat y-(1-y)\ln(1-\hat y).
$$
We adopt the convention that $0\ln 0 = 0$.
Then, $\delta_{\op{Ber}}(f)$ is bounded by an exponentiated excess bce-loss risk \citep{foster2021efficient}.
\begin{lemma}\label{lem:exponentiated-excess-bce} For any $f:\mathcal X\times\mathcal A\to[0,1]$,
\begin{align*}\textstyle\delta_{\op{Ber}}(f)\leq \Ecal_{\op{bce}}(f),\end{align*}
where $\Ecal_{\op{bce}}(f):=-\sum_{a=1}^A\ln\EE[\exp(\frac12\ell_{\op{bce}}(\bar C(x,a),c(a))-\frac12\ell_{\op{bce}}(f(x,a),c(a)))]$.
\end{lemma}
\begin{proof}
For each $a$, let $z\sim\op{Ber}(c(a))$,
\begin{align*}
    &\textstyle-\ln\EE[\exp(\frac12\ell_{\op{bce}}(\bar C(x,a),c(a))-\frac12\ell_{\op{bce}}(f(x,a),c(a)))]
    \\&\textstyle=-\ln\EE[\exp(\frac12(c(a)\ln\frac{f(x,a)}{\bar C(x,a)} + (1-c(a))\ln\frac{1-f(x,a)}{1-\bar C(x,a)} ))]
    \\&\overset{(i)}\geq\textstyle-\ln\EE[ \exp(\frac12(z\ln\frac{f(x,a)}{\bar C(x,a)} + (1-z)\ln\frac{1-f(x,a)}{1-\bar C(x,a)} )) ]
    \\&\textstyle=-\ln\EE[ \sqrt{f(x,a)\bar C(x,a)} + \sqrt{(1-f(x,a))(1-\bar C(x,a))} ]
    \\&\overset{(ii)}\geq\textstyle1-\EE[ \sqrt{f(x,a)\bar C(x,a)} + \sqrt{(1-f(x,a))(1-\bar C(x,a))} ]
    \\&\overset{(iii)}=\textstyle h_{\op{Ber}}^2(\bar C(x,a),f(x,a)).
\end{align*}
where (i) is by Jensen's inequality, (ii) is by $-\ln x \geq 1-x$, (iii) is by completing the square.
\end{proof}

To learn a predictor with low $\Ecal_{\op{bce}}$, we may consider minimizing the empirical bce-loss risk, simply replacing $\ell_{\op{sq}}$ by $\ell_{\op{bce}}$ in nonparametric least squares:
$$\textstyle
\hat f^{\op{bce}}_{\mathcal F}\in\argmin_{f\in \mathcal F}\sum_{i=1}^n\sum_{a=1}^A\ell_{\op{bce}}(f(x_i,a),c_i(a)).
$$

The bce loss $\ell_{\op{bce}}(\hat y,y)$ is exactly the negative log-likelihood of observing $y$ from a Bernoulli distribution with mean $\hat y$. Nevertheless, even if $y$ is \emph{not binary}, it can be used as a general-purpose surrogate loss for regression (sometimes under the moniker ``log loss" \citep{foster2021efficient,ayoub2024switching}). In particular, for any density $p\in\Delta([0,1])$, the mean $\bar p$ minimizes expected bce loss:
$$
\EE_{y\sim p}[\ell_{\op{bce}}(f,y)-\ell_{\op{bce}}(\bar p,y)]\geq 2(f-\bar p)^2.
$$
This inequality also means that we could use the excess bce-loss risk to bound \cref{eq: bound2} in place of excess squared-loss risk. The point of using bce loss, however, is to do better than \cref{eq: bound2} via \cref{eq: first order}.

For the final part of the proof, we need to show that minimizing the empirical bce-loss risk gives good control on $\Ecal_{\op{bce}}(\hat f^{\op{bce}}_{\mathcal F})$.
\edit{This is implied by the following tail bound.}
\begin{lemma}\label{lem:logsumexp-symmetrization}
Let $Z_1,\dots,Z_n$ denote $n$ \emph{i.i.d.} random variables.
For any $\delta\in(0,1)$, \wpal $1-\delta$,
\begin{equation*}
    -n\ln\EE[\exp(-Z_1)]\leq\textstyle\sum_{i=1}^nZ_i+\ln(1/\delta).
\end{equation*}
\end{lemma}
\begin{proof}
We note $\EE[\exp(\sum_{i=1}^n Z_i)]=(\EE[\exp(Z_1)])^n$.
By Chernoff's method, $\Pr(\sum_{i=1}^nZ_i-n\ln\EE_{Z_1}\exp(Z_1)\geq t)\leq\exp(-t)$ for all $t>0$.
Finally, set $t=\ln(1/\delta)$.
\end{proof}
Then, applying \cref{lem:logsumexp-symmetrization} with $Z_i=\frac12\ell_{\op{bce}}(f(x_i,a),c_i(a))-\frac12\ell_{\op{bce}}(\bar C(x_i,a),c_i(a))$ and a union bound over $a$ and $f$, we have \wpal $1-\delta$, for all $f\in\Fcal$
\begin{align*}
    \textstyle n\Ecal_{\op{bce}}(f)\leq&\textstyle \frac12\sum_{i=1}^n\sum_{a=1}^A\ell_{\op{bce}}(f(x_i,a),c_i(a))
    \\&\textstyle-\ell_{\op{bce}}(\bar C(x_i,a),c_i(a))+A\ln(2A\abs{\Fcal}/\delta).
\end{align*}
With \cref{ass:csc-realizability}, the empirical minimizer $\hat f^{\op{bce}}_{\mathcal F}$ enjoys
\begin{equation*}
    \Ecal_{\op{bce}}(\hat f^{\op{bce}}_{\mathcal F})\leq\textstyle\frac{A\ln(A\abs{\Fcal}/\delta)}{n}.
\end{equation*}
Thus, together with \cref{eq: first order csc pac} and \cref{lem:exponentiated-excess-bce}, we have shown a first-order PAC bound for bce-loss regression:
\begin{theorem}\label{thm:csc-first-order}
Under \cref{ass:csc-realizability}, for any $\delta\in(0,1)$, \wpal $1-\delta$, plug-in bce loss regression enjoys
\begin{equation*}
    V(\pi_{\hat f^{\op{bce}}_{\mathcal F}})-V^\star\lesssim \textstyle\sqrt{{\color{red}V^\star}\cdot \frac{A\ln(A\abs{\mathcal F}/\delta)}{n}} + \frac{A\ln(A\abs{\mathcal F}/\delta)}{n}.
\end{equation*}
\end{theorem}
This result was first observed in Theorem 3 of \citep{foster2021efficient}.
Notably, the bound is \emph{adaptive} to the optimal expected costs $V^\star$ and converges at a fast $n^{-1}$ rate when $V^\star\lesssim 1/n$.
Under the cost minimization setup, first-order bounds are also called `small-cost' bounds since they converge at a fast rate when the optimal cost $V^\star$ is small.
\begin{remark}
A refinement of \cref{eq: first order} keeps the first term as $\sqrt{f(1-f)}h_{\op{Ber}}$ instead of $\sqrt{f}h_{\op{Ber}}$.
This would imply a more refined first-order bound that scales as $\wt\Ocal(\sqrt{V^\star(1-V^\star)\cdot\frac{1}{n}} + \frac1n)$, where the leading term vanishes also if $V^\star\approx 1$.
Bounds scaling with $1-V^\star$ are sometimes called `small-reward' bounds \citep{auer2002nonstochastic}, which are generally easier to prove than `small-cost' bounds \citep{lykouris2018small,wang2023benefits}. \edit{For example, if the bound scales with the minimum achievable reward $R^\star$, the case where $R^\star=0$ is actually trivial since all policies are optimal. However, if the bound scales with the minimum achievable cost $L^\star$, the case where $L^\star=0$ is interesting because sub-optimal policies may still have large cost. Thus, following \citet{lykouris2018small,wang2023benefits}, we focus on the cost minimizing setting in this paper.}
\end{remark}

\subsection{Maximum Likelihood Estimation: Second-Order PAC Bounds for CSC}
Can we do even better than a first-order PAC bound with the bce loss?
In this section we show that a second-order, variance-adaptive bound is possible if we learn the conditional cost distribution instead of only regressing the mean.
To learn the distribution, we use a hypothesis class $\Pcal$ of conditional distributions $\Xcal\times\Acal\to\Delta([0,1])$
and minimize the negative-log likelihood loss from maximum likelihood estimation (mle): for a density $\hat p\in\Delta([0,1])$ and target $y\in[0,1]$, define
\begin{equation*}
    \ell_{\op{mle}}(\hat p,y):= -\ln\hat p(y).
\end{equation*}
Unlike the previous sections where the loss measured the discrepancy of a point prediction, the mle loss measures the discrepancy of a distributional prediction.
Indeed, if $\hat p=\op{Ber}(\hat y)$ and $p=\op{Ber}(y)$, then $\EE_{y\sim p}[\ell_{\op{mle}}(\hat p,y)]=\ell_{\op{bce}}(\hat y,y)$ so the bce loss can be viewed as a Bernoulli specialization of the general mle loss.
This generality allows us to directly apply \cref{lemma: second order} in place of \cref{eq: bound2} to obtain for any $p\in\Pcal$:
\begin{align}\label{eq: second order csc regret}
&V(\pi_{\bar p})-V^\star\lesssim\sqrt{(\sigma^2(\pi_{\bar p})+\sigma^2(\pi^\star))\delta_{\op{dis}}(p)}+\delta_{\op{dis}}(p),
\end{align}
where $\delta_{\op{dis}}(p):=\sum_a\EE[h^2(C(x,a),p(x,a))]$ and $\sigma^2(\pi):=\sigma^2(\bar C(x,\pi(x)))$.
As in the bce section, we then upper bound $\delta_{\op{dis}}(f)$ by an exponentiated excess mle-loss risk.
\begin{lemma}\label{lem:exponentiated-excess-mle} For any $f:\mathcal X\times\mathcal A\to\Delta([0,1])$,
\begin{align*}\textstyle\delta_{\op{dis}}(p)\leq \Ecal_{\op{mle}}(p),\end{align*}
where $\Ecal_{\op{mle}}(p):=-\sum_{a=1}^A\ln\EE[\exp(\frac12\ell_{\op{mle}}(C(x,a),c(a))-\frac12\ell_{\op{mle}}(p(x,a),c(a)))]$.
\end{lemma}
The proof is almost identical to that of \cref{lem:exponentiated-excess-bce}, and is even simpler since the inequality marked ``(i)" in the proof is not needed.
To learn a predictor with low $\Ecal_{\op{mle}}$, we minimize the empirical negative log-likelihood risk:
\begin{align*}
    &\hat p^{\op{mle}}_{\Pcal}\in\textstyle\argmin_{p\in \Pcal} L_{\op{mle}}(p),
    \\\text{where }&L_{\op{mle}}(p):=\textstyle\sum_{i=1}^n\sum_{a=1}^A\ell_{\op{mle}}(p(x_i,a),c_i(a)).
\end{align*}

We also posit realizability in the distribution class.
\begin{assumption}[Distributional Realizability]\label{ass:csc-dist-realizability}
    \hspace{-0.5em}$C\in\Pcal$.
\end{assumption}
\edit{While this assumption is more stringent than mean-realizability (\cref{ass:csc-realizability}  from before), it is required to prove that MLE succeeds \citep{devroye2001combinatorial,agarwal2020flambe}. Alternative algorithms such as Scheff\'e tournament can be used to learn distributions under misspecification, but they are computationally hard to implement \citep{agarwal2019reinforcement}.}

Finally, we apply \cref{lem:logsumexp-symmetrization} with $Z_i=\frac12\ell_{\op{mle}}(p(x_i,a),c_i(a))-\frac12\ell_{\op{mle}}(C(x_i,a),c_i(a))$ with union bound over $\Pcal$, to deduce that \wpal $1-\delta$, for all $p\in\Pcal$:
\begin{equation}\label{eq:csc-mle-symmetrization-inequality}
    n\Ecal_{\op{mle}}(p)\leq\textstyle\frac12 L_{\op{mle}}(p)-\frac12 L_{\op{mle}}(C) + A\ln(A|\Pcal|/\delta).
\end{equation}
Together with \cref{ass:csc-dist-realizability}, we have that
\begin{equation*}
    \Ecal_{\op{mle}}(\hat p^{\op{mle}}_{\Pcal})\leq \textstyle\frac{A\ln(A\abs{\Pcal}/\delta)}{n}.
\end{equation*}
Thus we have proven a second-order PAC bound for the greedy policy $\hat\pi^{\op{mle}}:=\pi_{\overline{\hat p^{\op{mle}}_{\Pcal}}}$.
\begin{theorem}
Under \cref{ass:csc-dist-realizability}, for any $\delta\in(0,1)$, \wpal $1-\delta$, plug-in mle enjoys
\begin{align*}
    V(\hat\pi^{\op{mle}})-V^\star
    &\lesssim \textstyle\sqrt{{\color{red}(\sigma^2(\hat\pi^{\op{mle}})+\sigma^2(\pi^\star))}\cdot\frac{A\ln(A\abs{\Pcal}/\delta)}{n}}
    \\&+ \textstyle\frac{A\ln(A\abs{\Pcal}/\delta)}{n}.
\end{align*}
\end{theorem}
Since costs are bounded in $[0,1]$, we observe that $\sigma^2(\pi)\leq V(\pi)$, and hence a second-order bound is tighter than a first-order bound.
However, our approach is distributional with bounds depending on $\ln|\Pcal|$ which can be larger than $\ln|\Fcal|$, and \edit{as mentioned earlier, realizability in distribution space is stricter than realizability in the conditional mean. We remark that distributional approaches still achieve superior performance \citep{bellemare2017distributional,bdr2023,wang2023benefits,wang2024more} in practice, suggesting that MLE with modern function approximators can well-approximate complex distributions despite the stronger assumption in theory.}

\subsection{Improved Second-Order PAC Bounds for CSC with Pessimistic MLE}\label{sec:csc-pessimistic-mle}
We can derive even tighter bounds if the distribution is learned in a pessimistic manner -- that is, the mean of the learned distribution \emph{upper bounds} the true optimal mean $V^\star$ with high probability.\footnote{Here, we say the learned mean is pessimistic if it upper bounds $V^\star$ since we're in the cost minimization setting. Under the reward maximization setting, pessimism would be to lower bound $V^\star$.}
In this section, we introduce how to achieve pessimism by optimizing over a subset of the function class defined by empirical losses, an approach that is often termed `version space' \citep{foster2018practical}.
This is also important warmup for the optimistic and pessimistic RL algorithms that we consider in the sequel.

We start by defining a subclass of near-optimal distributions w.r.t. the empirical mle loss
\begin{align*}
    &\Pcal_{n} := \{ p\in\Pcal: L_{\op{mle}}(p)-L_{\op{mle}}(\hat p^{\op{mle}}_\Pcal)\leq\beta \},
\end{align*}
where $\beta$ is a parameter that will be set appropriately.
Then, a pessimistic distribution is learnt by selecting the element with the lowest value.
The following lemma defines this formally and proves that the learned distribution (a) has low excess risk and (b) is nearly pessimistic.
\begin{lemma}\label{lem:csc-pessimism-key-lemma}
Under \cref{ass:csc-dist-realizability}, for any $\delta\in(0,1)$, set $\beta=2A\ln(A\abs{\Pcal}/\delta)$ and define,
\begin{equation}\label{eq:pessimistic-mle}
    \hat p^{\op{pes}}\in\textstyle\argmax_{p\in\Pcal_{n}}\sum_{i=1}^n\min_a\bar p(x_i,a).
\end{equation}
Then, \wpal $1-\delta$, (a) $\Ecal_{\op{mle}}(\hat p^{\op{pes}})\lesssim \frac{A\ln(A|\Pcal|/\delta)}{n}$, and (b) $V(\hat\pi^{\op{pes}})-\EE[\min_a\overline{\hat p^{\op{pes}}}(x,a)]\lesssim\frac{\ln(A|\Pcal|/\delta)}{n}$.
\end{lemma}
\begin{proof}[Proof of \cref{lem:csc-pessimism-key-lemma}]
For both claims, we condition on \cref{eq:csc-mle-symmetrization-inequality} which holds \wpal $1-\delta$.
For Claim (a): for any $p\in\Pcal_n$ (which includes $\hat p^{\op{pes}}$), we have 
\begin{align*}
    &\textstyle n\Ecal_{\op{mle}}(p)\leq \frac12L_{\op{mle}}(p)-\frac12 L_{\op{mle}}(\hat p^{\op{mle}}_\Pcal)+A\ln(A|\Pcal|/\delta)
    \\&\textstyle\leq\frac12\beta+A\ln(A|\Pcal|/\delta)\leq 2A\ln(A|\Pcal|/\delta),
\end{align*}
where the first inequality is by \cref{eq:csc-mle-symmetrization-inequality} and the fact that $\hat p^{\op{mle}}_\Pcal$ minimizes the empirical risk; and the second inequality is by the definition of $\Pcal_n$.
To prove Claim (b), we first show that $C\in\Pcal_n$: by \cref{eq:csc-mle-symmetrization-inequality} and the non-negativity of $\Ecal_{\op{mle}}$, we have $L_{\op{mle}}(C)-L_{\op{mle}}(p)\leq2A\ln(A|\Pcal|/\delta)=\beta$ for all $p\in\Pcal$ (which includes $\hat p^{\op{mle}}_\Pcal$). Thus, this shows that $C$ satisfies the $\Pcal_n$ condition, implying its membership in the set.
To conclude Claim (b), we have $\sum_{i=1}^n\overline{\hat p^{\op{pes}}}(x_i,\pi^\star(x_i))\geq\sum_{i=1}^n\min_a\overline{\hat p^{\op{pes}}}(x_i,a)\geq \sum_{i=1}^n\min_a\bar{C}(x_i,a)$.
Claim (b) then follows by multiplicative Chernoff \citep[Theorem 13.5]{zhang_2023_ltbook}.
\end{proof}
With pessimism, the induced policy $\hat\pi^{\op{pes}}:=\pi_{\overline{\hat p^{\op{pes}}}}$ only sufferes one of the terms before \cref{eq: bound2}, and so
\begin{align}
\textstyle V(\hat\pi^{\op{pes}})-V^\star&\leq\mathbb E[\overline{\hat p^{\op{pes}}}(x,\pi^\star(x))-\bar C(x,\pi^\star(x))]\notag
\\&\textstyle\lesssim\sqrt{ \sigma^2(\pi^\star)\cdot\delta_{\op{dis}}(\hat p^{\op{pes}}) } + \delta_{\op{dis}}(\hat p^{\op{pes}}) \label{eq: second order csc pessimistic}
\\&\textstyle\lesssim\sqrt{ \sigma^2(\pi^\star)\cdot\frac{A\ln(A|\Pcal|/\delta)}{n} } + \frac{A\ln(A|\Pcal|/\delta)}{n} \notag.
\end{align}
Thus, we have proven an \emph{improved} second-order PAC bound for pessimistic mle.
\begin{theorem}\label{thm:csc-improved-second-order-pessimism}
Under \cref{ass:csc-dist-realizability}, for any $\delta\in(0,1)$, \wpal $1-\delta$, pessimistic mle enjoys
\begin{align*}
    V(\hat\pi^{\op{pes}})-V^\star
    &\lesssim\textstyle\sqrt{{\color{red}\sigma^2(\pi^\star)}\cdot\frac{A\ln(A\abs{\Pcal}/\delta)}{n}}+\textstyle\frac{A\ln(A\abs{\Pcal}/\delta)}{n}.
\end{align*}
\end{theorem}
Notably, \cref{eq: second order csc pessimistic} is an improvement to \cref{eq: second order csc regret} since it only contains the variance of the optimal policy $\pi^\star$, which is a fixed quantity, and not that of the learned policy, which is a random algorithm-dependent quantity.
We remark that while pessimism is typically used to solve problems with poor coverage, \eg, offline RL, we see it also plays a crucial role in obtaining finer second-order PAC bounds in CSC, which has full coverage due to complete feedback. 

Pessimism could have also been applied with the bce-loss, but there would have been no improvement to the first-order bound. This is because $V(\hat\pi^{\op{bce}})-V^\star\leq\sqrt{ (V(\hat\pi^{\op{bce}})+V^\star)\cdot \frac{C}{n}}+\frac{C}{n}$ already implies $V(\hat\pi^{\op{bce}})-V^\star\leq\sqrt{V^\star\cdot\frac{C'}{n}}+\frac{C'}{n}$ where $\frac{C'}{C}$ is a universal constant, due to the AM-GM inequality as noted in the text preceding \cref{eq: first order csc pac}.
However, this implicit inequality does not hold for variance-based inequalities, and so pessimism is crucial for removing the dependence on the learned policy's variance.

Finally, we note that pessimistic mle requires more computation than plug-in mle, since we have the extra step of optimizing over $\Pcal_n$. For one-step settings like CSC or contextual bandits, this can be feasibly implemented with binary search \citep{foster2018practical} or \edit{width computation \citep{krishnamurthy2019active,feng2021provably}.}
However, in multi-step settings like RL as we will soon see, this optimization problem is NP-hard \citep{dann2018oracle}.

\begin{remark}[Another improved bound via optimism]
We could also consider optimistic mle where $\hat p^{\op{op}}\in\argmin_{p\in\Pcal_n}\sum_{i=1}^n\min_a\bar p(x_i,a)$ and $\hat\pi^{\op{op}}=\pi_{\overline{\hat p^{\op{op}}}}$.
Then, the decomposition of \cref{eq: second order csc pessimistic} would look like:
\begin{align*}
    \textstyle V(\hat\pi^{\op{op}})-V^\star&\leq\mathbb E[\bar C(x,\hat\pi^{\op{op}}(x))-\overline{\hat p^{\op{op}}}(x,\hat\pi^{\op{op}}(x))]\notag
    \\&\textstyle\lesssim \sqrt{ \sigma^2(\wh\pi^{\op{op}})\cdot\delta_{\op{dis}}(\hat p^{\op{op}}) } + \delta_{\op{dis}}(\hat p^{\op{op}})
    \\&\textstyle\lesssim \sqrt{ {\color{red}\sigma^2(\wh\pi^{\op{op}})}\cdot\frac{A\ln(A|\Pcal|/\delta)}{n} } + \frac{A\ln(A|\Pcal|/\delta)}{n}.
\end{align*}
This is also an improved second-order PAC bound, which depends only on the variance of the learned policy and not that of $\pi^\star$.
The bound in \cref{thm:csc-improved-second-order-pessimism} may be preferred since $\sigma(\pi^\star)$ is a fixed quantity; however, we note that $\sigma^2(\hat\pi)$ and $\sigma^2(\pi^\star)$ are not comparable in general, so neither bound dominates the other.
\end{remark}

\subsection{Proof of the Second-Order Lemma}\label{sec:proof-second-order-lemma}
The goal of this subsection is to prove the second-order lemma (\cref{lemma: second order}), a key tool to derive first- and second-order PAC bounds.
We prove the result in terms of another divergence called the triangular discrimination: for any densities $p,q$ on $[0,1]$ w.r.t. a common measure $\lambda'$, the triangular discrimination is defined as
\begin{equation}
    \textstyle\triangle(p,q):=\int_y \frac{(p(y)-q(y))^2}{p(y)+q(y)}\diff\lambda'(y).
\end{equation}
$\triangle(\cdot)$ is a symmetric $f$-divergence and is equivalent to the squared Hellinger distance up to universal constants:
\begin{equation}
    2h^2(p,q)\leq\Delta(p,q)\leq 4h^2(p,q),\label{eq:h2-equiv-dtri}
\end{equation}
which is a simple consequence of Cauchy-Schwartz \citep[Lemma A.1]{wang2023benefits}.
Thus, we first prove the second-order lemma using $\triangle(\cdot)$, which is more natural, and then convert the bounds to $h^2(\cdot)$ using \cref{eq:h2-equiv-dtri}.
\begin{lemma}\label{lem:dtri-inequalities}
Let $p,q$ be densities on $[0,1]$ and let $q$ be the one with smaller variance. Then,
\begin{align}
    &\sigma^2(p)-\sigma^2(q) \leq 2\sqrt{\sigma^2(q)\cdot\Delta(p,q)} + \Delta(p,q),\label{eq:dtri-variance-inequality}
    \\&\abs{\bar p-\bar q} \leq 3\sqrt{\sigma^2(q)\cdot\Delta(p,q)} + 2\Delta(p,q). \label{eq:dtri-mean-inequality}
\end{align}
\end{lemma}
\begin{proof}[Proof of \cref{lem:dtri-inequalities}]
We first prove \cref{eq:dtri-variance-inequality}:
\begin{align*}
    &\sigma^2(p)-\sigma^2(q)
    \\&\overset{i}\leq\textstyle\int_y (y-\bar q)^2(p(y)-q(y))\diff\lambda'(y)
    \\&\overset{ii}\leq\textstyle\sqrt{ \int_y(y-\bar q)^4(p(y)+q(y))\diff\lambda'(y)\cdot \triangle(p,q)}
    \\&\overset{iii}\leq\textstyle\sqrt{ \int_y(y-\bar q)^2(p(y)+q(y))\diff\lambda'(y)\cdot \triangle(p,q)}
    \\&\overset{iv}\leq\textstyle\frac12\prns{\int_y(y-\bar q)^2p(y)\diff\lambda'(y)+\sigma^2(q)} + \frac{\triangle(p,q)}{2},
\end{align*}
where (i) is by the definition of variance, (ii) is by Cauchy-Schwarz, (iii) is by the premise that $p,q$ are densities on $[0,1]$, and (iv) is by AM-GM. Rearranging the inequality that (i)$\leq$(iv), we get
\begin{equation*}
    \textstyle\int_y (y-\bar q)^2p(y)\diff\lambda'(y) \leq 3\sigma^2(q)+\triangle(p,q).
\end{equation*}
Finally, plugging this back into (iii) implies \cref{eq:dtri-variance-inequality}.

Now we prove \cref{eq:dtri-mean-inequality}.
Set $c=\frac{\bar p+\bar q}{2}$.
First, consider the case that $D_\triangle(p,q)\leq1$:
\begin{align*}
    \abs{\bar p-\bar q}^2
    &=\textstyle\abs{\int_y(p(y)-q(y))(y-c)\diff\lambda'(y)}^2
    \\&\overset{i}\leq\textstyle\int_y (p(y)+q(y))(y-c)^2\diff\lambda'(y)\cdot \triangle(p,q)
    \\&\overset{ii}=\textstyle\prns{\sigma^2(p)+\sigma^2(q)+2\prns{\frac{\bar p-\bar q}{2}}^2} \triangle(p,q)
    \\&\overset{iii}\leq\textstyle\prns{\sigma^2(p)+\sigma^2(q)} \triangle(p,q) + \frac{\prns{\bar p-\bar q}^2}{2}
\end{align*}
where (i) is by Cauchy-Schwarz,
(ii) is by expanding the variance $\sigma^2(f) = \int_y(f(y)-c)^2\diff\lambda(y)-(\bar f-c)^2$ which holds for any $c\in\RR$,
and (iii) is by $\triangle(p,q)\leq 1$.
Rearranging terms and using \cref{eq:dtri-variance-inequality}, we get
\begin{align*}
    \abs{\bar p-\bar q}
    &\leq \sqrt{2(\sigma^2(p)+\sigma^2(q))\triangle(p,q)}
    \\&\leq \sqrt{2(3\sigma^2(q)+2\triangle(p,q))\triangle(p,q)}
    \\&\leq 3\sqrt{\sigma^2(q)\triangle(p,q)} + 2\triangle(p,q).
\end{align*}
This finishes the case of $\triangle(p,q)\leq1$.
Otherwise, we simply have $\abs{\bar p-\bar q}\leq1<\triangle(p,q)$.
\end{proof}

\section{Lower Bounds for CSC}
So far, we have seen that plug-in regression with the squared loss, bce loss and mle loss have progressively more adaptive PAC bounds for CSC.
A natural question is if the previous bounds were tight: is it necessary to change the loss function if we want to achieve these sharper and more adaptive bounds?
In this section, we answer this in the affirmative by exhibiting counterexamples.

\subsection{Plug-in Squared Loss Cannot Achieve First-Order}
First, we show that the policy induced by squared loss regression cannot achieve first-order bounds.
The following counterexample is due to \citep{foster2021efficient}, where we have simplified the presentation and improved constants.
\begin{theorem}\label{thm:square-loss-cannot-achieve-first-order}
For all $n>400$, there exists a CSC problem with $|\Acal|=|\Xcal|=2$ and a realizable function class with $|\Fcal|=2$ such that: (a) $V^\star\leq\frac{1}{n}$, but (b) $V(\pi_{\hat f^{\op{sq}}_\Fcal})-V^\star\geq \frac{1}{32\sqrt{n}}$ \wpal $0.1$.
\end{theorem}
The intuition is that squared loss regression does not adapt to context-dependent variance, \emph{a.k.a.} heteroskedasticity; so the convergence of squared loss regression is dominated by the worst context's variance.
In this counterexample, the second context $x^2$ occurs with tiny probability $n^{-1}$ but has high variance; however, the empirical squared loss is dominated by this unlikely context. %
\begin{proof}[Proof of \cref{thm:square-loss-cannot-achieve-first-order}]
The structure of the proof is the following: for any $n>400$, we first construct the CSC problem and realizable function class, and then show that indeed $V^\star\leq\Ocal(\frac1n)$.
Next, we show that under a bad event which occurs with probability at least $0.1$, the function with the lowest empirical squared risk induces a policy that suffers regret which is lower bounded by $\Omega(\frac{1}{\sqrt{n}})$.

Fix any $n>400$. We begin by setting up the CSC problem: label the two states as $x^1,x^2$ and the two actions as $a^1,a^2$.
Set the data generating distribution $d$ as follows: $d(x^1)=1-n^{-1}$, $d(x^2)=n^{-1}$ and
\begin{align*}
    &c(a^1)\mid x^1\sim\op{Ber}(\mu_n),\quad &c(a^2)\mid x^1=\nu_n,
    \\&\textstyle c(a^1)\mid x^2\sim\op{Ber}(\frac12),\quad &\textstyle c(a^2)\mid x^2=\frac12,
\end{align*}
where $\mu_n=\frac{1}{8n},\nu_n=\frac{1}{8\sqrt{n}}$.
Our realizable function class $\Fcal$ contains two elements: the true $f^\star(x,a)=\EE[c(a)\mid x]$ and another function $\wt f$ defined as
\begin{align*}
    &\textstyle\wt f(x^1,a^1)=\eps_n,\quad &\wt f(x^1,a^2)=\textstyle\nu_n,
    \\&\textstyle\wt f(x^2,a^1)=0,\quad &\wt f(x^2,a^2)=\textstyle\frac12,
\end{align*}
where $\eps_n=\frac{1}{4\sqrt{n}}$.
Note that $\mu_n<\nu_n$ so $\pi^\star(x^1)=a^1$ but $\eps_n>\nu_n$ so $\pi_{\wt f}(x^1)=a^2$, \ie, $\pi_{\wt f}$ makes a mistake on $x^1$.
Also, $V^\star=(1-n^{-1})\mu_n + \frac{1}{2n}\leq \frac{1}{n}$.

Now, we compute the empirical squared-loss risk and show that $\hat f^{\op{sq}}_\Fcal = \wt f$ under a bad event.
The empirical risk can be simplified by shedding shared terms to be:
\begin{align*}
    \textstyle\wh L_{\op{sq}}(f) = \sum_{i\in[2]}\frac{n(x^i)}{n}(f(x^i,a^1)-\wh\mu(x^i,a^1))^2,%
\end{align*}
where $n(x)$ denotes the number of times $x$ occurs in the dataset and $\wh\mu(x,a)=\frac{1}{n(x)}\sum_{i: x_i=x}c_i(a)$ is the empirical conditional mean.
We split the bad event into two parts:
($\mathfrak{E}_1$) $x^2$ appears only once in the dataset (\ie, $n(x^2)=1$) and its observed cost at $a^1$ is $0$ (\ie, $\wh\mu(x^2,a^1)=0$);
and ($\mathfrak{E}_2$) $\wh\mu(x^1,a^1)\leq 2\mu_n+\frac{3}{n-1}$. We lower bound $\Pr(\mathfrak{E}_1\cap\mathfrak{E}_2)\geq 0.1$ at the end.

We now show that $\wh f^{\op{sq}}_\Fcal = \wt f$ under $\mathfrak{E}_1\cap\mathfrak{E}_2$.
Under $\mathfrak{E}_1$, we lower bound $\wh L_{\op{sq}}(f^\star)$ by:
\begin{equation*}\textstyle
    \frac{n(x^2)}{n}(f^\star(x^2,a^1)-\wh\mu(x^2,a^1))^2 = \frac{1}{4n}.
\end{equation*}
Under $\mathfrak{E}_{1}\cap\mathfrak{E}_2$, the $x^2$ term of $\wh L_{\op{sq}}(\wt f)$ vanishes since $\wt f(x^2,a^1)=\wh\mu(x^2,a^1)$, and so $\wh L_{\op{sq}}(\wt f)$ can be bounded by:
\begin{align*}
    \textstyle(\wt f(x^1,a^1)-\wh\mu(x^1,a^1))^2
    \leq 2\eps_n^2+2(2\mu_n+\frac{3}{n-1})^2 < \frac{1}{4n},
\end{align*}
where the last inequality holds due to $n>400$.
Thus, squared loss regression selects $\hat f^{\op{sq}}_\Fcal = \wt f$ and the regret of the induced policy can be lower bounded by:
\begin{align*}
    V(\pi_{\hat f^{\op{sq}}_\Fcal})-V^\star
    &\textstyle=\frac{n-1}{n}(\nu_n-\mu_n)\geq\frac{1}{16}(\frac{1}{\sqrt{n}}-\frac{1}{n})\geq \frac{1}{32\sqrt{n}}.
\end{align*}
\textbf{Probability of the bad event.}
For $\mathfrak{E}_1$, since $n(x^2)\sim\op{Bin}(n,n^{-1})$, thus $\Pr(n(x^2)=1)=(1-n^{-1})^{n-1}\geq e^{-1}$.
Hence, $\Pr(\mathfrak{E}_1)\geq (2e)^{-1}$.
For $\mathfrak{E}_2$, we apply the multiplicative Chernoff bound \citep[Theorem 13.5]{zhang_2023_ltbook}%
, which implies $\wh\mu(x^1,a^1)< 2\mu_n+\frac{3}{n-1}$ \wpal $1-e^{-3}$.
Thus, $\Pr(\mathfrak{E}_1\cap\mathfrak{E}_2)\geq 1-(1-(2e)^{-1})-e^{-3}\geq 0.1$.
\end{proof}

\subsection{Plug-in BCE Loss Cannot Achieve Second-Order}
Next, we show that the bce-loss induced policy cannot achieve second-order bounds. This is a new result that did not appear before.
\begin{theorem}\label{thm:bce-loss-cannot-achieve-second-order}
For all odd $n\in\NN$, there exists a CSC problem where $|\Acal|=2,|\Xcal|=1$ and a realizable function class with $|\Fcal|=2$ such that: (a) $\sigma^2(\pi^\star) = 0$, but (b) $V(\pi_{\hat f^{\op{bce}}_\Fcal})-V^\star\geq \frac{1}{8\sqrt{n}}$ \wpal $\frac14$.
\end{theorem}
\begin{proof}[Proof of \cref{thm:bce-loss-cannot-achieve-second-order}]
The proof structure is similar as before: for any odd $n$, we construct the CSC problem and a realizable function class. We show that $\sigma^2(\pi^\star)=0$ which is the second-order regime; we also sanity check that $V^\star$ is bounded away from $0$ and $1$, to ensure that we're not in the first-order regime.
Next, we show that under a bad event which occurs with constant probability, the function with the lowest empirical bce risk induces a policy that suffers regret which is lower bounded by $\Omega(\frac{1}{\sqrt{n}})$.

Fix any odd $n\in\NN$. We first construct the CSC problem: label the two actions as $a^1,a^2$ and drop the context notation since there is one context.
Set the data generating distribution $d$ such that: $c(a^1)\sim\op{Ber}(\frac12+\eps_n)$ and $c(a^2)=\frac12$ w.p. $1$.
The true conditional means are $f^\star(a^1)=\frac12+\eps_n$ and $f^\star(a^2)=\frac12$.
In addition to $f^\star$, the function class $\Fcal$ only contains one other function $\wt f$ defined as $\wt f(a^1)=\wt f(a^2)=\frac12$.
Note that the optimal action is $a^\star=a^2$ and the regret of $a^1$ is $f^\star(a^1)-f^\star(a^2)=\eps_n$.
We also check that $V^\star=\Theta(1)$, and so this is not the first-order regime.

Now, we compute the empirical bce-loss risk and show that $\hat f^{\op{bce}}_\Fcal = \wt f$ under a bad event.
Since all elements of $\Fcal$ have the same prediction for $a^2$, the empirical bce-loss risk can be simplified to
\begin{align*}
    \wh L_{\op{bce}}(f)&\textstyle=p\cdot\ell_{\op{bce}}(f(a^1),0) + (1-p)\cdot\ell_{\op{bce}}(f(a^1),1)
    \\&\textstyle=\ell_{\op{bce}}(f(a^1), 1-p),
\end{align*}
where $p$ is the fraction of times that $c(a^1)=0$ in the dataset.
The above loss is convex and its minimizer is $1-p$.
The bad event we consider is that $p>\frac12$, under which we have $1-p<\wt f(a^1)<f^\star(a^1)$; since the loss is convex, $\wt f$ indeed achieves lower loss than $f^\star$.
Thus, we have that $\wh f^{\op{bce}}_\Fcal = \wt f$ and the regret of the induced policy is
$
V(\pi_{\wh f^{\op{bce}}_\Fcal})-V^\star=f^\star(a^1)-f^\star(a^2)=\frac{1}{8\sqrt{n}}.
$
\\\textbf{Probability of the bad event.}
\citet{binomial_deviations} proved tight lower and upper bounds for binomial small deviations and we will make use of the following result: for all $n\geq 1$ and $\gamma\in[0,\frac{1}{\sqrt{n}}]$, let $\Pr(\op{Bin}(n,\frac12-\frac{\gamma}{2})\leq\floor{\frac{n}{2}})-\Pr(\op{Bin}(n,\frac12+\frac{\gamma}{2})\leq\floor{\frac{n}{2}})\leq\sqrt{n}\gamma$.
If $n$ is odd, we have that $\frac12=\Pr(\op{Bin}(n,\frac12)\leq\floor{\frac{n}{2}})<\Pr(\op{Bin}(n,\frac12-\frac{\gamma}{2})\leq\floor{\frac{n}{2}})$. Thus,
$\Pr(\op{Bin}(n,\frac12+\frac{\gamma}{2})<\frac{n}{2})\geq \frac12-\sqrt{n}\gamma$. Setting $\gamma = \frac{1}{4\sqrt{n}}$ (corresponding to $\eps_n=\frac{1}{8\sqrt{n}}$), we have shown that the bad event occurs with probability at least $\frac14$.
\end{proof}

In the above proof, we used two key properties of the empirical bce risk: (1) its minimizer is the empirical mean, and (2) it is convex w.r.t. the prediction (\ie, the first argument).
Since squared loss also has these properties, the above result also applies to squared regression.
However, since the mle loss learns a distribution rather than just the mean, the counterexample does not apply.
Finally, since CSC is the most basic decision making setting, the counterexamples in this section also apply to reinforcement learning via the online-to-batch conversion.

\begin{table*}[t]
\caption{\edit{Summary of main RL results in this paper. Since the algorithms presented all use temporal difference learning, all results posit some variant of Bellman Completeness. In online RL, each bound uses a slightly different eluder dimension based on the loss; in offline RL, all bounds use the single policy coverage $C^{\wt\pi}$ (defined in \cref{eq:single-policy-coverage}). Finally, we remark that the online and offline RL algorithms employ optimism and pessimism over a version space, so they are computationally inefficient \citep{dann2018oracle}. For computational efficiency, we study the hybrid RL setting (\cref{sec:hybrid-rl}) where we show that fitted-Q iteration with access to online and offline data achieves a bound that is the sum of online and offline bounds.}}
\label{table:summary-of-rl-results}
\renewcommand{\arraystretch}{1.5} %
\edit{
\begin{tabular}{@{}c@{\hspace{3pt}}c@{\hspace{3pt}}c@{\hspace{3pt}}c@{\hspace{3pt}}c@{}}
\toprule
Setting & Loss & Function Class & Bellman Completeness & Regret / PAC Bound \\ \midrule
\multirow{3}{*}{\rotatebox[origin=c]{90}{Online RL}} & $\ell_{\op{sq}}$ & $\Fcal:\Xcal\times\Acal\mapsto[0,1]$ & $\Tcal^{\star}$ (\cref{ass:rl-bc}) & $\Ocal\prns{H\sqrt{Kd_{\op{sq}}\ln(H|\Fcal|/\delta)}}$ (\cref{thm:online-rl-regret-squared-loss}) \\ \cmidrule{2-5} 
& $\ell_{\op{bce}}$ & $\Fcal:\Xcal\times\Acal\mapsto[0,1]$ & $\Tcal^{\star}$ (\cref{ass:rl-bc}) & $\Ocal\prns{H\sqrt{{\color{red}V^\star} \cdot Kd_{\op{bce}}\ln(H|\Fcal|/\delta)} + H^2 d_{\op{bce}}\ln(H|\Fcal|/\delta)}$ (\cref{thm:online-rl-regret-bce-loss}) \\ \cmidrule{2-5}
& $\ell_{\op{mle}}$ & $\Pcal:\Xcal\times\Acal\mapsto\Delta([0,1])$  & $\Tcal^{\op{D},\star}$ (\cref{ass:dist-rl-bc}) & $\Ocal\prns{H\sqrt{{\color{red}\sum_{k=1}^K\sigma^2(\pi^k)}\cdot d_{\op{mle}}\ln(H|\Pcal|/\delta)} + H^{2.5} d_{\op{mle}}\ln(H|\Pcal|/\delta)}$ (\cref{thm:online-rl-regret-mle-loss})  \\ \midrule
\multirow{3}{*}{\rotatebox[origin=c]{90}{Offline RL}} & $\ell_{\op{sq}}$ & $\Fcal:\Xcal\times\Acal\mapsto[0,1]$ & $\Tcal^{\pi},\forall\pi\in\Pi$ (\cref{ass:offrl-bc}) & $\Ocal\prns{H\sqrt{\frac{C^{\wt\pi}\ln(H|\Fcal|/\delta)}{n}}}$ (\cref{thm:offline-rl}) \\ \cmidrule{2-5}
& $\ell_{\op{bce}}$ & $\Fcal:\Xcal\times\Acal\mapsto[0,1]$ & $\Tcal^{\pi},\forall\pi\in\Pi$ (\cref{ass:offrl-bc}) & $\Ocal\prns{H\sqrt{{\color{red}V^{\wt\pi}}\cdot \frac{C^{\wt\pi}\ln(H|\Fcal|/\delta)}{n}} + H^2 \frac{C^{\wt\pi}\ln(H|\Fcal|/\delta)}{n}}$ (\cref{thm:offline-rl}) \\ \cmidrule{2-5}
& $\ell_{\op{mle}}$ & $\Pcal:\Xcal\times\Acal\mapsto\Delta([0,1])$  & $\Tcal^{\op{D},\pi},\forall\pi\in\Pi$ (\cref{ass:offrl-distbc}) & $\Ocal\prns{H\sqrt{{\color{red}\sigma^2(\wt\pi)}\cdot\frac{C^{\wt\pi}\ln(H|\Pcal|/\delta)}{n}} + H^{2.5}\frac{C^{\wt\pi}\ln(H|\Pcal|/\delta)}{n}}$ (\cref{thm:offline-dist-rl})  \\ \bottomrule
\end{tabular}
}
\end{table*}

\section{Setup for Reinforcement Learning}
In the preceding sections, we saw how the loss function plays a central role in the sample efficiency of algorithms for CSC, the simplest decision making problem.
The commonly used squared loss results in slow $\Theta(1/\sqrt{n})$ rates in benign problem instances where the optimal policy has small cost (\ie, first-order) or has small variance (\ie, second-order), while the bce or mle losses, respectively, can be used to achieve fast $\Ocal(1/n)$ rates.

In the rest of this paper, we will see that these observations and insights generally transfer to more complex decision making setups, in particular reinforcement learning (RL). Compared to the CSC setting, two new challenges of RL are that (1) the learner receives feedback only for the chosen action (\emph{a.k.a.}, partial or bandit feedback) and (2) the learner sequentially interacts with the environment over multiple time steps.
As before, we focus on value-based algorithms with function approximation and prove bounds for problems with high-dimensional observations, \ie, beyond the finite tabular setting.

\subsection{Problem Setup}
We formalize the RL environment as a Markov Decision Process (MDP) which consists of an observation space $\Xcal$, action space $\Acal$, horizon $H$, transition kernels $\{P_h:\Xcal\times\Acal\to\Delta(\Xcal)\}_{h\in[H]}$ and conditional cost distributions $\{C_h:\Xcal\times\Acal\to\Delta([0,1])\}_{h\in[H]}$.
We formalize the policy as a tuple of mappings $\pi=\{\pi_h:\Xcal\to\Delta(\Acal)\}_{h\in[H]}$
that interacts (\emph{a.k.a.} rolls-in) with the MDP as follows: start from an initial state $x_1$ and at each step $h=1,2,\dots,H$, sample an action $a_h\sim \pi_h(x_h)$, collect a cost $c_h\sim C_h(x_h,a_h)$ and transit to the next state $x_{h+1}\sim P_h(x_h,a_h)$.
We use $Z^\pi=\sum_{h=1}^H c_h$ to denote the cumulative cost, a random variable, from rolling in $\pi$; we consider the general setup where $Z^\pi$ is normalized between $[0,1]$ almost surely which allows for sparse rewards \citep{jiang2018open}.
We use $Z^\pi_h(x_h,a_h)=\sum_{t=h}^H c_t$ to denote the cumulative cost of rolling in $\pi$ from $x_h,a_h$ at step $h$.
We use $Q^\pi_h(x_h,a_h)=\EE[Z^\pi_h(x_h,a_h)]$ and $V^\pi_h(x_h)=Q_h^\pi(x_h,\pi)$ to denote the expected cumulative costs, where we use the shorthand $f(x,\pi)=\EE_{a\sim\pi(x)}f(x,a)$ for any $f$.
For simplicity, we assume the initial state $x_1$ is fixed and known, and we let $V^\pi:=V^\pi_1(x_1)$ denote the initial state value of $\pi$.
Our results can be extended to the case when $x_1$ is stochastic from an unknown distribution, or, in the online setting, the initial state at round $k$ may even be chosen by an adaptive adversary.

\textbf{Online RL.}
The learner aims to compete against the optimal policy denoted as $\pi^\star = \argmin_{\pi}V^\pi_1(x_1)$. We use $Z^\star,V^\star,Q^\star$ to denote $Z^{\pi^\star},V^{\pi^\star},Q^{\pi^\star}$, respectively.
The online RL problem iterates over $K$ rounds: for each round $k=1,2,\dots,K$, the learner selects a policy $\pi^k$ to roll-in and collect data, and the goal is to minimize regret,
\begin{equation}
    \op{Reg}_{\normalfont\textsf{RL}}(K)=\textstyle\sum_{k=1}^K V^{\pi^k} - V^\star.
\end{equation}
We also consider PAC bounds where the learner outputs $\pi^k$ at each round but may roll-in with other exploratory policies to better collect data.

\textbf{Offline RL.} The learner is given a dataset of prior interactions with the MDP and, unlike online RL, cannot gather more data by interacting with the environment.
The dataset takes the form $\Dcal=(\Dcal_1,\Dcal_2,\dots,\Dcal_H)$ where each $\Dcal_h$ contains $n$ \emph{i.i.d.} samples $(x_{h,i},a_{h,i},c_{h,i},x_{h,i}')$ where $(x_{h,i},a_{h,i})\sim\nu_h,\,c_{h,i}\sim C_h(x_{h,i},a_{h,i})$ and $x_{h,i}'\sim P_h(x_{h,i},a_{h,i})$.
We note that $\nu_h$ is simply the marginal distribution over $(x_h,a_h)$ induced by the data generating process, \eg, mixture of policies.
We also recall the (single-policy) coverage coefficient: given a comparator policy $\wt\pi$, define $C^{\wt\pi}=\max_{h\in[H]}\|\diff d^{\wt\pi}_{h}/\diff\nu_h\|_\infty$ \citep{xie2021bellman,uehara2022pessimistic}.
The goal is to learn a policy $\wh\pi$ with a PAC guarantee against any comparator policy $\wt\pi$ such that $C^{\wt\pi}<\infty$.

\textbf{Hybrid RL.} We also consider the hybrid setting where the learner is given a dataset as in offline RL, and can also gather more data by interacting with the environment as in online RL \citep{song2023hybrid,ball2023efficient}.
By combining the analyses from both online and offline settings, we prove that fitted $Q$-iteration (FQI) \citep{munos2008finite}, a computationally efficient algorithm that does not induce optimism or pessimism, can achieve first- and second-order regret and PAC bounds.

\textbf{Bellman equations.} 
Define the Bellman operator $\Tcal^\pi$ by $\Tcal_h^\pi f(x,a)=\EE_{c\sim C_h(x,a),x'\sim P_h(x,a)}[c+f(x',\pi_{h+1}) ]$ for any function $f$ and policy $\pi$
The Bellman equations are $f_h = \Tcal^\pi_h f_{h+1}$ for all $h$, where $f_h=Q^\pi_h$ is the unique solution.
Also, the Bellman optimality operator $\Tcal^\star$ is defined by $\Tcal_h^\star f(x,a)=\EE_{c\sim C_h(x,a),x'\sim P_h(x,a)}[c+\min_{a'} f(x',a') ]$.
The Bellman optimality equations are $f_h = \Tcal^\star_h f_{h+1}$ for all $h$, where $f_h=Q^\star_h$ is the unique solution.

\textbf{Distributional Bellman equations.} There are also distributional analogs to the above \citep{bdr2023}.
Let $\Tcal^{\op{D},\pi}$ denote the distributional Bellman operator for policy $\pi$, defined by $\Tcal_h^{\op{D},\pi} p(x,a)\overset{D}{=}c + p(x',a')$ where $c\sim C_h(x,a),x'\sim P_h(x,a),a'\sim\pi_{h+1}(x')$ for any conditional distribution $p$.
Here $\overset{D}{=}$ denotes equality in distribution.
The distributional Bellman equations are $p_h\overset{D}{=}\Tcal^{\op{D},\pi}_h p_{h+1}$ for all $h$, where $Z^\pi_h$ is a solution.
The distributional Bellman optimality operator $\Tcal^{\op{D},\star}$ is defined by $\Tcal_h^{\op{D},\star} p(x,a)\overset{D}{=}c + p(x',a')$ where $c\sim C_h(x,a),x'\sim P_h(x,a),a'=\argmin_{a'}\bar p(x',a')$.
The distributional Bellman optimality equations are $p_h\overset{D}{=}\Tcal^{\op{D},\star}_h p_{h+1}$ for all $h$, where $Z^\star_h$ is a solution.

\section{Online RL with Squared-Loss Regression}\label{sec:online-rl-squared-loss}

We begin our discussion of RL by solving online RL with optimistic temporal-difference (TD) learning with the squared loss for regression \citep{jin2021bellman,xie2023the}, which can be viewed as an abstraction for deep RL algorithms such as DQN \citep{mnih2015human}.
The algorithm is value-based, meaning that it aims to learn the optimal $Q$-function $Q^\star$, which then induces the optimal policy via greedy action selection $\pi^\star_h(x)=\argmin_{a}Q^\star_h(x,a)$.
To learn the $Q$-function, it uses a function class $\Fcal$ that consists of function tuples $f=(f_1,f_2,\dots,f_H)\in\Fcal$ where $f_h:\Xcal\times\Acal\to[0,1]$ and we use the convention that $f_{H+1}=0$ for all functions $f$.

In the sequential RL setting, TD learning is a powerful idea for regressing $Q$-functions where the function at step $h$ is regressed on the current cost plus a \emph{learned} prediction at the next step $h+1$. This process is also known as bootstrapping.
One can view this as an approximation to the Bellman equations $Q_h=\Tcal_h Q_{h+1}$ where $\Tcal$ is a Bellman operator.
For online RL, we use the Bellman optimality operator $\Tcal^\star$ to learn the optimal $Q^\star$, while in offline RL we use the policy-specific Bellman operator $\Tcal^\pi$ to learn $Q^\pi$ for all policies $\pi$.

\begin{algorithm}[!t]
\caption{Policy Data Collection}
\label{alg:roll-in}
\begin{algorithmic}[1]
    \State\textbf{Input:} policy $\pi$, uniform exploration (UA) flag.
    \If{UA flag is True}
    \For{step $h\in[H]$}
        \State Roll-in $\pi$ for $h$ steps to arrive at $x_{h}$.
        \State Then, randomly act $a_{h}\sim\op{Unif}(\Acal)$ and observe $c_{h},x_{h}'$.
    \EndFor
    \Else
    \State Roll-in $\pi$ for $H$ steps and collect $x_{1},a_{1},c_{1}$, $\dots$, $x_{H},a_{H},c_{H}$.
    \State Label $x_{h}'=x_{h+1}$ for all $h\in[H]$.
    \EndIf
    \State\textbf{Output:} dataset $\{(x_h,a_h,c_h,x_h')\}_{h\in[H]}$.
\end{algorithmic}
\end{algorithm}

\begin{algorithm}[!t]
\caption{Optimistic Online RL}
\label{alg:online-rl}
\begin{algorithmic}[1]
    \State\textbf{Input:} number of rounds $K$, function class $\Fcal$, loss function $\ell(\hat y,y)$, threshold $\beta$, uniform exploration (UA) flag
    \For{round $k=1,2,\dots,K$}
        \State Denote $\Fcal_k=\Ccal_\beta^\ell(\Dcal_{<k})$ as the version space defined by:
        \begin{align}
            \textstyle\Ccal_\beta^\ell(\Dcal)=\{f\in\Fcal:\,\, &\textstyle\forall h\in[H],\,L_h^\ell(f_h,f_{h+1},\Dcal_h) \notag
            \\&\textstyle-\min_{g_h\in\Fcal_h}L_h^\ell(g_h,f_{h+1},\Dcal_h)\leq\beta\}, \label{eq:rl-confidence-set}
        \end{align}
        where
        \begin{equation*}
            \textstyle L_h^\ell(f_h,g,\Dcal_h)=\sum_{i=1}^{|\Dcal_h|} \ell(f_h(x_{h,i},a_{h,i}),\tau^\star(g,c_{h,i},x_{h,i}'))
        \end{equation*}
        and $\tau^\star(g,c,x')=c+\min_{a'}g(x',a')$ is the regression target.
        In the proofs, we use $L^{\op{sq}}$ if $\ell=\ell_{\op{sq}}$ and $L^{\op{bce}}$ if $\ell=\ell_{\op{bce}}$.
        \State Get optimistic $f^k\gets\argmin_{f\in\Fcal_k}\min_a f_1(x_{1},a)$. \label{line:rl-optimism}
        \State Let $\pi^k$ be greedy w.r.t. $f^k$: $\pi^k_h(x)=\argmin_a f^k_h(x,a),\forall h$.
        \State Gather data $\Dcal_k\gets{\text{\cref{alg:roll-in}}}(\pi^k, \text{UA flag})$.
    \EndFor
\end{algorithmic}
\end{algorithm}

To formalize TD learning, let $(x_h,a_h,c_h,x_h')$ be a transition tuple where $c_h,x_h'$ are sampled conditional on $x_h,a_h$.
For a predictor $f_{h+1}$ at step $h+1$, the regression targets at step $h$ are:
\begin{align*}
    &\textstyle\tau^\star(f_{h+1},c,x')=c+\min_{a'}f_{h+1}(x',a'),
    \\&\textstyle\tau^\pi(f_{h+1},c,x')=c+f_{h+1}(x',\pi_{h+1}),
\end{align*}
where $\tau^\star$ is the target for learning $Q^\star$ which we use for online RL, and $\tau^\pi$ is the target for learning $Q^\pi$ which we use for offline RL.
The targets are indeed unbiased estimates of the Bellman backup since $\Tcal_hf_{h+1}(x,a)=\EE[\tau(f_{h+1},c,x')]$.
Then, we regress $f_h$ by minimizing the loss $\ell(f_h(x,a),\tau(f_{h+1},c,x'))$ averaged over the data, where the loss function $\ell(\hat y,y)$ captures the discrepancy between the prediction $\hat y$ and target $y$.
Note this takes the same form as the regression loss from the CSC warmup.

In online RL, the algorithm we consider (\cref{alg:online-rl}) performs TD learning optimistically by maintaining a version space constructed with the TD loss.
Specifically, given a dataset $\Dcal=(\Dcal_1,\dots,\Dcal_H)$ where each $\Dcal_h=\{x_{h,i},a_{h,i},c_{h,i},x_{h,i}'\}_{i\in[n]}$ is a set of transition tuples, the version space $\Ccal_\beta^\ell(\Dcal)$ is defined in \cref{eq:rl-confidence-set} of \cref{alg:online-rl}.
Intuitively, the version space contains all functions $f\in\Fcal$ which nearly minimize the empirical TD risk measured by loss function $\ell$, for all time steps $h$.
This construction is useful since it satisfies two properties with high probability.
First, any function in $\Ccal_\beta^\ell(\Dcal)$ has small population TD risk (a.k.a. Bellman error) w.r.t. $\ell$, so we can be assured that choosing any function from the version space is a good estimate of the desired $Q^\star$.
Second, we have that $Q^\star$ is an element of the version space, which provides a means to achieve optimism (or pessimism) by optimizing over the version space.
Indeed, by selecting the function in the version space with the minimum initial state value, we are guaranteed to select a function that lower bounds the optimal policy's cost $V^\star$.

We now summarize the online RL algorithm (\cref{alg:online-rl}), which proceeds iteratively.
At each round $k=1,2,\dots,K$, the learner selects an optimistic function $f^k$ from the version space defined by previously collected data: $f^k\gets\argmin_{f\in\Fcal_k}\min_a f_1(x_{1},a)$ where $\Fcal_k=\Ccal_\beta^\star(\Dcal_{<k})$ and $\Dcal_{<k}$ denotes the previously collected data.
Then, let $\pi^k$ be the greedy policy w.r.t. $f^k$: $\pi^k(x)=\argmin_a f^k(x,a)$.
Finally, roll-in with $\pi^k$ to collect data, as per \cref{alg:roll-in}.

The roll-in procedure (\cref{alg:roll-in}) has two variants depending on the uniform action (UA) flag.
If UA is enabled, we roll-in $H$ times with a slightly modified policy: for each $h\in[H]$, we collect a datapoint from $\pi^k\circ_h\op{unif}(\Acal)$, which denotes the policy that executes $\pi^k$ for $h-1$ steps and switches to uniform actions at step $h$.
If UA is disabled, we roll-in $\pi^k$ once and collect trajectory $x_{1,k},a_{1,k},c_{1,k},\dots,x_{H,k},a_{H,k},c_{H,k}$.
While UA requires $H$ roll-ins per round, this more exploratory data collection is useful for proving bounds with non-linear MDPs.
The collected data is then used to define the confidence set at the next round.

As a historical remark, this algorithm was first proposed with the squared loss $\ell_{\op{sq}}$ under the name GOLF by \citep{jin2021bellman} and then extended with the mle loss $\ell_{\op{mle}}$ under the name O-DISCO by \citep{wang2023benefits}.
In this section, we focus on the squared loss case, recovering the results of \citep{jin2021bellman}.
In the subsequent sections, we propose a new variant with the bce loss $\ell_{\op{bce}}$, and then finally disuss application of the mle loss, recovering the results of \citep{wang2024more}.

We now state the Bellman Completness (BC) assumption needed to ensure that \cref{alg:online-rl} succeeds \citep{chen2019information,jin2021bellman,xie2021bellman,chang2022learning}.
\begin{assumption}[$\Tcal^\star$-BC]\label{ass:rl-bc}
$\Tcal_h^\star f_{h+1}\in\Fcal_h$ for all $h\in[H]$ and $f_{h+1}\in\Fcal_{h+1}$.
\end{assumption}
BC ensures that the TD-style regression which bootstraps on the next prediction is realizable, playing the same role as realizability (\cref{ass:csc-realizability}) in the CSC setting.
In fact, BC implies realizability in $Q^\star$: $Q_h^\star\in\Fcal_h$ for all $h$, which can be verified by using the Bellman optimality equations and induction from $h=H\to1$.
While appealing, $Q^\star$-realizability is not sufficient for sample efficient RL \citep{wang2021what,foster2022offline,weisz2021exponential} and TD learning can diverge or converge to bad points with realizability alone \citep{tsitsiklis1996analysis,munos2008finite,kolter2011fixed}.
We note that $Q^\star$-realizability becomes sufficient when combined with \edit{low coverability and} generative access to the MDP \citep{mhammedi2024power}, where the learner can reset to any previously observed states. We believe that the techniques in this paper can lead to first- and second-order bounds with realizability plus generative access for example. However, we do not pursue this direction here since exchanging BC for other conditions is orthogonal to our study of loss functions.

We also define the eluder dimension,\footnote{\cref{def:dist-eluder-dim} is often called the \emph{distributional} eluder dimension to distinguish it from the classic eluder dimension of \citep{russo2013eluder}. To not confuse with distributional RL, we simply refer to it as the eluder dimension.} a flexible structural measure that quantifies the complexity of exploration and representation learning \citep{jin2021bellman}.
\begin{definition}[Eluder Dimension]\label{def:dist-eluder-dim}
Fix any set $\Scal$, function class $\Psi=\{\psi:\Scal\to\edit{\RR}\}$, distribution class $\Mcal=\{\mu:\Delta(\Scal)\}$, threshold $\eps_0$, and number $q\in\NN$. The $\ell_q$-eluder dimension $\EluDim_q(\Psi,\edit{\Mcal},\eps_0)$ is defined as the length of the longest sequence \edit{$\mu^{(1)},\dots,\mu^{(L)}\subset\Mcal$} s.t. $\exists\eps\geq\eps_0$, $\forall t\in[L]$, $\exists\psi\in\Psi$ s.t. $\abs{\EE_{\mu^{(t)}}[\psi]}>\eps$ but $\sum_{i<t}\abs{\EE_{\mu^{(i)}}[\psi]}^q\leq\eps^q$.
\end{definition}
Taking $\Scal=\Xcal\times\Acal$, we will instantiate the $\Psi$ class to be a set of TD errors measured by the regression loss function.
For example with squared loss, we set $\Psi_h^{\op{sq}}=\{ \Ecal^{\op{sq}}_h(\cdot;f): f\in\Fcal \}$ where
\begin{equation*}
    \textstyle\Ecal^{\op{sq}}_h(x,a;f):= (f_h(x,a) - \Tcal^\star_h f_{h+1}(x,a))^2.
\end{equation*}
The distribution class $\Mcal$ will be the set of all visitation distributions by any policy, \ie, $\Mcal_h = \{x,a\mapsto d^\pi_h(x,a):\pi\in\Pi\}$ where $d^\pi_h(x,a)$ is the state-action visitation distribution of $\pi$ at time step $h$.
If UA is enabled, then we will have $\Scal=\Xcal$ and \edit{the ``V-type'' distribution class is $\Mcal^{\op{V}}_h=\{x\mapsto d^\pi_h(x):\pi\in\Pi\}$ where $d^\pi_h(x)$ is the state visitation distribution of $\pi$ at time step $h$. Moreover we let
$\Psi_h^{\op{sq},\op{V}}=\{ \EE_{a\sim\op{unif}(\Acal)}[ \psi(x,a) ]: \psi\in\Psi^{\op{sq}}_h\}$ denote the ``V-type'' function class, where ``V-type'' refers to the fact that the functions only depend on state and not action. } 
Thus, define the eluder dimension for squared loss:
\begin{align*}
    &\textstyle d_{\op{sq}}=\max_{h\in[H]}\EluDim_2(\Psi_h^{\op{sq}},\Mcal_h,1/K),
    \\&\textstyle d_{\op{sq}}^{\op{V}}=\max_{h\in[H]}\EluDim_2(\Psi_h^{\op{sq},\op{V}},\Mcal_h^{\op{V}},1/K).
\end{align*}

We now state the guarantees for \cref{alg:online-rl} with the squared loss $\ell_{\op{sq}}$, which recovers the main results of \citet{jin2021bellman}.
\begin{theorem}\label{thm:online-rl-regret-squared-loss}
Under \cref{ass:rl-bc}, for any $\delta\in(0,1)$, \wpal $1-\delta$, \cref{alg:online-rl} with the squared loss $\ell_{\op{sq}}$ and $\beta=2\ln(H|\Fcal|/\delta)$ enjoys the following:
\begin{equation*}
    \textstyle\sum_{k=1}^K V^{\pi^k}-V^\star\leq\wt\Ocal\prns*{ H\sqrt{K\cdot d\beta} },
\end{equation*}
where $d=d_{\op{sq}}$ if UA is false and $d=Ad_{\op{sq}}^{\op{V}}$ if UA is true, where $A$ is the number of actions.
\end{theorem}
This shows that \cref{alg:online-rl} with the squared loss is guaranteed to learn a policy that converges to the optimal policy at a $\Theta(K^{-1/2})$ rate, which is the minimax-optimal rate.
In \cref{sec:low-rank-mdp} we show that the V-type dimension can be bounded by the rank of the transition kernel in a low-rank MDP \citep{agarwal2020flambe,agarwal2023provable}, a canonical model for RL with non-linear function approximation.

\textbf{Computational complexity of version space algorithms.}
\edit{The algorithms we present for online and offline RL optimize over version spaces to establish optimism and pessimism, and this optimization over version space is computationally inefficient in general \citep{dann2018oracle}. The computational hardness comes from the non-convex optimization, not necessarily from any loss function itself. However, we note that the version space optimization is oracle-efficient in the one-step $H=1$ setting, a.k.a. contextual bandits \citep{foster2018practical,feng2021provably,wang2024more}.
In the RL setting, there are also approaches to mitigate the computational hardness of optimism.
One common approach in practice is to use myopic exploration strategies such as epsilon-greedy \citep{mnih2015human,bellemare2017distributional}, which also has bounded regret when exploration is easy \citep{dann2022guarantees,zhang2023on}. Another approach which we later present is to work in the hybrid RL setting, where the learner can both interact with the MDP in an online manner and has access to offline data with good coverage \citep{song2023hybrid}. In \cref{sec:hybrid-rl}, we will see that our results for all loss functions naturally carry over to the hybrid setting, giving us both computational and statistical efficiency.}

We now prove the main theorem for squared loss \cref{thm:online-rl-regret-squared-loss}. We prove the nested lemmas in the Appendix.
\begin{proof}[Proof of \cref{thm:online-rl-regret-squared-loss}.]
We define the excess squared-loss risk for $f\in\Fcal$ under the visitation distribution of $\pi$ as
\begin{equation*}
    \textstyle\Ecal^{\op{sq}}_h(\pi;f):=\textstyle\EE_{\pi}[\Ecal^{\op{sq}}_h(x_h,a_h;f)],
\end{equation*}
and also set $\Ecal^{\op{RL}}_{\op{sq}}=\sum_{h=1}^H \Ecal^{\op{sq}}_h$.
We first establish an optimism lemma for $f^k$ and prove it in the appendix. %
\begin{restatable}{lemma}{SqRLOptimismLemma}\label{lem:sq-rl-optimism}
Let $\ell=\ell_{\op{sq}}$ and $\Dcal_h$ be a dataset where the $i$-th datapoint is collected from $\pi^i$, and denote $\Dcal=(\Dcal_1,\dots,\Dcal_H)$.
Then under BC (\cref{ass:rl-bc}), for any $\delta\in(0,1)$, let $\beta=2\ln(H|\Fcal|/\delta)$ and define
\begin{equation}
    \hat f^{\op{op}}\in\argmin_{f\in\Ccal_\beta^{\op{sq}}(\Dcal)}\min_a f_1(x_1,a).\label{eq:optimistic-f-def-rl}
\end{equation}
\Wpal $1-\delta$, we have (a) $\sum_{i=1}^n\Ecal^{\op{RL}}_{\op{sq}}(\pi^i;\hat f^{\op{op}})\leq 2H\beta$, and (b) $\min_a\hat f^{\op{op}}_1(x_1,a)\leq V^\star$.
\end{restatable}
By \cref{lem:sq-rl-optimism}, we have $\sum_{k=1}^KV^{\pi^k}-V^\star\leq\sum_{k=1}^K V^{\pi^k}-\min_a f^k_1(x_1,a)$, which can be further decomposed by the performance difference lemma (PDL) \citep{agarwal2019reinforcement,kakade2002approximately}.
\begin{lemma}[PDL]\label{lem:pdl}
$\forall f=(f_1,f_2,\dots,f_H)$ and $\pi$, we have
$
V^{\pi}-f_1(x_1,\pi)=\sum_{h=1}^H\EE_{\pi}\bracks{ (\Tcal^\pi_h f_{h+1}-f_h)(x_h,a_h) }.
$
\end{lemma}
By PDL and Cauchy-Schwarz, we have
\begin{align}
    &\textstyle\sum_{k=1}^K V^{\pi^k}-f^k_1(x_{1},\pi^k(x_{1})) \label{eq:online-rl-regret-decomp-squared-loss}
    \\&=\textstyle\sum_{k=1}^K\sum_{h=1}^H\EE_{\pi^k}[\Tcal_h f^k_{h+1}(x_{h},a_h)-f^k_h(x_h,a_h)] \notag
    \\&\leq\textstyle\sum_{k=1}^K\sum_{h=1}^H\sqrt{\Ecal^{\op{sq}}_h(f^k,\pi^k)} \notag
    \leq\textstyle\sum_{k=1}^K \sqrt{H\Ecal_{\op{sq}}^{\op{RL}}(f^k, \pi^k)} \notag
    \\&\leq\textstyle\sqrt{HK \sum_{k=1}^K\Ecal_{\op{sq}}^{\op{RL}}(f^k,\pi^k)}. \notag
\end{align}
The final step is to bound $\sum_{k=1}^K\Ecal_{\op{sq}}^{\op{RL}}(f^k,\pi^k)$.
By \cref{lem:sq-rl-optimism}, we have that $\sum_{i<k}\Ecal_{\op{sq}}^{\op{RL}}(f^k,\pi^i)\lesssim H\beta$ for all $k$, which is very similar except that the expectations are taken under previous policies $\pi^{<k}$ instead of $\pi^k$.
It turns out that the eluder dimension can establish a link between the two, via the following ``pigeonhole principle'' lemma:
\begin{lemma}[Pigeonhole]\label{lem:pigeonhole}
Fix a number $N\in\NN$, a sequence of functions $\psi^{(1)},\dots,\psi^{(N)}\in\Psi$, and distributions \edit{$\mu^{(1)},\dots,\mu^{(N)}\in\Mcal$}.
Suppose $\sum_{i<j}|\EE_{\mu^{(i)}}[\psi^{(j)}]|^q\leq \beta^q$ for all $j\in[N]$. 
Then we have $\sum_{j=1}^N\abs*{\EE_{\mu^{(j)}}[\psi^{(j)}]}\leq 2\EluDim_q(\Psi,\Mcal,N^{-1})\cdot(E+\beta^q\ln(EN))$, where $E:=\sup_{\mu\in\Mcal,\psi\in\Psi}\abs{\EE_\mu[\psi]}$ is the envelope.
\end{lemma}
If we interpret $\psi^{(i)}$ as the regression error at round $i$, then \cref{lem:pigeonhole} essentially states that ratio of (online) out-of-distribution errors (\ie, $\psi^{(i)}$ measured under $p^{(i)}$) to the (offline) in-distribution errors (\ie, $\psi^{(i)}$ measured under $p^{(1)},\dots,p^{(i-1)}$) is bounded by the eluder dimension. This generalizes similar results from \citep{russo2013eluder,jin2021bellman,liu2022partially,wang2023benefits,xie2023the}.

Finally, going back to the regret decomposition in \cref{eq:online-rl-regret-decomp-squared-loss}, applying the pigeonhole lemma implies that $\sum_{k=1}^K\Ecal^{RL}_{\op{sq}}(f^k,\pi^k)\leq\wt\Ocal(d_{\op{sq}}H\beta)$.
Thus, we have shown that $\sum_{k=1}^K V^{\pi^k}-f^k_1(x_{1},\pi^k(x_{1}))\leq \wt\Ocal(H\sqrt{Kd_{\op{sq}}\beta})$, which proves the desired regret bound.

Moreover, if the uniform action (UA) flag is set, we also perform a change of measure so that:
$\sum_{k=1}^K\Ecal^{RL}_{\op{sq}}(f^k,\pi^k)\leq A\sum_{k=1}^K\sum_{h=1}^H\EE_{\pi^k\circ_h\op{unif}(\Acal)}((f^h-\Tcal_hf^k_{h+1})(x_h,a_h))^2\lesssim Ad_{\op{sq}}^{\op{V}}H\beta$.
Plugging into \cref{eq:online-rl-regret-decomp-squared-loss} proves the desired PAC bound of $\wt\Ocal(H\sqrt{AKd_{\op{sq}}^{\op{V}}\beta})$.
This finishes the proof.
\end{proof}

\subsection{Verifying Assumptions for Low-Rank MDPs}\label{sec:low-rank-mdp}
In this subsection, we show that the assumptions in \cref{thm:online-rl-regret-squared-loss} (as well as subsequent theorems with other loss functions) are satisfied in low-rank MDPs \citep{agarwal2020flambe}, a class of rich-observation MDPs where the transition kernel has an unknown low-rank decomposition.
\begin{definition}[Low-Rank MDP]\label{def:low-rank-mdp}
An MDP has rank $d$ if its transition kernel has a low-rank decomposition: $P_h(x'\mid x,a) = \phi_h^\star(x,a)^\top\mu_h^\star(x')$ where $\phi_h^\star,\mu_h^\star\in\RR^d$ are unknown feature maps that satisfy $\|\phi^\star_h(x,a)\|_2\leq 1$ and $\|\int g\diff\mu^\star_h(x')\|_2 \leq \|g\|_\infty\sqrt{d}$ for all $x,a,x'$ and $g:\Xcal\to\RR$.
We also require that the expected cost is linear in the features: $\bar C_h(x,a)=\phi^\star_h(x,a)^\top v_h^\star$ for some unknown vectors $v_h^\star\in\RR^d$ that satisfy $\|v_h^\star\|_2\leq \sqrt{d}$.
\end{definition}
This model captures non-linear representation learning since $\phi^\star$ and $\mu^\star$ are \emph{unknown} and can be non-linear. The low-rank MDP model also generalizes many other models such as linear MDPs (where $\phi^\star$ is known) \citep{jin2020provably}, block MDPs \citep{misra2020kinematic} and latent variable models \citep{modi2024model}.

To perform representation learning, we posit a feature class $\Phi=\Phi_1\times\dots\times\Phi_H$ where each $\phi_h:\Xcal\times\Acal\to\RR^d\in\Phi_h$ is a candidate for the ground truth features $\phi^\star$.
\begin{assumption}[$\phi^\star$-realizability]\label{asm:phi-realizability}
$\phi^\star_h\in\Phi_h$ for all $h$.
\end{assumption}
Then, the following class of linear functions in $\Phi$ satisfies all the assumptions needed in \cref{thm:online-rl-regret-squared-loss} and \cref{thm:online-rl-regret-bce-loss}, a subsequent result with the bce loss.
\begin{equation*}
    \Fcal^{\op{lin}}_h:=\{\op{clip}(\langle\phi_h(\cdot),w\rangle, 0, 1): w\in\RR^d\text{ s.t. }\|w\|_2\leq2\sqrt{d}\},
\end{equation*}
where $\op{clip}(y,l,h):=\max(\min(y,h),l)$.
This function class is sensible because Bellman backups of any function are linear in $\phi^\star_h$; thus $Q$-functions are Bellman backups via the Bellman equations, they are linear in $\phi^\star$.
The clipping is to ensure that the functions are bounded in $[0,1]$ which is true for the desired $Q^\star$.

We now show that $\Fcal^{\op{lin}}$ satisfies BC (\cref{ass:rl-bc}).
\begin{lemma}\label{lem:linear-class-satisfies-bc}
In a low-rank MDP, under \cref{asm:phi-realizability}, $\Fcal^{\op{lin}}$ satisfies Bellman Completeness (\cref{ass:rl-bc,ass:offrl-bc}).
\end{lemma}
\begin{proof}[Proof of \cref{lem:linear-class-satisfies-bc}]
Fix any $f_{h+1}\in\Fcal^{\op{lin}}_{h+1}$ and $\pi$.
We want to show $\Tcal^\pi_hf_{h+1}\in\Fcal^{\op{lin}}_h$.
First, we note $\Tcal_h^\pi f_{h+1}(x,a)$ is equal to
\begin{equation}
    \textstyle\phi_h^\star(x,a)^\top(v^\star_h+\int_{x'}f_{h+1}(x',\pi(x'))\diff\mu_h^\star(x')).\label{eq:low-rank-decomposition}
\end{equation}
Setting $w=v^\star_h+\int_{x'}f_{h+1}(x',\pi(x'))\diff\mu_h^\star(x')$, we indeed have that $\|w\|_2\leq \sqrt{d} +\sqrt{d}\|f_{h+1}\|_\infty \leq 2\sqrt{d}$, which implies $\Tcal^\pi_hf_{h+1}(x,a)\in\Fcal^{\op{lin}}_{h}$.
\end{proof}

Moreover, we can also show that the V-type eluder dimension is bounded by the rank $d$ of the low-rank MDP, as defined in \cref{def:low-rank-mdp}. \edit{We note this applies to both $\ell_1$ and $\ell_2$ eluder dimensions.}
\begin{lemma}\label{lem:low-rank-bounded-eluder}
In a low-rank MDP with rank $d$, we have \edit{$\EluDim_1(\Psi^{\op{V}}_h,\Mcal^{\op{V}}_h,\eps)\leq \EluDim_2(\Psi^{\op{V}}_h,\Mcal^{\op{V}}_h,\eps)\leq \Ocal(d\ln(d/\eps))$} for all steps $h\in[H]$ and function classes $\Psi^{\op{V}}_h\subset\Xcal\to\RR$.
\end{lemma}
\begin{proof}
$\EluDim_2(\Psi^{\op{V}}_h,\Mcal^{\op{V}}_h,\eps)\leq \Ocal(d\ln(d/\eps))$ can be proved by applying a standard elliptical potential argument to the decomposition in \cref{eq:low-rank-decomposition}; for example, see \citep{jin2021bellman} or Theorem G.4 of \citep{wang2023benefits}. Then, $\EluDim_1(\Psi^{\op{V}}_h,\Mcal^{\op{V}}_h,\eps)\leq \EluDim_2(\Psi^{\op{V}}_h,\Mcal^{\op{V}}_h,\eps)$ is a simple consequence of the fact that $\sqrt{\sum_i x_i^2}\leq\sum_i|x_i|$ \citep[Lemma 5.4]{wang2023benefits}.
\end{proof}
Since the above lemma holds \emph{for all} values of $\Psi^{\op{V}}_h$, this implies that $d^{\op{V}}_{\op{sq}}$ (and $d^{\op{V}}_{\op{bce}},d^{\op{V}}_{\op{mle}}$ to be defined in future theorems) are all bounded by $\wt\Ocal(d)$ in low-rank MDPs.
Finally, one can also show that the bracketing entropy of $\Fcal^{\op{lin}}_h$ is $\wt\Ocal(d+\log|\Phi|)$.
We note that our PAC bounds can all be extended to allow for infinite classes such as $\Fcal^{\op{lin}}$ via a standard bracketing argument, \eg, see \citep{jin2021bellman,wang2023benefits} for detailed extensions.
Thus, we have established that our bounds hold in low-rank MDPs when the algorithm uses the linear function class $\Fcal^{\op{lin}}$.

\section{First/Second-Order Bounds for Online RL}\label{sec:first-second-bounds-online-rl}
As we learned from the CSC warmup, algorithms with the squared loss can be sub-optimal in small-cost or small-variance problems. \edit{The intuition is that squared loss regression bounds do not capture the underlying heteroskedastic variance; indeed, minimizing squared loss implicitly assumes that the underlying distribution is a homoskedastic Gaussian.}
We also learned that simply swapping the loss function for the bce loss or mle loss can yield first-order or second-order bounds that are more sample efficient in small-loss or small-variance settings. We now show that this observation smoothly extends to RL as well. \edit{In particular, we will sharpen the Cauchy-Schwarz step in the proof of \cref{thm:online-rl-regret-squared-loss} by leveraging \cref{eq: first order} from the CSC warmup.}

\subsection{First-Order Bounds with BCE Regression}
In this subsection, we analyze \cref{alg:online-rl} with the bce loss $\ell_{\op{bce}}$ and derive improved first-order bounds. Before stating guarantees with the bce loss, we first define the eluder dimension which measures discrepancy with the Bernoulli squared hellinger distance.
Let $\Psi_h^{\op{bce}}=\{ \delta^{\op{Ber}}_h(\cdot;f): f\in\Fcal \}$ where
\begin{equation*}
    \textstyle\delta^{\op{Ber}}_h(x,a;f):= h^2_{\op{Ber}}(f_h(x,a), \Tcal^\star_h f_{h+1}(x,a))^2,
\end{equation*}
and $\Psi^{\op{bce},\op{V}}_h=\{\EE_{a\sim\op{unif}(\Acal)}[\psi(x,a)]:\psi\in\Psi^{\op{bce}}_h\}$.
Then define the eluder dimension for bce loss:
\begin{align*}
    &\textstyle d_{\op{bce}}=\max_{h\in[H]}\EluDim_1(\Psi_h^{\op{bce}},\Mcal_h,1/K),
    \\&\textstyle d_{\op{bce}}^{\op{V}}=\max_{h\in[H]}\EluDim_1(\Psi_h^{\op{bce},\op{V}},\Mcal_h^{\op{V}},1/K).
\end{align*}
The following guarantees for \cref{alg:online-rl} with bce loss is new.
\begin{theorem}\label{thm:online-rl-regret-bce-loss}
Under \cref{ass:rl-bc}, for any $\delta\in(0,1)$, \wpal $1-\delta$, \cref{alg:online-rl} with the bce loss $\ell_{\op{bce}}$ and $\beta=2\ln(H|\Fcal|/\delta)$ enjoys the following:
\begin{equation*}
    \textstyle\sum_{k=1}^K(V^{\pi^k}-V^\star)\leq\wt\Ocal\prns*{ H\sqrt{{\color{red}V^\star} K\cdot d\beta} + H^2d\beta },
\end{equation*}
where $d=d_{\op{bce}}$ if UA is false and $d=Ad_{\op{bce}}^{\op{V}}$ if UA is true.
\end{theorem}
Compared to the non-adaptive bounds of squared loss (\cref{thm:online-rl-regret-squared-loss}), the above bce loss bounds are first-order and shrinks with the optimal policy's cost $V^\star$.
This adaptive scaling with $V^\star$ gives the bound a small-cost property: if $V^\star\leq \Ocal(1/K)$ (\ie, if the optimal policy achieves low cost), then the leading term vanishes and the bound enjoys logarithmic-in-$K$ regret, \ie, $\sum_{k=1}^KV^{\pi^k}-V^\star\leq \wt\Ocal(H^2d\beta)$.
In other words, by dividing both sides by $K$, the sub-optimality gap of the best learned policy shrinks at a fast $\wt\Ocal(1/K)$ rate. %
Moreover, since $V^\star\leq 1$, \cref{thm:online-rl-regret-bce-loss} is never worse than the $\wt\Ocal(\sqrt{K})$ rate from \cref{thm:online-rl-regret-squared-loss}, and so these two bounds match in the worst-case but bce loss is strictly better in the small-cost regime.

\edit{We remark that \cref{thm:online-rl-regret-bce-loss} uses the same completeness assumption as \cref{thm:online-rl-regret-squared-loss}, although its bound contains a slightly different eluder dimension $d_{\op{bce}}$ instead of $d_{\op{sq}}$. The eluder dimension for bce is different because the bce loss naturally measures Bellman error via the Bernoulli squared hellinger while the squared loss uses squared distance. The change to $\ell_1$ eluder is due to the generalization bound analysis for log-losses, although it is actually sharper than $\ell_2$ eluder \citep{liu2022partially,wang2023benefits}. However, this is not a significant change for our MDP of interest: \cref{lem:low-rank-bounded-eluder} ensures that both eluder dimensions are bounded by $\wt\Ocal(d)$ in low-rank MDPs with rank $d$. Indeed, this implies that the first-order bound in \cref{thm:online-rl-regret-bce-loss} can be specialized for low-rank MDPs by the same argument as before, yielding the first small-loss bound for low-rank MDPs in online RL without distributional RL \citep{wang2023benefits}. Characterizing the exact differences between these eluder dimension variants is an interesting question for future work.}

\edit{In terms of related works, several other works also propose to use different losses than squared loss to evaluate the Bellman error. For example, \citet{bas2021logistic} argued that the logistic loss is advantageous since it is convex, whereas the squared loss is not convex in the $Q$-function due to the $\max$ operator. \citet{farebrother2024stop} is a more applied paper which shows that classification losses such as cross-entropy scale much better with deep networks than squared loss. Overall these other works provide other reasons for why squared loss is sub-optimal, while we focus on the improved sample efficiency aspect of employing alternative losses.}
We now prove \cref{thm:online-rl-regret-bce-loss}.
\begin{proof}[Proof of \cref{thm:online-rl-regret-bce-loss}.]
For the bce loss, we measure the Bellman error of $f\in\Fcal$ under $\pi$ using the squared Hellinger distance of Bernoullis (as defined in \cref{eq: first order}):
\begin{equation*}
    \textstyle\delta_h^{\op{Ber}}(\pi;f):=\EE_\pi[\delta^{\op{Ber}}_h(x_h,a_h;f)],
\end{equation*}
We define the bce excess risk $\Ecal^{\op{bce}}_h(\pi;f)$ as:
\begin{align*}
    \textstyle-\ln\EE_\pi[\exp(&\textstyle\frac12\ell_{\op{bce}}(\Tcal_hf_{h+1}(x_h,a_h),\tau^\star(f_{h+1},c_h,x_{h+1}))
    \\&\textstyle-\frac12\ell_{\op{bce}}(f_h(x_h,a_h),\tau^\star(f_{h+1},c_h,x_{h+1})))].
\end{align*}
Note that $\Tcal_h^\star f_{h+1}(x_h,a_h) = \EE[\tau^\star(f_{h+1},c_h,x_{h+1})\mid x_h,a_h]$, which is realizable by BC.
Recall that $\delta_h^{\op{Ber}}\leq \Ecal_h^{\op{bce}}$ by \cref{lem:exponentiated-excess-bce}.
We also write $\delta_{\op{Ber}}^{\op{RL}}=\sum_{h=1}^H\delta^{\op{Ber}}_h$ and $\Ecal_{\op{bce}}^{\op{RL}}=\sum_{h=1}^H\Ecal^{\op{bce}}_h$.
The following lemma establishes optimism and is analogous to \cref{lem:sq-rl-optimism}.
\begin{restatable}{lemma}{BCERLOptimism}\label{lem:bce-rl-optimism}
Let $\ell=\ell_{\op{bce}}$. Under the same setup as \cref{lem:sq-rl-optimism} with $\hat f^{\op{op}}$ selected from $\Ccal^{\op{bce}}_\beta$ instead of $\Ccal^{\op{sq}}_\beta$, \wpal $1-\delta$, we have (a) $\sum_{i=1}^n\Ecal_{\op{bce}}^{\op{RL}}(\hat f^{\op{op}},\pi^i)\leq 2H\beta$, and (b) $\min_a\hat f^{\op{op}}_1(x_1,a)\leq V^\star$.
\end{restatable}
By \cref{lem:bce-rl-optimism}, we have $\sum_{k=1}^K(V^{\pi^k}-V^\star)\leq\sum_{k=1}^K (V^{\pi^k}-\min_a f^k_1(x_1,a))$.
Then, the proof follows similarly as the squared loss case from before, except that we apply the finer \cref{eq: first order} in place of Cauchy-Schwarz:
\begin{align}
    &\textstyle\sum_{k=1}^K (V^{\pi^k}-f^k_1(x_{1},\pi^k(x_{1}))) \label{eq:online-rl-regret-decomp-bce-loss}
    \\&=\textstyle\sum_{k=1}^K\sum_{h=1}^H\EE_{\pi^k}[\Tcal_h^\star f^k_{h+1}(x_h,a_h)-f^k_h(x_h,a_h)] \notag
    \\&\leq\textstyle\sum_{k=1}^K\sum_{h=1}^H\sqrt{\EE_{\pi^k}[f^k_h(x_h,a_h)]\cdot \delta^{\op{Ber}}_h(f^k,\pi^k)} \notag
    \\&\textstyle\phantom{\sum_{k=1}^K\sum_{h=1}^H}+\delta^{\op{Ber}}_h(f^k,\pi^k) \notag.
    \\&\leq\textstyle\sum_{k=1}^K\sqrt{\sum_{h=1}^H\EE_{\pi^k}[f^k_h(x_h,a_h)]\cdot \delta_{\op{Ber}}^{\op{RL}}(f^k,\pi^k)} \notag
    \\&\textstyle\phantom{\sum_{k=1}^K\sum_{h=1}^H}+\delta_{\op{Ber}}^{\op{RL}}(f^k,\pi^k) \notag.
\end{align}
Now, we bound $\sum_{h=1}^H\EE_{\pi^k}[f^k_h(x_h,a_h)]$ by $HV^{\pi^k}$ plus some lower-order error terms, which we achieve with a `self-bounding' lemma:
\begin{restatable}{lemma}{SelfBoundingBCELemma}\label{lem:self-bounding-bce}
Define $\delta^{\op{Ber}}_h$ that uses $\Tcal^\pi$ instead of $\Tcal^\star$:
\begin{equation*}
    \textstyle\delta^{\op{Ber}}_h(f,\pi,x_h,a_h):=h^2_{\op{Ber}}(f_h(x_h,a_h),\Tcal^\pi_hf_{h+1}(x_h,a_h)).
\end{equation*}
Then, for any $f$, $\pi$, $x_h,a_h$,
\begin{align*}
    \textstyle f_h&\textstyle(x_h,a_h)\leq eQ^\pi_h(x_h,a_h)+77H\delta_{\op{Ber}}^{\op{RL}}(f,\pi).
\end{align*}
This implies the corollary:
\begin{equation*}
    \textstyle\EE_\pi[f_h(x_h,a_h)]\lesssim V^\pi+H\delta_{\op{Ber}}^{\op{RL}}(f,\pi).
\end{equation*}
\end{restatable}
By \cref{lem:self-bounding-bce}, we can bound \cref{eq:online-rl-regret-decomp-bce-loss} by
\begin{align*}
    &\lesssim\textstyle\sum_{k=1}^K\sqrt{HV^{\pi^k}\delta_{\op{Ber}}^{\op{RL}}(f^k,\pi^k)}
    +H\delta_{\op{Ber}}^{\op{RL}}(f^k,\pi^k)
    \\&\leq\textstyle\sqrt{H\sum_{k=1}^KV^{\pi^k}\cdot\sum_{k=1}^K\delta_{\op{Ber}}^{\op{RL}}(f^k,\pi^k)}
    \\&\textstyle\phantom{\leq}+H\sum_{k=1}^K\delta_{\op{Ber}}^{\op{RL}}(f^k,\pi^k).
\end{align*}
By \cref{lem:bce-rl-optimism} and the pigeonhole principle, the error terms $\Ecal_{\op{bce}}^{\op{RL}}$ can be bounded similarly as in the squared loss proof: we can bound $\sum_{k=1}^K\delta_{\op{Ber}}^{\op{RL}}(f^k,\pi^k)$ by $\wt\Ocal(d_{\op{bce}}H\beta)$ if UA is false, and by $\wt\Ocal(Ad_{\op{bce}}^{\op{V}}H\beta)$ if UA is true.

Thus, setting $d=d_{\op{bce}}$ if UA is false, and $d=Ad_{\op{bce}}^{\op{V}}$ if UA is true, we have proven
\begin{equation*}
    \textstyle\sum_{k=1}^K(V^{\pi^k}-V^\star)\leq H\sqrt{\sum_{k=1}^KV^{\pi^k} \cdot d\beta} + H^2d\beta.
\end{equation*}
Finally, we observe an implicit inequality where we can replace $\sum_{k=1}^KV^{\pi^k}$ by $KV^\star$ by collecting a factor of $3$, as shown by the following lemma.
\begin{restatable}{lemma}{ImplicitFirstOrderIneq}\label{lem:implicit-first-order-ineq}
If $\sum_{k=1}^K (V^{\pi^k}-V^\star)\leq c\sqrt{\sum_{k=1}^K V^{\pi^k}} + c^2$, then $\sum_{k=1}^K V^{\pi^k}-V^\star\leq c\sqrt{2KV^\star} + 3c^2$.
\end{restatable}
This concludes the proof of \cref{thm:online-rl-regret-bce-loss}.
\end{proof}

\subsection{Second-Order Bounds with MLE}
\begin{algorithm}[!t]
\caption{Optimistic Online Distributional RL}
\label{alg:online-dist-rl}
\begin{algorithmic}[1]
    \State\textbf{Input:} number of rounds $K$, conditional distribution class $\Pcal$, threshold $\beta$, uniform exploration (UA) flag
    \For{round $k=1,2,\dots,K$}
        \State Define confidence set $\Pcal_k=\Ccal_\beta^{\op{mle}}(\Dcal_{<k})$ where we define:
        \begin{align}
            \textstyle\Ccal_\beta^{\op{mle}}(\Dcal)=\{p\in\Pcal:\,\, &\textstyle\forall h\in[H],\, L_h^{\op{mle}}(p_h,p_{h+1},\Dcal_h) \notag
            \\&\textstyle-\min_{g_h\in\Pcal_h}L_h^{\op{mle}}(g_h,p_{h+1},\Dcal_h)\leq\beta\}, \label{eq:dist-rl-confidence-set}
        \end{align}
        where
        \begin{equation*}
            \textstyle L_h^{\op{mle}}(p_h,g,\Dcal_h)=\sum_{i=1}^{|\Dcal_h|} \ell_{\op{mle}}(p_h(x_{h,i},a_{h,i}),\tau^{\op{D},\star}(g,c_{h,i},x_{h,i}'))
        \end{equation*}
        and $\tau^{\op{D},\star}(g,c,x')=c+Z$, $Z\sim g(x',\pi_{\bar g}(x'))$ be the mle target.
        Note that if $c,x'$ are sampled conditional on $x,a$, then the target a sample of the random variable $\Tcal^{\op{D},\star}_hg(x,a)$.
        \State Get optimistic $p^k\gets\argmin_{p\in\Pcal_k}\min_a \bar p_1(x_{1},a)$.
        \State Let $\pi^k$ be greedy w.r.t. $\bar p^k$: $\pi^k_h(x)=\argmin_a \bar p^k_h(x,a),\forall h$.
        \State Gather data $\Dcal_k\gets{\text{\cref{alg:roll-in}}}(\pi^k, \text{UA flag})$.
    \EndFor
\end{algorithmic}
\end{algorithm}

A natural question is how can we achieve second-order bounds in RL?
In this section, we consider a distributional variant of the online RL algorithm that uses the mle loss to learn the cost-to-go distributions $Z^\star$.
RL algorithms that learn the cost-to-go distributions are often referred to as \emph{distributional RL} (DistRL) \citep{bdr2023} and have resulted in a plethora of empirical success \citep{bellemare2017distributional,dabney2018implicit,imani2018improving,bellemare2020autonomous,jason2022conservative,farebrother2024stop}.
Distributional losses, such as the mle loss and quantile regression loss, were initially motivated by improve representation learning and multi-task learning, but a theoretically rigorous explanation was an open question.
Recently, \citep{wang2023benefits,wang2024more} provided an answer to this mystery by proving that DistRL automatically yields first- and second-order bounds in RL, thus establishing the benefits of DistRL.

In this section, we review the results of \citep{wang2024more}, a refinement of \citep{wang2023benefits} that introduced the mle-loss variant of the optimistic online RL algorithm.
To learn the optimal policy's cost-to-go distributions $Z^\star$, we posit a conditional distribution class $\Pcal$ that consists of conditional distribution tuples $p=(p_1,p_2,\dots,p_H)\in\Pcal$ where $p_h:\Xcal\times\Acal\to\Delta([0,1])$.
We use the convention that $p_{H+1}$ is deterministic point-mass at $0$ for all conditional distributions $p$.
Then, \cref{alg:online-dist-rl} takes exactly the same structure as \cref{alg:online-rl} except that it performs a distributional variant of TD to solve the distributional Bellman equation $Z^\star_h\overset{D}{=}\Tcal_h^{\op{D}}Z^\star_{h+1}$.
It uses mle to learn the cost-to-go distributions and acts greedily with respect to the learned distribution's mean.

To ensure that distributional TD learning succeeds, we assume distributional BC (DistBC) \citep{wu2023distributional,wang2023benefits}.
\begin{assumption}[$\Tcal^{\op{D},\star}$-DistBC]\label{ass:dist-rl-bc}
$\Tcal_h^{\op{D},\star}p_{h+1}\in\Pcal_h$ for all $h\in[H]$ and $p_{h+1}\in\Pcal_{h+1}$.
\end{assumption}
\edit{Distributional BC posits that the distribution class is closed under the distributional Bellman operator and is a stronger condition than the standard completeness condition (\cref{ass:rl-bc}). Nevertheless, in low-rank MDPs with discrete cost distributions, \cref{ass:dist-rl-bc} can be satisfied with a linear distribution class \citep[Section 5.1]{wang2024more}, which also has bounded bracketing entropy of $\wt\Ocal(dM+\log|\Phi|)$ where $M$ is the number of discretizations.}

Next, we define the eluder dimension for mle loss.
Let $\Psi^{\op{mle}}_h=\{\delta^{\op{dis}}_h(\cdot;p):p\in\Pcal\}$ where
\begin{equation*}
    \textstyle\delta^{\op{dis}}_h(x,a;p):=h^2(p_h(x,a),\Tcal^{\op{D},\star}p_h(x,a))
\end{equation*}
and $\Psi^{\op{mle},\op{V}}_h=\{\EE_{a\sim\op{unif}(\Acal)}[\psi(x,a)]:\psi\in\Psi^{\op{mle}}_h\}$.
Define:
\begin{align*}
    &\textstyle d_{\op{mle}}=\max_{h\in[H]}\EluDim_1(\Psi_h^{\op{mle}},\Mcal_h,1/K),
    \\&\textstyle d_{\op{mle}}^{\op{V}}=\max_{h\in[H]}\EluDim_1(\Psi_h^{\op{mle},\op{V}},\Mcal_h^{\op{V}},1/K).
\end{align*}
The following is the main online RL result from \citep{wang2024more}.
\begin{theorem}\label{thm:online-rl-regret-mle-loss}
Under \cref{ass:dist-rl-bc}, for any $\delta\in(0,1)$, \wpal $1-\delta$, \cref{alg:online-dist-rl} with the mle loss $\ell_{\op{mle}}$ and $\beta=2\ln(H|\Pcal|/\delta)$ enjoys the following:
\begin{equation*}
    \textstyle\sum_{k=1}^K(V^{\pi^k}-V^\star)\leq\textstyle\wt\Ocal\prns*{ H\sqrt{{\color{red}\sum_{k=1}^K\sigma^2(\pi^k)}\cdot d\beta} + H^{2.5}d\beta },
\end{equation*}
where $d=d_{\op{mle}}$ if UA is false and $d=Ad_{\op{mle}}^{\op{V}}$ if UA is true.
\end{theorem}
The above mle loss bounds scales with the variances of the policies selected by the algorithm, and are thus are called second-order (a.k.a. variance dependent) bounds.
As we saw in the CSC setting, a second-order bound is actually strictly sharper than the first-order bound and this is also true in RL \citep[Theorem 2.1]{wang2024more}.
The variance bound can be much tighter in near-deterministic settings where the optimal policy's cost is far from zero.

\edit{One drawback of the DistRL approach is that it requires modeling the entire conditional distributions which are more complex than the conditional mean, \ie, $\Pcal$ is generally larger and more complex than $\Fcal$ from \cref{thm:online-rl-regret-squared-loss,thm:online-rl-regret-bce-loss}. Also, DistRL requires completeness w.r.t. the distributional Bellman operator, which is stronger than standard BC. However, it is worth noting that in practice DistRL is often much more performant than non-distributional approaches, which suggests that these modeling conditions (\eg, distributional completeness) are not so restrictive \citep{bellemare2017distributional,dabney2018implicit,farebrother2024stop}. Closing this gap between theory and practice is an interesting future direction. }

We also highlight two recent works on second-order bounds for contextual bandits, the one-step special case of RL where DistBC simplifies to distributional realizability (\cref{ass:csc-dist-realizability}). In this setting, \citet{pacchiano2024second} proved a second-order bound for contextual bandits with only mean realizability, which is weaker than distributional realizability, by using thresholded least squares.
Concurrently, \citet{jia2024does} obtained a similar bound as \citep{pacchiano2024second} for the strong adversary setting and a complementary bound for the weak adversary setting. \citet{jia2024does} also furnished a lower bound proving that the eluder-based second-order bounds in \cref{thm:online-rl-regret-mle-loss} and \citep{wang2024more} are tight and unimprovable even if the number of actions is less than the eluder dimension. 

\edit{\textbf{Computational considerations for bce and mle.} 
We remark that switching from squared loss to bce does not incur any computational overhead and is a single line of code in practice. The mle loss is used for distribution fitting and thus the function approximator should output a distribution instead of a scalar. In practice, the distribution can be modeled with histograms \citep{bellemare2017distributional,imani2018improving,farebrother2024stop} or quantiles \citep{dabney2018implicit} and is only a constant factor more to compute and maintain than the non-distributional losses. Thus, the mle loss and distributional RL more generally also do not incur any notable computational overhead in practice.
}

\begin{proof}[Proof of \cref{thm:online-rl-regret-mle-loss}]
We measure the distributional Bellman error of $p\in\Pcal$ under $\pi$ with the squared Hellinger distance:
\begin{align*}
    \delta^{\op{dis}}_h(\pi;p):=\EE_\pi[\delta^{\op{dis}}_h(x_h,a_h;p)],
\end{align*}
We define the mle excess risk $\Ecal^{\op{mle}}_h(p,\pi)$ as:
\begin{align*}
    \textstyle-\ln\EE_\pi[\exp(&\textstyle\frac12\ell_{\op{bce}}(\Tcal_h^{\op{D},\star}p_{h+1}(x_h,a_h),\tau^{\op{D},\star}(p,c_h,x_{h+1}))
    \\&\textstyle-\frac12\ell_{\op{bce}}(p_h(x_h,a_h),\tau^{\op{D},\star}(p,c_h,x_{h+1})))],
\end{align*}
Recall we have that $\delta_h^{\op{dis}}\leq \Ecal_h^{\op{mle}}$ by \cref{lem:exponentiated-excess-mle}.
We also write $\delta_{\op{dis}}^{\op{RL}}=\sum_{h=1}^H\delta^{\op{dis}}_h$ and $\Ecal_{\op{mle}}^{\op{RL}}=\sum_{h=1}^H\Ecal^{\op{mle}}_h$.
We now establish the optimism lemma, which is analogous to \cref{lem:sq-rl-optimism} and \cref{lem:bce-rl-optimism}.
\begin{restatable}{lemma}{MLERLOptimism}\label{lem:mle-rl-optimism}
Let $\ell=\ell_{\op{mle}}$ and $\Dcal_h$ be the same as in \cref{lem:sq-rl-optimism,lem:bce-rl-optimism}.
Then, under \cref{ass:dist-rl-bc}, for any $\delta\in(0,1)$ let $\beta=2\ln(H|\Pcal|/\delta)$ and define
\begin{equation*}
    \hat p^{\op{op}}\in\argmin_{p\in\Ccal^{\op{mle}}_\beta(\Dcal)}\min_a\bar p_1(x_1,a)
\end{equation*}
\Wpal $1-\delta$, we have (a) $\sum_{i=1}^n\Ecal_{\op{mle}}^{\op{RL}}(\hat p^{\op{op}},\pi^i)\leq 2H\beta$ and (b) $\min_a\overline{\hat p^{\op{op}}_1}(x_1,a)\leq V^\star$.
\end{restatable}
By \cref{lem:mle-rl-optimism}, we have $\sum_{k=1}^KV^{\pi^k}-V^\star\leq\sum_{k=1}^K V^{\pi^k}-\min_a \bar p^k_1(x_1,a)$.
Now we apply the second-order lemma:
\begin{align}
    &\textstyle\sum_{k=1}^K (V^{\pi^k}-\bar p^k_1(x_{1},\pi^k(x_{1}))) \label{eq:online-rl-regret-decomp-mle-loss}
    \\&=\textstyle\sum_{k=1}^K\sum_{h=1}^H\EE_{\pi^k}[\Tcal_h^\star\bar p^k_{h+1}(x_{h+1})-\bar p^k_h(x_h,a_h)] \notag
    \\&=\textstyle\sum_{k=1}^K\sum_{h=1}^H\EE_{\pi^k}[\overline{\Tcal_h^{\op{D},\star}p^k_{h+1}}(x_{h+1})-\bar p^k_h(x_h,a_h)] \notag
    \\&\leq\textstyle\sum_{k=1}^K\sqrt{\sum_{h=1}^H\EE_{\pi^k}[\sigma^2(p^k_h(x_h,a_h))]\cdot\delta_{\op{dis}}^{\op{RL}}(p^k,\pi^k)} \notag
    \\&\textstyle\phantom{\sum_{k=1}^K\sum_{h=1}^H}+\delta_{\op{dis}}^{\op{RL}}(p^k,\pi^k) \notag.
\end{align}
Now, we bound the variance term by $\sum_{h=1}^H\EE_{\pi^k}[\sigma^2(c_h+V^{\pi^k}_{h+1}(x_{h+1}))]$ plus some lower-order error terms.
We achieve this with the following lemma, which can be viewed as a variance analog of \cref{lem:self-bounding-bce}.
\begin{restatable}{lemma}{VarianceChangeOfMeasure}\label{lem:variance-change-of-measure}
Define the state-action analog of $\delta^{\op{dis}}_h$:
\begin{equation*}
    \textstyle\delta^{\op{dis}}_h(p,\pi,x_h,a_h):=h^2(p_h(x_h,a_h),\Tcal_h^{\op{D},\pi}p_{h+1}(x_h,a_h)),
\end{equation*}
where $\Tcal_h^{\op{D},\pi}p(x,a)\overset{D}{:=}C(x,a)+p(X',\pi(X'))$ is the distributional Bellman backup of $p$ under $\pi$.
Then, for any $p$, $\pi$, $x_h$, $a_h$, we have
\begin{align*}
    &\textstyle\sigma^2(p_h(x_h,a_h))\leq 2e\sigma^2(Z^\pi_h(x_h,a_h)) + H\delta_{\op{dis}}^{\op{RL}}(p,\pi),
\end{align*}
\edit{where $\sigma^2(Z^\pi_h(x_h,a_h))$ denotes the variance of the random variable $Z^\pi_h(x_h,a_h)$.}
This also implies the corollary:
\begin{equation*}
    \textstyle\EE_\pi[\sigma^2(p_h(x_h,a_h))]\lesssim \sigma^2(Z^\pi)+H^2\delta_{\op{dis}}^{\op{RL}}(p,\pi).
\end{equation*}
\end{restatable}
By \cref{lem:variance-change-of-measure}, we can bound \cref{eq:online-rl-regret-decomp-mle-loss} by
\begin{align*}
    &\lesssim\textstyle\sum_{k=1}^K\sqrt{H\sigma^2(Z^{\pi^k})\cdot\delta_{\op{dis}}^{\op{RL}}(p^k,\pi^k)}
    +H^{1.5}\delta_{\op{dis}}^{\op{RL}}(p^k,\pi^k)
    \\&\leq\textstyle\sqrt{H\sum_{k=1}^K\sigma^2(Z^{\pi^k})\cdot\sum_{k=1}^K\delta_{\op{dis}}^{\op{RL}}(p^k,\pi^k)}
    \\&\textstyle\phantom{\leq}+H^{1.5}\sum_{k=1}^K\delta_{\op{dis}}^{\op{RL}}(p^k,\pi^k).
\end{align*}
By \cref{lem:mle-rl-optimism} and the pigeonhole principle, the error terms $\Ecal_{\op{mle}}^{\op{RL}}$ can be bounded similarly as before: we can bound $\sum_{k=1}^K\delta_{\op{dis}}^{\op{RL}}(p^k,\pi^k)$ by $\wt\Ocal(d_{\op{mle}}H\beta)$ if UA is false, and by $\wt\Ocal(Ad_{\op{mle}}^{\op{V}}H\beta)$ if UA is true.
This finishes the proof of \cref{thm:online-rl-regret-mle-loss}.
\end{proof}

\section{Offline RL via Pessimistic Regression}
\begin{algorithm}[!t]
\caption{Pessimistic Offline RL}
\label{alg:offline-rl}
\begin{algorithmic}[1]
    \State\textbf{Input:} function class $\Fcal$, offline dataset $\Dcal$, loss function $\ell(\hat y,y)$, threshold $\beta$.
    \For{each policy $\pi\in\Pi$}
        \State Denote $\Fcal_\pi=\Ccal_\beta^\ell(\Dcal;\pi)$ as the version space defined by:
        \begin{align}
            \textstyle\Ccal_\beta^\ell(\Dcal;\pi)=\{f\in\Fcal:\,\, &\textstyle\forall h\in[H],\,L_h^\ell(f_h,f_{h+1},\Dcal_h,\pi) \notag
            \\&\textstyle-\min_{g_h\in\Fcal_h}L_h^\ell(g_h,f_{h+1},\Dcal_h,\pi)\leq\beta\}, \label{eq:offline-rl-confidence-set}
        \end{align}
        where
        \begin{equation*}
            \textstyle L_h^\ell(f_h,g,\Dcal_h,\pi)=\sum_{i=1}^{|\Dcal_h|} \ell(f_h(x_{h,i},a_{h,i}),\tau^\pi(g,c_{h,i},x_{h,i}'))
        \end{equation*}
        and $\tau^\pi(g,c,x')=c+g(x',\pi)$ is the regression target.
        In the proofs, we use $L^{\op{sq}}$ if $\ell=\ell_{\op{sq}}$ and $L^{\op{bce}}$ if $\ell=\ell_{\op{bce}}$.
        \State Get pessimistic $f^\pi\gets\argmax_{f\in\Fcal_\pi}\min_a f_1(x_1,a)$.
    \EndFor
    \State \textbf{Return:} $\wh\pi=\argmin_{\pi\in\Pi}\min_a f^{\pi}_1(x_1,a)$.
\end{algorithmic}
\end{algorithm}

In offline RL, we are given a dataset $\Dcal$ of size $N$ and the goal is to learn a good policy in a purely offline manner, without any interactions with the environment.
Since we cannot explore in offline RL, a natural strategy is to be cautious about any states and actions not covered by the given dataset -- that is, we should be conservative or pessimistic about unseen parts of the environment where we may make catastrophic errors \citep{kumar2020conservative,rashidinejad2021bridging,xie2021bellman}.
Indeed, it is intuitively clear that we can only hope to learn a good policy on the support of the given data. This will be soon formalized with the single-policy coverage coefficient. %

We summarize the offline RL algorithm in \cref{alg:offline-rl}.
We achieve pessimism by maximizing over the version space defined in \cref{eq:offline-rl-confidence-set}, which is an inversion of online RL which minimizes over a similar version space.
The only other difference is the regression target:
\begin{equation*}
    \textstyle\tau^\pi(f_{h+1},c,x')=c+f_{h+1}(x',\pi_{h+1}),
\end{equation*}
is an unbiased estimate of $\Tcal^\pi_h f_{h+1}$ in contrast to the online case where $\tau^\star$ was used to estimate $\Tcal^\star_h f_{h+1}$.
Thus, we instead use the policy-wise BC for offline RL:
\begin{assumption}[$\Tcal^\pi$-BC]\label{ass:offrl-bc}
$\Tcal^\pi_hf_{h+1}\in\Fcal_h$ for all $h\in[H]$, $f_{h+1}\in\Fcal_{h+1}$ and $\pi\in\Pi$.
\end{assumption}
Since we may take $\pi=\pi_f$, this is technically stronger than \cref{ass:rl-bc}.
Nevertheless, \cref{ass:offrl-bc} is also satisfied in low-rank MDPs by the linear function class $\Fcal^{\op{lin}}$ and so changing from \cref{ass:rl-bc} to \cref{ass:offrl-bc} does not change any conclusions we make.
As a historical remark, \cref{alg:offline-rl} was first proposed with the squared loss $\ell_{\op{sq}}$ under the name BCP by \citep{xie2021bellman} and then extended with the mle loss $\ell_{\op{mle}}$ under the name P-DISCO by \citep{wang2023benefits}.

We introduce the single-policy coverage coefficient: for any given comparator policy in the policy class $\wt\pi\in\Pi$, its coverage coefficient is defined by:
\begin{equation}
    \textstyle C^{\wt\pi} := \max_{h\in[H]}\max_{x,a}\frac{d^{\wt\pi}_h(x,a)}{\nu_h(x,a)}. \label{eq:single-policy-coverage}
\end{equation}
For simplicity, we set the policy class to all greedy policies induced by our function class $\Pi_\Fcal=\{\pi_f: f\in\Fcal\}$.\footnote{The offline RL results can be extended for general, infinite policy classes with log covering numbers \citep{cheng2022adversarially} or entropy integrals \citep{kallus2022doubly}.}
In the following theorem, the squared loss case recovers the results of \citep{xie2021bellman} and the bce loss result is new.
\begin{theorem}\label{thm:offline-rl}
Under \cref{ass:offrl-bc}, for any $\delta\in(0,1)$, \wpal $1-\delta$,
\cref{alg:offline-rl} with $\beta=2\ln(H|\Fcal|/\delta)$ has the following guarantees each loss function:
\begin{enumerate}
\item If $\ell=\ell_{\op{sq}}$, then for any comparator policy $\wt\pi\in\Pi_\Fcal$,
\begin{equation*}
    \textstyle V^{\wh\pi}-V^{\wt\pi}\leq \wt\Ocal\prns*{ H\sqrt{\frac{C^{\wt\pi}\beta}{n}} }.
\end{equation*}
\item If $\ell=\ell_{\op{bce}}$, then for any comparator policy $\wt\pi\in\Pi_\Fcal$,
\begin{equation*}
    \textstyle V^{\wh\pi}-V^{\wt\pi}\leq \wt\Ocal\prns*{ H\sqrt{{\color{red}V^{\wt\pi}}\cdot\frac{C^{\wt\pi}\beta}{n}} + H^2\frac{C^{\wt\pi}\beta}{n}}.
\end{equation*}
\end{enumerate}
\end{theorem}
We see that the squared loss algorithm always converges at a slow $\wt\Ocal(1/\sqrt{n})$ rate.
Simply changing the squared loss to the bce loss yields a first-order bound that converges at a fast $\wt\Ocal(1/n)$ rate in the small-cost regime where $V^{\wt\pi}\lesssim1/n$, and is never worse than the squared loss bound since $V^{\wt\pi}\leq 1$.
Again, the only change needed to achieve the improved bound is to change the loss function from squared loss to bce loss, which mirrors our observations from before.
One difference with the first-order online RL bound is that small-cost term here is $V^{\wt\pi}$ instead of $V^\star$.
Of course, we can set $\wt\pi=\pi^\star$ to recover the same small-cost term. However, this offline RL bound is more general since it can be applied to any comparator policy $\wt\pi$ with bounded coverage coefficient.

\begin{proof}[Proof of \cref{thm:offline-rl}]
We only prove the bce case as the squared loss case follows essentially the same structure.
The key difference compared to online RL is that we establish pessimism instead of optimism.
\begin{restatable}[Pessimism]{lemma}{BCERLPessimism}\label{lem:bce-rl-pessimism}
Let $\ell=\ell_{\op{bce}}$.
Under \cref{ass:rl-bc}, for any $\delta\in(0,1)$, setting $\beta=\Theta(\ln(H|\Fcal|/\delta))$.
Then, \wpal $1-\delta$, for all $\pi\in\Pi$, (a) $\Ecal^{\op{RL}}_{\op{bce}}(\hat f^\pi,\nu)\leq \frac{2H\beta}{n}$, and (b) $\min_a\hat f^\pi_1(x_1,a)\geq V^{\pi}$.
\end{restatable}
\begin{proof}[Proof of \cref{lem:bce-rl-pessimism}]
The proof is essentially identical to that of \cref{lem:bce-rl-optimism} where we show that \wpal $1-\delta$, (1) all elements of the version space have low excess risk and (2) $Q^\pi$ lies in the version space.
The only difference is that $\hat f^\pi$ is defined as the argmax rather than argmin, so that we have pessimism (greater than $V^\pi$) instead of optimism.
\end{proof}
By \cref{lem:bce-rl-pessimism}, we have $V^{\wh\pi}-V^{\wt\pi}\leq \min_a f_1^{\wh\pi}(x_1,a)-V^{\wt\pi}$.
Then, by definition of $\wh\pi$, we further bound this by $\min_a f_1^{\wt\pi}(x_1,a)-V^{\wt\pi}$.
Now, we decompose with PDL:
\begin{align*}
    &\textstyle\min_a f_1^{\wt\pi}(x_1,a)-V^{\wt\pi}
    \\&\textstyle=\sum_{h=1}^H\EE_{\wt\pi}[ f^{\wt\pi}_h(x_h,a_h)-\Tcal^{\wt\pi}_h f^{\wt\pi}_{h+1}(x_h,a_h) ]
    \\&\textstyle\leq\textstyle\sqrt{\sum_{h=1}^H\EE_{\wt\pi}[f^{\wt\pi}_h(x_h,a_h)]\cdot\delta^{\op{RL}}_{\op{Ber}}(f^{\wt\pi},\wt\pi) }
    +\delta^{\op{RL}}_{\op{Ber}}(f^{\wt\pi},\wt\pi)
    \\&\textstyle\lesssim\sqrt{HV^{\wt\pi}\cdot \delta^{\op{RL}}_{\op{Ber}}(f^{\wt\pi},\wt\pi)} + H\delta^{\op{RL}}_{\op{Ber}}(f^{\wt\pi},\wt\pi)
\end{align*}
By importance sampling and \cref{lem:bce-rl-pessimism}, the error terms can be bounded by $\wt\Ocal(C^{\wt\pi}\cdot \frac{H\beta}{n})$.
This completes the proof of \cref{thm:offline-rl}.
\end{proof}

\subsection{Second-Order Bounds via Distributional RL}
We now show that DistRL with the mle loss can yield second-order guarantees, recovering the results of \citep{wang2024more}. We make a few minor changes to the pessimistic offline RL algorithm.
We consider the set of greedy policies w.r.t. the means of the conditional distribution as our policy class $\Pi_\Pcal=\{\pi_{\bar p}:p\in\Pcal\}$.
\edit{Following \citet{xie2021bellman,wang2024more}, in offline RL, we posit the policy-wise distributional BC condition}.
\begin{assumption}[$\Tcal^{\op{D},\pi}$-DistBC]\label{ass:offrl-distbc}
$\Tcal^{\op{D},\pi}_hp_{h+1}\in\Pcal_h$ for all $h\in[H], p_{h+1}\in\Pcal_{h+1}$ and $\pi\in\Pi$.
\end{assumption}
\begin{algorithm}[!t]
\caption{Pessimistic Offline Distributional RL}
\label{alg:offline-distributional-rl}
\begin{algorithmic}[1]
    \State\textbf{Input:} conditional distribution class $\Pcal$, offline dataset $\Dcal$, threshold $\beta$.
    \For{each policy $\pi\in\Pi$}
        \State Denote $\Pcal_\pi^{\op{mle}}=\Ccal_\beta^{\op{mle}}(\Dcal;\pi)$ as the version space defined by:
        \begin{align}
            \textstyle\Ccal_\beta^{\op{mle}}(\Dcal;\pi)=\{p\in&\textstyle\Pcal:\,\,\forall h\in[H],\,L_h^{\op{mle}}(p_h,p_{h+1},\Dcal_h,\pi) \notag
            \\&\textstyle-\min_{g_h\in\Fcal_h}L_h^{\op{mle}}(g_h,p_{h+1},\Dcal_h,\pi)\leq\beta\}, \label{eq:offline-dist-rl-confidence-set}
        \end{align}
        where $L_h^{\op{mle}}(f_h,g,\Dcal_h,\pi)$ is
        \begin{equation*}
            \textstyle \sum_{i=1}^{|\Dcal_h|} \ell_{\op{mle}}(f_h(x_{h,i},a_{h,i}),\tau^{\op{D},\pi}(g,c_{h,i},x_{h,i}'))
        \end{equation*}
        and $\tau^{\op{D},\pi}(g,c,x')=c+Z,Z\sim g(x',\pi(x'))$ is the mle target.
        Note that if $c,x'$ are sampled conditional on $x,a$, then the target is a sample of the random variable $\Tcal^{\op{D},\pi}_hg(x,a)$.
        \State Get pessimistic $p^\pi\gets\argmax_{p\in\Pcal_\pi}\min_a \bar p_1(x_1,a)$.
    \EndFor
    \State \textbf{Return:} $\wh\pi=\argmin_{\pi\in\Pi}\min_a \bar p^{\pi}_1(x_1,a)$.
\end{algorithmic}
\end{algorithm}
\begin{theorem}\label{thm:offline-dist-rl}
Under \cref{ass:offrl-distbc}, for any $\delta\in(0,1)$, \wpal $1-\delta$,
\cref{alg:offline-distributional-rl} with $\beta=2\ln(H|\Pcal|/\delta)$ has the following guarantee: for any comparator policy $\wt\pi\in\Pi_\Pcal$,
\begin{equation*}
    \textstyle V^{\wh\pi}-V^{\wt\pi}\leq \wt\Ocal\prns*{ H\sqrt{{\color{red}\sigma^2(\wt\pi)}\cdot \frac{C^{\wt\pi}\beta}{n}} + H^{2.5}\frac{C^{\wt\pi}\beta}{n}}.
\end{equation*}
\end{theorem}
Since \edit{$\sigma^2(\wt\pi)\leq V^{\wt\pi}$}, this implies a first-order bound as well. This variance bound can be much tighter in near-deterministic settings where the comparator's variance is near zero, but its cost is far from zero.
However, as was the case in online RL, DistRL still has the drawbacks of requiring a distributional class and DistBC. While these are more stringent conditions in theory, DistRL has achieved state-of-the-art in many offline RL tasks as well \citep{ma2021conservative}, suggesting that the benefits of DistRL can outweight the stronger modeling assumptions in practice.
The proof of \cref{thm:offline-dist-rl}, which we omit due to space, follows from the same argument as the proof of \cref{thm:offline-rl}, coupled with the variance arguments from \cref{thm:online-rl-regret-mle-loss}.
The interested reader may find the full proof in \citep{wang2024more}.

\section{Computational Efficiency via Hybrid RL}\label{sec:hybrid-rl}
While we have exhibited the central role of loss functions, achieving tight variance-adaptive bounds, in both online and offline RL, one issue which we have not yet addressed is computational efficiency.
As mentioned earlier, optimizing over the version space is computationally difficult (NP-hard) even in tabular MDPs \citep{dann2018oracle}.

In this section, we discuss a solution via the hybrid RL setting where the learner can access an offline dataset $\Dcal^{\op{off}}$ with good coverage and also interact with the environment.
We show that Fitted-Q Iteration (FQI) \citep{munos2008finite}, a computationally efficient algorithm, can also enjoy first- and second-order guarantees by simply regressing with the bce and mle losses. The FQI algorithm in the hybrid setting was first proposed with the squared loss $\ell_{\op{sq}}$ by \citep{song2023hybrid} and our extensions to the bce and mle losses are novel.

Intuitively, the offline dataset mitigates the need for optimism, while the online interactions mitigate the need for pessimism -- together, they obviate the need for maintaining a version space.
For the following guarantees, we use $C^{\wt\pi}$ to denote the coverage coefficient of the comparator policy $\wt\pi$ under the data generating distribution of $\Dcal^{\op{off}}$.
We also assume the offline dataset to be as large as the number of interactions, \ie, $|\Dcal^{\op{off}}|\geq\Omega(K)$ \citep{song2023hybrid}.
\begin{algorithm}[!t]
\caption{Fitted $Q$-Iteration for Hybrid RL}
\label{alg:hybrid_fqi}
\begin{algorithmic}[1]
    \State\textbf{Input:} number of rounds $K$, function class $\Fcal$, offline dataset $\Dcal^{\op{off}}$, loss function $\ell(\hat y,y)$, uniform exploration (UA) flag
    \For{episode $k=1,2,\dots,K$}
        \For{each $h=H,H-1,\dots,1$}
            \State Recall the loss from \cref{alg:online-rl} (\cref{eq:rl-confidence-set}):
            \begin{equation*}
                \textstyle L_h^\ell(f_h,g,\Dcal_h)=\sum_{i=1}^{|\Dcal_h|} \ell(f_h(x_{h,i},a_{h,i}),\tau^\star(g,c_{h,i},x_{h,i}'))
            \end{equation*}
            \State Set $f^{k}_h = \argmin_{f_h\in\Fcal_h}L_{h}^\ell(f_h,f^k_{h+1},\Dcal_h^{\op{off}}\cup\Dcal^{\op{on}}_{<k})$.
        \EndFor
        \State Let $\pi^k$ be greedy w.r.t. $f^k$: $\pi^k_h(x)=\argmin_a f^k_h(x,a)$.
        \State Gather data $\Dcal_k^{\op{on}}\gets\text{\cref{alg:roll-in}}(\pi^k,\text{UA flag})$.
    \EndFor
\end{algorithmic}
\end{algorithm}

\begin{theorem}\label{thm:hybrid-rl}
Under \cref{ass:rl-bc} and $|\Dcal^{\op{off}}|\geq\Omega(K)$, for any $\delta\in(0,1)$, \wpal $1-\delta$, \cref{alg:hybrid_fqi} has the following guarantees for each loss function:
\begin{enumerate}
    \item If $\ell=\ell_{\op{sq}}$, for any comparator policy $\wt\pi\in\Pi_\Fcal$,
    \begin{equation*}
        \textstyle\sum_{k=1}^K(V^{\pi^k}-V^{\wt\pi})\leq \wt\Ocal(H\sqrt{K\cdot(d+C^{\wt\pi})\beta}),
    \end{equation*}
    where $d=d_{\op{sq}}$ if UA is false, and $d=Ad_{\op{sq}}^{\op{V}}$ if UA is true.
    \item If $\ell=\ell_{\op{bce}}$, for any comparator policy $\wt\pi\in\Pi_\Fcal$,
    \begin{align*}
        &\textstyle\sum_{k=1}^K(V^{\pi^k}-V^{\wt\pi})\leq \wt\Ocal\big(H\sqrt{{\color{red}V^{\wt\pi}K}\cdot(d+C^{\wt\pi})\beta}
        \\&\textstyle\qquad\qquad\qquad\qquad+H^2(d+C^{\wt\pi})\beta\big)
    \end{align*}
    where $d=d_{\op{bce}}$ if UA is false, and $d=Ad_{\op{bce}}^{\op{V}}$ if UA is true.
\end{enumerate}
\end{theorem}
Importantly, we see that simply changing the loss from $\ell_{\op{sq}}$ to $\ell_{\op{bce}}$ again leads to improved first-order bounds, which again supports our earlier observations.
Compared with our prior results, the main advantage of \cref{thm:hybrid-rl} is computational: it bounds the sub-optimality of a computationally efficient algorithm FQI, which much more closely resembles deep RL algorithms such as DQN \citep{mnih2015human}. 
From a statistical perspective, the hybrid RL bound is actually worse than either pure online or offline bounds, since it takes the form of:
\begin{equation*}
    \text{online RL bound} + \text{offline RL bound}.
\end{equation*}
Indeed, the hybrid RL bounds contain both the structural condition such as eluder dimension and the coverage coefficient $V^{\wt\pi}$.
This form will be made clear in the proof, which simply combines the prior online and offline RL results. We finally discuss some related works. 
\citep{ayoub2024switching} analyzed FQI with $\ell_{\op{bce}}$ in the pure offline setting and proved a first-order bound that depends on the much larger global coverage coefficient $C^{\Pi}=\max_{\wt\pi\in\Pi}C^{\wt\pi}$, which is needed to analyze FQI in the pure offline setting \citep{chen2019information}. \edit{Also, \citep{mhammedi2024efficient} is able to achieve computationally efficient learning low-rank MDPs without requiring offline data with good partial coverage; however, their bounds are neither first nor second-order. It would be interesting future work to adapt the techniques in this paper to derive variance dependent bounds for computationally efficient algorithms without requiring good offline data.}

\begin{proof}[Proof of \cref{thm:hybrid-rl}]
For any comparator policy $\wt\pi$, we decompose:
\begin{align*}
    \textstyle\sum_{k=1}^K(V^{\pi^k}-V^{\wt\pi})
    =&\textstyle\sum_{k=1}^K\EE[V^{\pi^k}-\min_a f^k_1(x_1,a)]
    \\&\textstyle+\EE[\min_a f^k_1(x_1,a)-V^{\wt\pi}]
\end{align*}
We see that the first term is exactly the same term in the online RL proof after we apply optimism (\eg, \cref{eq:online-rl-regret-decomp-squared-loss}); thus the first term is bounded by the online RL results, \eg, \cref{thm:online-rl-regret-squared-loss,thm:online-rl-regret-bce-loss}.
We also see that the second term is exactly the same term in the offline RL proof after apply pessimism. %
Thus, we can bound the second term by the offline RL results, \eg, \cref{thm:offline-rl}. Since we posit the offline dataset has as many samples as the online dataset, the offline bound matches the online one in terms of $K$. 
This completes the proof and shows why the bound in hybrid RL is the sum of online and offline RL bounds.
\end{proof}

\begin{algorithm}[!t]
\caption{Distributional FQI for Hybrid RL}
\label{alg:hybrid_dist_fqi}
\begin{algorithmic}[1]
    \State\textbf{Input:} number of rounds $K$, conditional distribution class $\Pcal$, offline dataset $\Dcal^{\op{off}}$, loss function $\ell(\hat y,y)$, uniform exploration (UA) flag
    \For{episode $k=1,2,\dots,K$}
        \For{each $h=H,H-1,\dots,1$}
            \State Recall the loss from \cref{alg:online-dist-rl} (\cref{eq:dist-rl-confidence-set}):
            \begin{equation*}
                \textstyle L_h^{\op{mle}}(p_h,g,\Dcal_h)=\sum_{i=1}^{|\Dcal_h|} \ell_{\op{mle}}(p_h(x_{h,i},a_{h,i}),\tau^{\op{D},\star}(g,c_{h,i},x_{h,i}'))
            \end{equation*}
            \State Set $p^{k}_h = \argmin_{p_h\in\Pcal_h}L_{h}^{\op{mle}}(p_h,p^k_{h+1},\Dcal_h^{\op{off}}\cup\Dcal^{\op{on}}_{<k})$.
        \EndFor
        \State Let $\pi^k$ be greedy w.r.t. $p^k$: $\pi^k_h(x)=\argmin_a \bar p^k_h(x,a)$.
        \State Gather data $\Dcal_k^{\op{on}}\gets\text{\cref{alg:roll-in}}(\pi^k,\text{UA flag})$.
    \EndFor
\end{algorithmic}
\end{algorithm}

Finally, to apply the mle loss to achieve second-order bounds, we naturally extend FQI with DistRL which closely resembles deep DistRL algorithms such as C51 \citep{bellemare2017distributional}.
This gives the following new second-order guarantees for hybrid RL.
\begin{theorem}
Under \cref{ass:dist-rl-bc} and $|\Dcal^{\op{off}}|\geq\Omega(K)$, for any $\delta\in(0,1)$, \wpal $1-\delta$, \cref{alg:hybrid_dist_fqi} has the following guarantee:
for any comparator policy $\wt\pi\in\Pi_\Pcal$,
\begin{align*}
    &\textstyle\sum_{k=1}^K(V^{\pi^k}-V^{\wt\pi})
    \leq \wt\Ocal\big(H^{2.5}(d+C^{\wt\pi})\beta
    \\&\textstyle+H\sqrt{{\color{red}(\sigma^2(\wt\pi) K+\sum_{k=1}^K\sigma^2(\pi^k))}\cdot(d+C^{\wt\pi})\beta}\big),
\end{align*}
where $d=d_{\op{mle}}$ if UA is false, and $d=Ad_{\op{mle}}^{\op{V}}$ if UA is true.
\end{theorem}
The hybrid second-order bound, being the sum of the second-order bounds for online and offline DistRL (\cref{thm:online-rl-regret-mle-loss,thm:offline-dist-rl}), contains both the variance of the played policies as well as the variance of the comparator policy.
Nevertheless, the hybrid second-order bound still implies a hybrid first-order bound by the same AM-GM argument as in CSC.
Thus, this again shows that DistRL yields a notable benefit compared to other losses.

\section{Discussion and Conclusion}
From the one-step CSC to online, offline and hybrid RL, we see time and time again that the loss function plays a central role in the adaptivity and efficiency of decision making algorithms.
The classical squared loss always converges at a slow $\wt\Ocal(1/\sqrt{n})$ rate and cannot adapt to easier problem instances with heteroskedasticity.
The bce loss can serve as a drop-in improvement that yields first-order bounds with a much faster $\wt\Ocal(1/n)$ rate when the optimal cost is small.
Switching from conditional-mean learning to conditional-distribution learning, the mle loss can tighten the bounds further with a second-order guarantee, that is bounds that converge at a $\wt\Ocal(1/n)$ rate in near-deterministic settings even if the optimal cost is large.
Crucially, these gaps in performance are not merely theoretical as they have been observed many times by the deep RL community \citep{farebrother2024stop,bdr2023,imani2018improving,ayoub2024switching,ma2021conservative}.
The theory outlined herein is very general and can be applied to a wide range of settings including imitation learning \citep{foster2024behavior}, model-based \citep{foster2021statistical,wang2024model}, risk-sensitive RL \citep{wang2023near,wang2024risk}, and robust bandits \citep{kallus2022doubly}, and RL \citep{bennett2024efficient}.
Moreover, the principles herein can improve algorithms for post-training large language models \citep{gao2024rebel,adler2024nemotron,wang2024conditioned,zhou2025q}, learning query optimizers \citep{krishnan2018learning,wang2024joingym} and  many more real-world applications.
We hope to have not only clearly demonstrated that the loss function choice is important in RL, but also to inspire the reader to seek out opportunities for better loss functions to improve their decision-making algorithms.

\section{Acknowledgments}
This material is based upon work supported by a Google PhD Fellowship and grants NSF IIS-1846210, NSF IIS-2154711, NSF CAREER 2339395 and DARPA Learning Network Cyberagents (LANCER). We also thank the editor and anonymous reviewers for useful discussions and feedback.
\bibliographystyle{imsart-nameyear}
\bibliography{main}

\begin{thebibliography}{77}

\bibitem[\protect\citeauthoryear{Adler et~al.}{2024}]{adler2024nemotron}
\begin{barticle}[author]
\bauthor{\bsnm{Adler},~\bfnm{Bo}\binits{B.}}, \bauthor{\bsnm{Agarwal},~\bfnm{Niket}\binits{N.}}, \bauthor{\bsnm{Aithal},~\bfnm{Ashwath}\binits{A.}}, \bauthor{\bsnm{Anh},~\bfnm{Dong~H}\binits{D.~H.}}, \bauthor{\bsnm{Bhattacharya},~\bfnm{Pallab}\binits{P.}}, \bauthor{\bsnm{Brundyn},~\bfnm{Annika}\binits{A.}}, \bauthor{\bsnm{Casper},~\bfnm{Jared}\binits{J.}}, \bauthor{\bsnm{Catanzaro},~\bfnm{Bryan}\binits{B.}}, \bauthor{\bsnm{Clay},~\bfnm{Sharon}\binits{S.}}, \bauthor{\bsnm{Cohen},~\bfnm{Jonathan}\binits{J.}} \betal{et~al.}
(\byear{2024}).
\btitle{Nemotron-4 340B Technical Report}.
\bjournal{arXiv preprint arXiv:2406.11704}.
\end{barticle}
\endbibitem

\bibitem[\protect\citeauthoryear{Agarwal et~al.}{2019}]{agarwal2019reinforcement}
\begin{barticle}[author]
\bauthor{\bsnm{Agarwal},~\bfnm{Alekh}\binits{A.}}, \bauthor{\bsnm{Jiang},~\bfnm{Nan}\binits{N.}}, \bauthor{\bsnm{Kakade},~\bfnm{Sham~M}\binits{S.~M.}} \AND \bauthor{\bsnm{Sun},~\bfnm{Wen}\binits{W.}}
(\byear{2019}).
\btitle{Reinforcement learning: Theory and algorithms}.
\bjournal{CS Dept., UW Seattle, Seattle, WA, USA, Tech. Rep}
\bvolume{32}
\bpages{96}.
\end{barticle}
\endbibitem

\bibitem[\protect\citeauthoryear{Agarwal et~al.}{2020}]{agarwal2020flambe}
\begin{barticle}[author]
\bauthor{\bsnm{Agarwal},~\bfnm{Alekh}\binits{A.}}, \bauthor{\bsnm{Kakade},~\bfnm{Sham}\binits{S.}}, \bauthor{\bsnm{Krishnamurthy},~\bfnm{Akshay}\binits{A.}} \AND \bauthor{\bsnm{Sun},~\bfnm{Wen}\binits{W.}}
(\byear{2020}).
\btitle{Flambe: Structural complexity and representation learning of low rank mdps}.
\bjournal{Advances in neural information processing systems}
\bvolume{33}
\bpages{20095--20107}.
\end{barticle}
\endbibitem

\bibitem[\protect\citeauthoryear{Agarwal et~al.}{2023}]{agarwal2023provable}
\begin{binproceedings}[author]
\bauthor{\bsnm{Agarwal},~\bfnm{Alekh}\binits{A.}}, \bauthor{\bsnm{Song},~\bfnm{Yuda}\binits{Y.}}, \bauthor{\bsnm{Sun},~\bfnm{Wen}\binits{W.}}, \bauthor{\bsnm{Wang},~\bfnm{Kaiwen}\binits{K.}}, \bauthor{\bsnm{Wang},~\bfnm{Mengdi}\binits{M.}} \AND \bauthor{\bsnm{Zhang},~\bfnm{Xuezhou}\binits{X.}}
(\byear{2023}).
\btitle{Provable benefits of representational transfer in reinforcement learning}.
In \bbooktitle{The Thirty Sixth Annual Conference on Learning Theory}
\bpages{2114--2187}.
\bpublisher{PMLR}.
\end{binproceedings}
\endbibitem

\bibitem[\protect\citeauthoryear{Audibert and Tsybakov}{2007}]{audibert2007fast}
\begin{barticle}[author]
\bauthor{\bsnm{Audibert},~\bfnm{Jean-Yves}\binits{J.-Y.}} \AND \bauthor{\bsnm{Tsybakov},~\bfnm{Alexandre~B}\binits{A.~B.}}
(\byear{2007}).
\btitle{Fast learning rates for plug-in classifiers}.
\end{barticle}
\endbibitem

\bibitem[\protect\citeauthoryear{Auer et~al.}{2002}]{auer2002nonstochastic}
\begin{barticle}[author]
\bauthor{\bsnm{Auer},~\bfnm{Peter}\binits{P.}}, \bauthor{\bsnm{Cesa-Bianchi},~\bfnm{Nicolo}\binits{N.}}, \bauthor{\bsnm{Freund},~\bfnm{Yoav}\binits{Y.}} \AND \bauthor{\bsnm{Schapire},~\bfnm{Robert~E}\binits{R.~E.}}
(\byear{2002}).
\btitle{The nonstochastic multiarmed bandit problem}.
\bjournal{SIAM journal on computing}
\bvolume{32}
\bpages{48--77}.
\end{barticle}
\endbibitem

\bibitem[\protect\citeauthoryear{Ayoub et~al.}{2024}]{ayoub2024switching}
\begin{binproceedings}[author]
\bauthor{\bsnm{Ayoub},~\bfnm{Alex}\binits{A.}}, \bauthor{\bsnm{Wang},~\bfnm{Kaiwen}\binits{K.}}, \bauthor{\bsnm{Liu},~\bfnm{Vincent}\binits{V.}}, \bauthor{\bsnm{Robertson},~\bfnm{Samuel}\binits{S.}}, \bauthor{\bsnm{McInerney},~\bfnm{James}\binits{J.}}, \bauthor{\bsnm{Liang},~\bfnm{Dawen}\binits{D.}}, \bauthor{\bsnm{Kallus},~\bfnm{Nathan}\binits{N.}} \AND \bauthor{\bsnm{Szepesvari},~\bfnm{Csaba}\binits{C.}}
(\byear{2024}).
\btitle{Switching the Loss Reduces the Cost in Batch Reinforcement Learning}.
In \bbooktitle{Forty-first International Conference on Machine Learning}.
\end{binproceedings}
\endbibitem

\bibitem[\protect\citeauthoryear{Ball et~al.}{2023}]{ball2023efficient}
\begin{binproceedings}[author]
\bauthor{\bsnm{Ball},~\bfnm{Philip~J}\binits{P.~J.}}, \bauthor{\bsnm{Smith},~\bfnm{Laura}\binits{L.}}, \bauthor{\bsnm{Kostrikov},~\bfnm{Ilya}\binits{I.}} \AND \bauthor{\bsnm{Levine},~\bfnm{Sergey}\binits{S.}}
(\byear{2023}).
\btitle{Efficient online reinforcement learning with offline data}.
In \bbooktitle{International Conference on Machine Learning}
\bpages{1577--1594}.
\bpublisher{PMLR}.
\end{binproceedings}
\endbibitem

\bibitem[\protect\citeauthoryear{Bas-Serrano et~al.}{2021}]{bas2021logistic}
\begin{binproceedings}[author]
\bauthor{\bsnm{Bas-Serrano},~\bfnm{Joan}\binits{J.}}, \bauthor{\bsnm{Curi},~\bfnm{Sebastian}\binits{S.}}, \bauthor{\bsnm{Krause},~\bfnm{Andreas}\binits{A.}} \AND \bauthor{\bsnm{Neu},~\bfnm{Gergely}\binits{G.}}
(\byear{2021}).
\btitle{Logistic Q-learning}.
In \bbooktitle{International conference on artificial intelligence and statistics}
\bpages{3610--3618}.
\bpublisher{PMLR}.
\end{binproceedings}
\endbibitem

\bibitem[\protect\citeauthoryear{Bellemare, Dabney and Munos}{2017}]{bellemare2017distributional}
\begin{binproceedings}[author]
\bauthor{\bsnm{Bellemare},~\bfnm{Marc~G}\binits{M.~G.}}, \bauthor{\bsnm{Dabney},~\bfnm{Will}\binits{W.}} \AND \bauthor{\bsnm{Munos},~\bfnm{R{\'e}mi}\binits{R.}}
(\byear{2017}).
\btitle{A distributional perspective on reinforcement learning}.
In \bbooktitle{International conference on machine learning}
\bpages{449--458}.
\bpublisher{PMLR}.
\end{binproceedings}
\endbibitem

\bibitem[\protect\citeauthoryear{Bellemare, Dabney and Rowland}{2023}]{bdr2023}
\begin{bbook}[author]
\bauthor{\bsnm{Bellemare},~\bfnm{Marc~G.}\binits{M.~G.}}, \bauthor{\bsnm{Dabney},~\bfnm{Will}\binits{W.}} \AND \bauthor{\bsnm{Rowland},~\bfnm{Mark}\binits{M.}}
(\byear{2023}).
\btitle{Distributional Reinforcement Learning}.
\bpublisher{MIT Press}
\bnote{\url{http://www.distributional-rl.org}}.
\end{bbook}
\endbibitem

\bibitem[\protect\citeauthoryear{Bellemare et~al.}{2020}]{bellemare2020autonomous}
\begin{barticle}[author]
\bauthor{\bsnm{Bellemare},~\bfnm{Marc~G}\binits{M.~G.}}, \bauthor{\bsnm{Candido},~\bfnm{Salvatore}\binits{S.}}, \bauthor{\bsnm{Castro},~\bfnm{Pablo~Samuel}\binits{P.~S.}}, \bauthor{\bsnm{Gong},~\bfnm{Jun}\binits{J.}}, \bauthor{\bsnm{Machado},~\bfnm{Marlos~C}\binits{M.~C.}}, \bauthor{\bsnm{Moitra},~\bfnm{Subhodeep}\binits{S.}}, \bauthor{\bsnm{Ponda},~\bfnm{Sameera~S}\binits{S.~S.}} \AND \bauthor{\bsnm{Wang},~\bfnm{Ziyu}\binits{Z.}}
(\byear{2020}).
\btitle{Autonomous navigation of stratospheric balloons using reinforcement learning}.
\bjournal{Nature}
\bvolume{588}
\bpages{77--82}.
\end{barticle}
\endbibitem

\bibitem[\protect\citeauthoryear{Bennett et~al.}{2024}]{bennett2024efficient}
\begin{barticle}[author]
\bauthor{\bsnm{Bennett},~\bfnm{Andrew}\binits{A.}}, \bauthor{\bsnm{Kallus},~\bfnm{Nathan}\binits{N.}}, \bauthor{\bsnm{Oprescu},~\bfnm{Miruna}\binits{M.}}, \bauthor{\bsnm{Sun},~\bfnm{Wen}\binits{W.}} \AND \bauthor{\bsnm{Wang},~\bfnm{Kaiwen}\binits{K.}}
(\byear{2024}).
\btitle{Efficient and Sharp Off-Policy Evaluation in Robust Markov Decision Processes}.
\bjournal{Advances in Neural Information Processing Systems}.
\end{barticle}
\endbibitem

\bibitem[\protect\citeauthoryear{Chang et~al.}{2022}]{chang2022learning}
\begin{binproceedings}[author]
\bauthor{\bsnm{Chang},~\bfnm{Jonathan}\binits{J.}}, \bauthor{\bsnm{Wang},~\bfnm{Kaiwen}\binits{K.}}, \bauthor{\bsnm{Kallus},~\bfnm{Nathan}\binits{N.}} \AND \bauthor{\bsnm{Sun},~\bfnm{Wen}\binits{W.}}
(\byear{2022}).
\btitle{Learning bellman complete representations for offline policy evaluation}.
In \bbooktitle{International Conference on Machine Learning}
\bpages{2938--2971}.
\bpublisher{PMLR}.
\end{binproceedings}
\endbibitem

\bibitem[\protect\citeauthoryear{Chen and Jiang}{2019}]{chen2019information}
\begin{binproceedings}[author]
\bauthor{\bsnm{Chen},~\bfnm{Jinglin}\binits{J.}} \AND \bauthor{\bsnm{Jiang},~\bfnm{Nan}\binits{N.}}
(\byear{2019}).
\btitle{Information-theoretic considerations in batch reinforcement learning}.
In \bbooktitle{International Conference on Machine Learning}
\bpages{1042--1051}.
\bpublisher{PMLR}.
\end{binproceedings}
\endbibitem

\bibitem[\protect\citeauthoryear{Cheng et~al.}{2022}]{cheng2022adversarially}
\begin{binproceedings}[author]
\bauthor{\bsnm{Cheng},~\bfnm{Ching-An}\binits{C.-A.}}, \bauthor{\bsnm{Xie},~\bfnm{Tengyang}\binits{T.}}, \bauthor{\bsnm{Jiang},~\bfnm{Nan}\binits{N.}} \AND \bauthor{\bsnm{Agarwal},~\bfnm{Alekh}\binits{A.}}
(\byear{2022}).
\btitle{Adversarially trained actor critic for offline reinforcement learning}.
In \bbooktitle{International Conference on Machine Learning}
\bpages{3852--3878}.
\bpublisher{PMLR}.
\end{binproceedings}
\endbibitem

\bibitem[\protect\citeauthoryear{Dabney et~al.}{2018}]{dabney2018implicit}
\begin{binproceedings}[author]
\bauthor{\bsnm{Dabney},~\bfnm{Will}\binits{W.}}, \bauthor{\bsnm{Ostrovski},~\bfnm{Georg}\binits{G.}}, \bauthor{\bsnm{Silver},~\bfnm{David}\binits{D.}} \AND \bauthor{\bsnm{Munos},~\bfnm{R{\'e}mi}\binits{R.}}
(\byear{2018}).
\btitle{Implicit quantile networks for distributional reinforcement learning}.
In \bbooktitle{International conference on machine learning}
\bpages{1096--1105}.
\bpublisher{PMLR}.
\end{binproceedings}
\endbibitem

\bibitem[\protect\citeauthoryear{Dann et~al.}{2018}]{dann2018oracle}
\begin{barticle}[author]
\bauthor{\bsnm{Dann},~\bfnm{Christoph}\binits{C.}}, \bauthor{\bsnm{Jiang},~\bfnm{Nan}\binits{N.}}, \bauthor{\bsnm{Krishnamurthy},~\bfnm{Akshay}\binits{A.}}, \bauthor{\bsnm{Agarwal},~\bfnm{Alekh}\binits{A.}}, \bauthor{\bsnm{Langford},~\bfnm{John}\binits{J.}} \AND \bauthor{\bsnm{Schapire},~\bfnm{Robert~E}\binits{R.~E.}}
(\byear{2018}).
\btitle{On oracle-efficient pac rl with rich observations}.
\bjournal{Advances in neural information processing systems}
\bvolume{31}.
\end{barticle}
\endbibitem

\bibitem[\protect\citeauthoryear{Dann et~al.}{2022}]{dann2022guarantees}
\begin{binproceedings}[author]
\bauthor{\bsnm{Dann},~\bfnm{Chris}\binits{C.}}, \bauthor{\bsnm{Mansour},~\bfnm{Yishay}\binits{Y.}}, \bauthor{\bsnm{Mohri},~\bfnm{Mehryar}\binits{M.}}, \bauthor{\bsnm{Sekhari},~\bfnm{Ayush}\binits{A.}} \AND \bauthor{\bsnm{Sridharan},~\bfnm{Karthik}\binits{K.}}
(\byear{2022}).
\btitle{Guarantees for epsilon-greedy reinforcement learning with function approximation}.
In \bbooktitle{International conference on machine learning}
\bpages{4666--4689}.
\bpublisher{PMLR}.
\end{binproceedings}
\endbibitem

\bibitem[\protect\citeauthoryear{Devroye and Lugosi}{2001}]{devroye2001combinatorial}
\begin{bbook}[author]
\bauthor{\bsnm{Devroye},~\bfnm{Luc}\binits{L.}} \AND \bauthor{\bsnm{Lugosi},~\bfnm{G{\'a}bor}\binits{G.}}
(\byear{2001}).
\btitle{Combinatorial methods in density estimation}.
\bpublisher{Springer Science \& Business Media}.
\end{bbook}
\endbibitem

\bibitem[\protect\citeauthoryear{Farebrother et~al.}{2024}]{farebrother2024stop}
\begin{binproceedings}[author]
\bauthor{\bsnm{Farebrother},~\bfnm{Jesse}\binits{J.}}, \bauthor{\bsnm{Orbay},~\bfnm{Jordi}\binits{J.}}, \bauthor{\bsnm{Vuong},~\bfnm{Quan}\binits{Q.}}, \bauthor{\bsnm{Taiga},~\bfnm{Adrien~Ali}\binits{A.~A.}}, \bauthor{\bsnm{Chebotar},~\bfnm{Yevgen}\binits{Y.}}, \bauthor{\bsnm{Xiao},~\bfnm{Ted}\binits{T.}}, \bauthor{\bsnm{Irpan},~\bfnm{Alex}\binits{A.}}, \bauthor{\bsnm{Levine},~\bfnm{Sergey}\binits{S.}}, \bauthor{\bsnm{Castro},~\bfnm{Pablo~Samuel}\binits{P.~S.}}, \bauthor{\bsnm{Faust},~\bfnm{Aleksandra}\binits{A.}}, \bauthor{\bsnm{Kumar},~\bfnm{Aviral}\binits{A.}} \AND \bauthor{\bsnm{Agarwal},~\bfnm{Rishabh}\binits{R.}}
(\byear{2024}).
\btitle{Stop Regressing: Training Value Functions via Classification for Scalable Deep {RL}}.
In \bbooktitle{Forty-first International Conference on Machine Learning}.
\end{binproceedings}
\endbibitem

\bibitem[\protect\citeauthoryear{Feng et~al.}{2021}]{feng2021provably}
\begin{binproceedings}[author]
\bauthor{\bsnm{Feng},~\bfnm{Fei}\binits{F.}}, \bauthor{\bsnm{Yin},~\bfnm{Wotao}\binits{W.}}, \bauthor{\bsnm{Agarwal},~\bfnm{Alekh}\binits{A.}} \AND \bauthor{\bsnm{Yang},~\bfnm{Lin}\binits{L.}}
(\byear{2021}).
\btitle{Provably correct optimization and exploration with non-linear policies}.
In \bbooktitle{International Conference on Machine Learning}
\bpages{3263--3273}.
\bpublisher{PMLR}.
\end{binproceedings}
\endbibitem

\bibitem[\protect\citeauthoryear{Foster, Block and Misra}{2024}]{foster2024behavior}
\begin{barticle}[author]
\bauthor{\bsnm{Foster},~\bfnm{Dylan~J}\binits{D.~J.}}, \bauthor{\bsnm{Block},~\bfnm{Adam}\binits{A.}} \AND \bauthor{\bsnm{Misra},~\bfnm{Dipendra}\binits{D.}}
(\byear{2024}).
\btitle{Is Behavior Cloning All You Need? Understanding Horizon in Imitation Learning}.
\bjournal{arXiv preprint arXiv:2407.15007}.
\end{barticle}
\endbibitem

\bibitem[\protect\citeauthoryear{Foster and Krishnamurthy}{2021}]{foster2021efficient}
\begin{barticle}[author]
\bauthor{\bsnm{Foster},~\bfnm{Dylan~J}\binits{D.~J.}} \AND \bauthor{\bsnm{Krishnamurthy},~\bfnm{Akshay}\binits{A.}}
(\byear{2021}).
\btitle{Efficient first-order contextual bandits: Prediction, allocation, and triangular discrimination}.
\bjournal{Advances in Neural Information Processing Systems}
\bvolume{34}
\bpages{18907--18919}.
\end{barticle}
\endbibitem

\bibitem[\protect\citeauthoryear{Foster et~al.}{2018}]{foster2018practical}
\begin{binproceedings}[author]
\bauthor{\bsnm{Foster},~\bfnm{Dylan}\binits{D.}}, \bauthor{\bsnm{Agarwal},~\bfnm{Alekh}\binits{A.}}, \bauthor{\bsnm{Dud{\'\i}k},~\bfnm{Miroslav}\binits{M.}}, \bauthor{\bsnm{Luo},~\bfnm{Haipeng}\binits{H.}} \AND \bauthor{\bsnm{Schapire},~\bfnm{Robert}\binits{R.}}
(\byear{2018}).
\btitle{Practical contextual bandits with regression oracles}.
In \bbooktitle{International Conference on Machine Learning}
\bpages{1539--1548}.
\bpublisher{PMLR}.
\end{binproceedings}
\endbibitem

\bibitem[\protect\citeauthoryear{Foster et~al.}{2021}]{foster2021statistical}
\begin{barticle}[author]
\bauthor{\bsnm{Foster},~\bfnm{Dylan~J}\binits{D.~J.}}, \bauthor{\bsnm{Kakade},~\bfnm{Sham~M}\binits{S.~M.}}, \bauthor{\bsnm{Qian},~\bfnm{Jian}\binits{J.}} \AND \bauthor{\bsnm{Rakhlin},~\bfnm{Alexander}\binits{A.}}
(\byear{2021}).
\btitle{The statistical complexity of interactive decision making}.
\bjournal{arXiv preprint arXiv:2112.13487}.
\end{barticle}
\endbibitem

\bibitem[\protect\citeauthoryear{Foster et~al.}{2022a}]{foster2022complexity}
\begin{barticle}[author]
\bauthor{\bsnm{Foster},~\bfnm{Dylan~J}\binits{D.~J.}}, \bauthor{\bsnm{Rakhlin},~\bfnm{Alexander}\binits{A.}}, \bauthor{\bsnm{Sekhari},~\bfnm{Ayush}\binits{A.}} \AND \bauthor{\bsnm{Sridharan},~\bfnm{Karthik}\binits{K.}}
(\byear{2022}a).
\btitle{On the complexity of adversarial decision making}.
\bjournal{Advances in Neural Information Processing Systems}
\bvolume{35}
\bpages{35404--35417}.
\end{barticle}
\endbibitem

\bibitem[\protect\citeauthoryear{Foster et~al.}{2022b}]{foster2022offline}
\begin{binproceedings}[author]
\bauthor{\bsnm{Foster},~\bfnm{Dylan~J}\binits{D.~J.}}, \bauthor{\bsnm{Krishnamurthy},~\bfnm{Akshay}\binits{A.}}, \bauthor{\bsnm{Simchi-Levi},~\bfnm{David}\binits{D.}} \AND \bauthor{\bsnm{Xu},~\bfnm{Yunzong}\binits{Y.}}
(\byear{2022}b).
\btitle{Offline Reinforcement Learning: Fundamental Barriers for Value Function Approximation}.
In \bbooktitle{Conference on Learning Theory}
\bpages{3489--3489}.
\bpublisher{PMLR}.
\end{binproceedings}
\endbibitem

\bibitem[\protect\citeauthoryear{Gao et~al.}{2024}]{gao2024rebel}
\begin{barticle}[author]
\bauthor{\bsnm{Gao},~\bfnm{Zhaolin}\binits{Z.}}, \bauthor{\bsnm{Chang},~\bfnm{Jonathan~D}\binits{J.~D.}}, \bauthor{\bsnm{Zhan},~\bfnm{Wenhao}\binits{W.}}, \bauthor{\bsnm{Oertell},~\bfnm{Owen}\binits{O.}}, \bauthor{\bsnm{Swamy},~\bfnm{Gokul}\binits{G.}}, \bauthor{\bsnm{Brantley},~\bfnm{Kiant{\'e}}\binits{K.}}, \bauthor{\bsnm{Joachims},~\bfnm{Thorsten}\binits{T.}}, \bauthor{\bsnm{Bagnell},~\bfnm{J~Andrew}\binits{J.~A.}}, \bauthor{\bsnm{Lee},~\bfnm{Jason~D}\binits{J.~D.}} \AND \bauthor{\bsnm{Sun},~\bfnm{Wen}\binits{W.}}
(\byear{2024}).
\btitle{Rebel: Reinforcement learning via regressing relative rewards}.
\bjournal{arXiv preprint arXiv:2404.16767}.
\end{barticle}
\endbibitem

\bibitem[\protect\citeauthoryear{Hu, Kallus and Mao}{2022}]{hu2022fast}
\begin{barticle}[author]
\bauthor{\bsnm{Hu},~\bfnm{Yichun}\binits{Y.}}, \bauthor{\bsnm{Kallus},~\bfnm{Nathan}\binits{N.}} \AND \bauthor{\bsnm{Mao},~\bfnm{Xiaojie}\binits{X.}}
(\byear{2022}).
\btitle{Fast rates for contextual linear optimization}.
\bjournal{Management Science}
\bvolume{68}
\bpages{4236--4245}.
\end{barticle}
\endbibitem

\bibitem[\protect\citeauthoryear{Imani and White}{2018}]{imani2018improving}
\begin{binproceedings}[author]
\bauthor{\bsnm{Imani},~\bfnm{Ehsan}\binits{E.}} \AND \bauthor{\bsnm{White},~\bfnm{Martha}\binits{M.}}
(\byear{2018}).
\btitle{Improving regression performance with distributional losses}.
In \bbooktitle{International conference on machine learning}
\bpages{2157--2166}.
\bpublisher{PMLR}.
\end{binproceedings}
\endbibitem

\bibitem[\protect\citeauthoryear{Jason~Ma, Jayaraman and Bastani}{2022}]{jason2022conservative}
\begin{barticle}[author]
\bauthor{\bsnm{Jason~Ma},~\bfnm{Yecheng}\binits{Y.}}, \bauthor{\bsnm{Jayaraman},~\bfnm{Dinesh}\binits{D.}} \AND \bauthor{\bsnm{Bastani},~\bfnm{Osbert}\binits{O.}}
(\byear{2022}).
\btitle{Conservative Offline Distributional Reinforcement Learning}.
\bjournal{Advances in neural information processing systems}.
\end{barticle}
\endbibitem

\bibitem[\protect\citeauthoryear{Jia et~al.}{2024}]{jia2024does}
\begin{barticle}[author]
\bauthor{\bsnm{Jia},~\bfnm{Zeyu}\binits{Z.}}, \bauthor{\bsnm{Qian},~\bfnm{Jian}\binits{J.}}, \bauthor{\bsnm{Rakhlin},~\bfnm{Alexander}\binits{A.}} \AND \bauthor{\bsnm{Wei},~\bfnm{Chen-Yu}\binits{C.-Y.}}
(\byear{2024}).
\btitle{How Does Variance Shape the Regret in Contextual Bandits?}
\bjournal{Advances in Neural Information Processing Systems}.
\end{barticle}
\endbibitem

\bibitem[\protect\citeauthoryear{Jiang and Agarwal}{2018}]{jiang2018open}
\begin{binproceedings}[author]
\bauthor{\bsnm{Jiang},~\bfnm{Nan}\binits{N.}} \AND \bauthor{\bsnm{Agarwal},~\bfnm{Alekh}\binits{A.}}
(\byear{2018}).
\btitle{Open problem: The dependence of sample complexity lower bounds on planning horizon}.
In \bbooktitle{Conference On Learning Theory}
\bpages{3395--3398}.
\bpublisher{PMLR}.
\end{binproceedings}
\endbibitem

\bibitem[\protect\citeauthoryear{Jin, Liu and Miryoosefi}{2021}]{jin2021bellman}
\begin{barticle}[author]
\bauthor{\bsnm{Jin},~\bfnm{Chi}\binits{C.}}, \bauthor{\bsnm{Liu},~\bfnm{Qinghua}\binits{Q.}} \AND \bauthor{\bsnm{Miryoosefi},~\bfnm{Sobhan}\binits{S.}}
(\byear{2021}).
\btitle{Bellman eluder dimension: New rich classes of rl problems, and sample-efficient algorithms}.
\bjournal{Advances in neural information processing systems}
\bvolume{34}
\bpages{13406--13418}.
\end{barticle}
\endbibitem

\bibitem[\protect\citeauthoryear{Jin et~al.}{2020}]{jin2020provably}
\begin{binproceedings}[author]
\bauthor{\bsnm{Jin},~\bfnm{Chi}\binits{C.}}, \bauthor{\bsnm{Yang},~\bfnm{Zhuoran}\binits{Z.}}, \bauthor{\bsnm{Wang},~\bfnm{Zhaoran}\binits{Z.}} \AND \bauthor{\bsnm{Jordan},~\bfnm{Michael~I}\binits{M.~I.}}
(\byear{2020}).
\btitle{Provably efficient reinforcement learning with linear function approximation}.
In \bbooktitle{Conference on learning theory}
\bpages{2137--2143}.
\bpublisher{PMLR}.
\end{binproceedings}
\endbibitem

\bibitem[\protect\citeauthoryear{Kakade and Langford}{2002}]{kakade2002approximately}
\begin{binproceedings}[author]
\bauthor{\bsnm{Kakade},~\bfnm{Sham}\binits{S.}} \AND \bauthor{\bsnm{Langford},~\bfnm{John}\binits{J.}}
(\byear{2002}).
\btitle{Approximately optimal approximate reinforcement learning}.
In \bbooktitle{Proceedings of the Nineteenth International Conference on Machine Learning}
\bpages{267--274}.
\end{binproceedings}
\endbibitem

\bibitem[\protect\citeauthoryear{Kallus et~al.}{2022}]{kallus2022doubly}
\begin{binproceedings}[author]
\bauthor{\bsnm{Kallus},~\bfnm{Nathan}\binits{N.}}, \bauthor{\bsnm{Mao},~\bfnm{Xiaojie}\binits{X.}}, \bauthor{\bsnm{Wang},~\bfnm{Kaiwen}\binits{K.}} \AND \bauthor{\bsnm{Zhou},~\bfnm{Zhengyuan}\binits{Z.}}
(\byear{2022}).
\btitle{Doubly robust distributionally robust off-policy evaluation and learning}.
In \bbooktitle{International Conference on Machine Learning}
\bpages{10598--10632}.
\bpublisher{PMLR}.
\end{binproceedings}
\endbibitem

\bibitem[\protect\citeauthoryear{Kolter}{2011}]{kolter2011fixed}
\begin{barticle}[author]
\bauthor{\bsnm{Kolter},~\bfnm{J}\binits{J.}}
(\byear{2011}).
\btitle{The fixed points of off-policy TD}.
\bjournal{Advances in neural information processing systems}
\bvolume{24}.
\end{barticle}
\endbibitem

\bibitem[\protect\citeauthoryear{Kontorovich}{2024}]{binomial_deviations}
\begin{bmisc}[author]
\bauthor{\bsnm{Kontorovich},~\bfnm{Aryeh}\binits{A.}}
(\byear{2024}).
\btitle{Binomial small deviations}.
\bnote{Tweet}.
\end{bmisc}
\endbibitem

\bibitem[\protect\citeauthoryear{Krishnamurthy et~al.}{2019}]{krishnamurthy2019active}
\begin{barticle}[author]
\bauthor{\bsnm{Krishnamurthy},~\bfnm{Akshay}\binits{A.}}, \bauthor{\bsnm{Agarwal},~\bfnm{Alekh}\binits{A.}}, \bauthor{\bsnm{Huang},~\bfnm{Tzu-Kuo}\binits{T.-K.}}, \bauthor{\bsnm{Daum{\'e}~III},~\bfnm{Hal}\binits{H.}} \AND \bauthor{\bsnm{Langford},~\bfnm{John}\binits{J.}}
(\byear{2019}).
\btitle{Active learning for cost-sensitive classification}.
\bjournal{Journal of Machine Learning Research}
\bvolume{20}
\bpages{1--50}.
\end{barticle}
\endbibitem

\bibitem[\protect\citeauthoryear{Krishnan et~al.}{2018}]{krishnan2018learning}
\begin{barticle}[author]
\bauthor{\bsnm{Krishnan},~\bfnm{Sanjay}\binits{S.}}, \bauthor{\bsnm{Yang},~\bfnm{Zongheng}\binits{Z.}}, \bauthor{\bsnm{Goldberg},~\bfnm{Ken}\binits{K.}}, \bauthor{\bsnm{Hellerstein},~\bfnm{Joseph}\binits{J.}} \AND \bauthor{\bsnm{Stoica},~\bfnm{Ion}\binits{I.}}
(\byear{2018}).
\btitle{Learning to optimize join queries with deep reinforcement learning}.
\bjournal{arXiv preprint arXiv:1808.03196}.
\end{barticle}
\endbibitem

\bibitem[\protect\citeauthoryear{Kumar et~al.}{2020}]{kumar2020conservative}
\begin{barticle}[author]
\bauthor{\bsnm{Kumar},~\bfnm{Aviral}\binits{A.}}, \bauthor{\bsnm{Zhou},~\bfnm{Aurick}\binits{A.}}, \bauthor{\bsnm{Tucker},~\bfnm{George}\binits{G.}} \AND \bauthor{\bsnm{Levine},~\bfnm{Sergey}\binits{S.}}
(\byear{2020}).
\btitle{Conservative q-learning for offline reinforcement learning}.
\bjournal{Advances in Neural Information Processing Systems}
\bvolume{33}
\bpages{1179--1191}.
\end{barticle}
\endbibitem

\bibitem[\protect\citeauthoryear{Liu et~al.}{2022}]{liu2022partially}
\begin{binproceedings}[author]
\bauthor{\bsnm{Liu},~\bfnm{Qinghua}\binits{Q.}}, \bauthor{\bsnm{Chung},~\bfnm{Alan}\binits{A.}}, \bauthor{\bsnm{Szepesv{\'a}ri},~\bfnm{Csaba}\binits{C.}} \AND \bauthor{\bsnm{Jin},~\bfnm{Chi}\binits{C.}}
(\byear{2022}).
\btitle{When is partially observable reinforcement learning not scary?}
In \bbooktitle{Conference on Learning Theory}
\bpages{5175--5220}.
\bpublisher{PMLR}.
\end{binproceedings}
\endbibitem

\bibitem[\protect\citeauthoryear{Luedtke and Chambaz}{2017}]{luedtke2017faster}
\begin{barticle}[author]
\bauthor{\bsnm{Luedtke},~\bfnm{Alexander}\binits{A.}} \AND \bauthor{\bsnm{Chambaz},~\bfnm{Antoine}\binits{A.}}
(\byear{2017}).
\btitle{Faster rates for policy learning}.
\bjournal{arXiv preprint arXiv:1704.06431}.
\end{barticle}
\endbibitem

\bibitem[\protect\citeauthoryear{Lykouris, Sridharan and Tardos}{2018}]{lykouris2018small}
\begin{binproceedings}[author]
\bauthor{\bsnm{Lykouris},~\bfnm{Thodoris}\binits{T.}}, \bauthor{\bsnm{Sridharan},~\bfnm{Karthik}\binits{K.}} \AND \bauthor{\bsnm{Tardos},~\bfnm{{\'E}va}\binits{{\'E}.}}
(\byear{2018}).
\btitle{Small-loss bounds for online learning with partial information}.
In \bbooktitle{Conference on Learning Theory}
\bpages{979--986}.
\bpublisher{PMLR}.
\end{binproceedings}
\endbibitem

\bibitem[\protect\citeauthoryear{Ma, Jayaraman and Bastani}{2021}]{ma2021conservative}
\begin{barticle}[author]
\bauthor{\bsnm{Ma},~\bfnm{Yecheng}\binits{Y.}}, \bauthor{\bsnm{Jayaraman},~\bfnm{Dinesh}\binits{D.}} \AND \bauthor{\bsnm{Bastani},~\bfnm{Osbert}\binits{O.}}
(\byear{2021}).
\btitle{Conservative offline distributional reinforcement learning}.
\bjournal{Advances in neural information processing systems}
\bvolume{34}
\bpages{19235--19247}.
\end{barticle}
\endbibitem

\bibitem[\protect\citeauthoryear{Mhammedi, Foster and Rakhlin}{2024}]{mhammedi2024power}
\begin{barticle}[author]
\bauthor{\bsnm{Mhammedi},~\bfnm{Zakaria}\binits{Z.}}, \bauthor{\bsnm{Foster},~\bfnm{Dylan~J}\binits{D.~J.}} \AND \bauthor{\bsnm{Rakhlin},~\bfnm{Alexander}\binits{A.}}
(\byear{2024}).
\btitle{The Power of Resets in Online Reinforcement Learning}.
\bjournal{arXiv preprint arXiv:2404.15417}.
\end{barticle}
\endbibitem

\bibitem[\protect\citeauthoryear{Mhammedi et~al.}{2024}]{mhammedi2024efficient}
\begin{barticle}[author]
\bauthor{\bsnm{Mhammedi},~\bfnm{Zak}\binits{Z.}}, \bauthor{\bsnm{Block},~\bfnm{Adam}\binits{A.}}, \bauthor{\bsnm{Foster},~\bfnm{Dylan~J}\binits{D.~J.}} \AND \bauthor{\bsnm{Rakhlin},~\bfnm{Alexander}\binits{A.}}
(\byear{2024}).
\btitle{Efficient model-free exploration in low-rank mdps}.
\bjournal{Advances in Neural Information Processing Systems}
\bvolume{36}.
\end{barticle}
\endbibitem

\bibitem[\protect\citeauthoryear{Misra et~al.}{2020}]{misra2020kinematic}
\begin{binproceedings}[author]
\bauthor{\bsnm{Misra},~\bfnm{Dipendra}\binits{D.}}, \bauthor{\bsnm{Henaff},~\bfnm{Mikael}\binits{M.}}, \bauthor{\bsnm{Krishnamurthy},~\bfnm{Akshay}\binits{A.}} \AND \bauthor{\bsnm{Langford},~\bfnm{John}\binits{J.}}
(\byear{2020}).
\btitle{Kinematic state abstraction and provably efficient rich-observation reinforcement learning}.
In \bbooktitle{International conference on machine learning}
\bpages{6961--6971}.
\bpublisher{PMLR}.
\end{binproceedings}
\endbibitem

\bibitem[\protect\citeauthoryear{Mnih et~al.}{2015}]{mnih2015human}
\begin{barticle}[author]
\bauthor{\bsnm{Mnih},~\bfnm{Volodymyr}\binits{V.}}, \bauthor{\bsnm{Kavukcuoglu},~\bfnm{Koray}\binits{K.}}, \bauthor{\bsnm{Silver},~\bfnm{David}\binits{D.}}, \bauthor{\bsnm{Rusu},~\bfnm{Andrei~A}\binits{A.~A.}}, \bauthor{\bsnm{Veness},~\bfnm{Joel}\binits{J.}}, \bauthor{\bsnm{Bellemare},~\bfnm{Marc~G}\binits{M.~G.}}, \bauthor{\bsnm{Graves},~\bfnm{Alex}\binits{A.}}, \bauthor{\bsnm{Riedmiller},~\bfnm{Martin}\binits{M.}}, \bauthor{\bsnm{Fidjeland},~\bfnm{Andreas~K}\binits{A.~K.}}, \bauthor{\bsnm{Ostrovski},~\bfnm{Georg}\binits{G.}} \betal{et~al.}
(\byear{2015}).
\btitle{Human-level control through deep reinforcement learning}.
\bjournal{nature}
\bvolume{518}
\bpages{529--533}.
\end{barticle}
\endbibitem

\bibitem[\protect\citeauthoryear{Modi et~al.}{2024}]{modi2024model}
\begin{barticle}[author]
\bauthor{\bsnm{Modi},~\bfnm{Aditya}\binits{A.}}, \bauthor{\bsnm{Chen},~\bfnm{Jinglin}\binits{J.}}, \bauthor{\bsnm{Krishnamurthy},~\bfnm{Akshay}\binits{A.}}, \bauthor{\bsnm{Jiang},~\bfnm{Nan}\binits{N.}} \AND \bauthor{\bsnm{Agarwal},~\bfnm{Alekh}\binits{A.}}
(\byear{2024}).
\btitle{Model-free representation learning and exploration in low-rank mdps}.
\bjournal{Journal of Machine Learning Research}
\bvolume{25}
\bpages{1--76}.
\end{barticle}
\endbibitem

\bibitem[\protect\citeauthoryear{Munos and Szepesv{\'a}ri}{2008}]{munos2008finite}
\begin{barticle}[author]
\bauthor{\bsnm{Munos},~\bfnm{R{\'e}mi}\binits{R.}} \AND \bauthor{\bsnm{Szepesv{\'a}ri},~\bfnm{Csaba}\binits{C.}}
(\byear{2008}).
\btitle{Finite-Time Bounds for Fitted Value Iteration.}
\bjournal{Journal of Machine Learning Research}
\bvolume{9}.
\end{barticle}
\endbibitem

\bibitem[\protect\citeauthoryear{Ouyang et~al.}{2022}]{ouyang2022training}
\begin{barticle}[author]
\bauthor{\bsnm{Ouyang},~\bfnm{Long}\binits{L.}}, \bauthor{\bsnm{Wu},~\bfnm{Jeffrey}\binits{J.}}, \bauthor{\bsnm{Jiang},~\bfnm{Xu}\binits{X.}}, \bauthor{\bsnm{Almeida},~\bfnm{Diogo}\binits{D.}}, \bauthor{\bsnm{Wainwright},~\bfnm{Carroll}\binits{C.}}, \bauthor{\bsnm{Mishkin},~\bfnm{Pamela}\binits{P.}}, \bauthor{\bsnm{Zhang},~\bfnm{Chong}\binits{C.}}, \bauthor{\bsnm{Agarwal},~\bfnm{Sandhini}\binits{S.}}, \bauthor{\bsnm{Slama},~\bfnm{Katarina}\binits{K.}}, \bauthor{\bsnm{Ray},~\bfnm{Alex}\binits{A.}} \betal{et~al.}
(\byear{2022}).
\btitle{Training language models to follow instructions with human feedback}.
\bjournal{Advances in neural information processing systems}
\bvolume{35}
\bpages{27730--27744}.
\end{barticle}
\endbibitem

\bibitem[\protect\citeauthoryear{Pacchiano}{2024}]{pacchiano2024second}
\begin{barticle}[author]
\bauthor{\bsnm{Pacchiano},~\bfnm{Aldo}\binits{A.}}
(\byear{2024}).
\btitle{Second Order Bounds for Contextual Bandits with Function Approximation}.
\bjournal{arXiv preprint arXiv:2409.16197}.
\end{barticle}
\endbibitem

\bibitem[\protect\citeauthoryear{Rashidinejad et~al.}{2021}]{rashidinejad2021bridging}
\begin{barticle}[author]
\bauthor{\bsnm{Rashidinejad},~\bfnm{Paria}\binits{P.}}, \bauthor{\bsnm{Zhu},~\bfnm{Banghua}\binits{B.}}, \bauthor{\bsnm{Ma},~\bfnm{Cong}\binits{C.}}, \bauthor{\bsnm{Jiao},~\bfnm{Jiantao}\binits{J.}} \AND \bauthor{\bsnm{Russell},~\bfnm{Stuart}\binits{S.}}
(\byear{2021}).
\btitle{Bridging offline reinforcement learning and imitation learning: A tale of pessimism}.
\bjournal{Advances in Neural Information Processing Systems}
\bvolume{34}
\bpages{11702--11716}.
\end{barticle}
\endbibitem

\bibitem[\protect\citeauthoryear{Russo and Van~Roy}{2013}]{russo2013eluder}
\begin{barticle}[author]
\bauthor{\bsnm{Russo},~\bfnm{Daniel}\binits{D.}} \AND \bauthor{\bsnm{Van~Roy},~\bfnm{Benjamin}\binits{B.}}
(\byear{2013}).
\btitle{Eluder dimension and the sample complexity of optimistic exploration}.
\bjournal{Advances in Neural Information Processing Systems}
\bvolume{26}.
\end{barticle}
\endbibitem

\bibitem[\protect\citeauthoryear{Song et~al.}{2023}]{song2023hybrid}
\begin{binproceedings}[author]
\bauthor{\bsnm{Song},~\bfnm{Yuda}\binits{Y.}}, \bauthor{\bsnm{Zhou},~\bfnm{Yifei}\binits{Y.}}, \bauthor{\bsnm{Sekhari},~\bfnm{Ayush}\binits{A.}}, \bauthor{\bsnm{Bagnell},~\bfnm{Drew}\binits{D.}}, \bauthor{\bsnm{Krishnamurthy},~\bfnm{Akshay}\binits{A.}} \AND \bauthor{\bsnm{Sun},~\bfnm{Wen}\binits{W.}}
(\byear{2023}).
\btitle{Hybrid {RL}: Using both offline and online data can make {RL} efficient}.
In \bbooktitle{The Eleventh International Conference on Learning Representations}.
\end{binproceedings}
\endbibitem

\bibitem[\protect\citeauthoryear{Tsitsiklis and Van~Roy}{1996}]{tsitsiklis1996analysis}
\begin{barticle}[author]
\bauthor{\bsnm{Tsitsiklis},~\bfnm{John}\binits{J.}} \AND \bauthor{\bsnm{Van~Roy},~\bfnm{Benjamin}\binits{B.}}
(\byear{1996}).
\btitle{Analysis of temporal-diffference learning with function approximation}.
\bjournal{Advances in neural information processing systems}
\bvolume{9}.
\end{barticle}
\endbibitem

\bibitem[\protect\citeauthoryear{Uehara and Sun}{2022}]{uehara2022pessimistic}
\begin{binproceedings}[author]
\bauthor{\bsnm{Uehara},~\bfnm{Masatoshi}\binits{M.}} \AND \bauthor{\bsnm{Sun},~\bfnm{Wen}\binits{W.}}
(\byear{2022}).
\btitle{Pessimistic Model-based Offline Reinforcement Learning under Partial Coverage}.
In \bbooktitle{International Conference on Learning Representations}.
\end{binproceedings}
\endbibitem

\bibitem[\protect\citeauthoryear{Wainwright}{2019}]{wainwright2019high}
\begin{bbook}[author]
\bauthor{\bsnm{Wainwright},~\bfnm{Martin~J}\binits{M.~J.}}
(\byear{2019}).
\btitle{High-dimensional statistics: A non-asymptotic viewpoint}
\bvolume{48}.
\bpublisher{Cambridge university press}.
\end{bbook}
\endbibitem

\bibitem[\protect\citeauthoryear{Wang, Foster and Kakade}{2021}]{wang2021what}
\begin{binproceedings}[author]
\bauthor{\bsnm{Wang},~\bfnm{Ruosong}\binits{R.}}, \bauthor{\bsnm{Foster},~\bfnm{Dean}\binits{D.}} \AND \bauthor{\bsnm{Kakade},~\bfnm{Sham~M.}\binits{S.~M.}}
(\byear{2021}).
\btitle{What are the Statistical Limits of Offline {RL} with Linear Function Approximation?}
In \bbooktitle{International Conference on Learning Representations}.
\end{binproceedings}
\endbibitem

\bibitem[\protect\citeauthoryear{Wang, Kallus and Sun}{2023}]{wang2023near}
\begin{binproceedings}[author]
\bauthor{\bsnm{Wang},~\bfnm{Kaiwen}\binits{K.}}, \bauthor{\bsnm{Kallus},~\bfnm{Nathan}\binits{N.}} \AND \bauthor{\bsnm{Sun},~\bfnm{Wen}\binits{W.}}
(\byear{2023}).
\btitle{Near-minimax-optimal risk-sensitive reinforcement learning with cvar}.
In \bbooktitle{International Conference on Machine Learning}
\bpages{35864--35907}.
\bpublisher{PMLR}.
\end{binproceedings}
\endbibitem

\bibitem[\protect\citeauthoryear{Wang et~al.}{2023}]{wang2023benefits}
\begin{barticle}[author]
\bauthor{\bsnm{Wang},~\bfnm{Kaiwen}\binits{K.}}, \bauthor{\bsnm{Zhou},~\bfnm{Kevin}\binits{K.}}, \bauthor{\bsnm{Wu},~\bfnm{Runzhe}\binits{R.}}, \bauthor{\bsnm{Kallus},~\bfnm{Nathan}\binits{N.}} \AND \bauthor{\bsnm{Sun},~\bfnm{Wen}\binits{W.}}
(\byear{2023}).
\btitle{The benefits of being distributional: Small-loss bounds for reinforcement learning}.
\bjournal{Advances in Neural Information Processing Systems}
\bvolume{36}.
\end{barticle}
\endbibitem

\bibitem[\protect\citeauthoryear{Wang et~al.}{2024a}]{wang2024more}
\begin{barticle}[author]
\bauthor{\bsnm{Wang},~\bfnm{Kaiwen}\binits{K.}}, \bauthor{\bsnm{Oertell},~\bfnm{Owen}\binits{O.}}, \bauthor{\bsnm{Agarwal},~\bfnm{Alekh}\binits{A.}}, \bauthor{\bsnm{Kallus},~\bfnm{Nathan}\binits{N.}} \AND \bauthor{\bsnm{Sun},~\bfnm{Wen}\binits{W.}}
(\byear{2024}a).
\btitle{More Benefits of Being Distributional: Second-Order Bounds for Reinforcement Learning}.
\bjournal{International Conference on Machine Learning}.
\end{barticle}
\endbibitem

\bibitem[\protect\citeauthoryear{Wang et~al.}{2024b}]{wang2024model}
\begin{barticle}[author]
\bauthor{\bsnm{Wang},~\bfnm{Zhiyong}\binits{Z.}}, \bauthor{\bsnm{Zhou},~\bfnm{Dongruo}\binits{D.}}, \bauthor{\bsnm{Lui},~\bfnm{John}\binits{J.}} \AND \bauthor{\bsnm{Sun},~\bfnm{Wen}\binits{W.}}
(\byear{2024}b).
\btitle{Model-based RL as a Minimalist Approach to Horizon-Free and Second-Order Bounds}.
\bjournal{arXiv preprint arXiv:2408.08994}.
\end{barticle}
\endbibitem

\bibitem[\protect\citeauthoryear{Wang et~al.}{2024c}]{wang2024risk}
\begin{barticle}[author]
\bauthor{\bsnm{Wang},~\bfnm{Kaiwen}\binits{K.}}, \bauthor{\bsnm{Liang},~\bfnm{Dawen}\binits{D.}}, \bauthor{\bsnm{Kallus},~\bfnm{Nathan}\binits{N.}} \AND \bauthor{\bsnm{Sun},~\bfnm{Wen}\binits{W.}}
(\byear{2024}c).
\btitle{A Reductions Approach to Risk-Sensitive Reinforcement Learning with Optimized Certainty Equivalents}.
\bjournal{arXiv preprint arXiv:2403.06323}.
\end{barticle}
\endbibitem

\bibitem[\protect\citeauthoryear{Wang et~al.}{2024d}]{wang2024conditioned}
\begin{barticle}[author]
\bauthor{\bsnm{Wang},~\bfnm{Kaiwen}\binits{K.}}, \bauthor{\bsnm{Kidambi},~\bfnm{Rahul}\binits{R.}}, \bauthor{\bsnm{Sullivan},~\bfnm{Ryan}\binits{R.}}, \bauthor{\bsnm{Agarwal},~\bfnm{Alekh}\binits{A.}}, \bauthor{\bsnm{Dann},~\bfnm{Christoph}\binits{C.}}, \bauthor{\bsnm{Michi},~\bfnm{Andrea}\binits{A.}}, \bauthor{\bsnm{Gelmi},~\bfnm{Marco}\binits{M.}}, \bauthor{\bsnm{Li},~\bfnm{Yunxuan}\binits{Y.}}, \bauthor{\bsnm{Gupta},~\bfnm{Raghav}\binits{R.}}, \bauthor{\bsnm{Dubey},~\bfnm{Avinava}\binits{A.}} \betal{et~al.}
(\byear{2024}d).
\btitle{Conditional Language Policy: A General Framework for Steerable Multi-Objective Finetuning}.
\bjournal{Findings of Empirical Methods in Natural Language Processing}.
\end{barticle}
\endbibitem

\bibitem[\protect\citeauthoryear{Wang et~al.}{2024e}]{wang2024joingym}
\begin{barticle}[author]
\bauthor{\bsnm{Wang},~\bfnm{Junxiong}\binits{J.}}, \bauthor{\bsnm{Wang},~\bfnm{Kaiwen}\binits{K.}}, \bauthor{\bsnm{Li},~\bfnm{Yueying}\binits{Y.}}, \bauthor{\bsnm{Kallus},~\bfnm{Nathan}\binits{N.}}, \bauthor{\bsnm{Trummer},~\bfnm{Immanuel}\binits{I.}} \AND \bauthor{\bsnm{Sun},~\bfnm{Wen}\binits{W.}}
(\byear{2024}e).
\btitle{{JoinGym}: {A}n Efficient Join Order Selection Environment}.
\bjournal{Reinforcement Learning Journal}
\bvolume{1}.
\end{barticle}
\endbibitem

\bibitem[\protect\citeauthoryear{Watkins and Dayan}{1992}]{watkins1992q}
\begin{barticle}[author]
\bauthor{\bsnm{Watkins},~\bfnm{Christopher~JCH}\binits{C.~J.}} \AND \bauthor{\bsnm{Dayan},~\bfnm{Peter}\binits{P.}}
(\byear{1992}).
\btitle{Q-learning}.
\bjournal{Machine learning}
\bvolume{8}
\bpages{279--292}.
\end{barticle}
\endbibitem

\bibitem[\protect\citeauthoryear{Weisz, Amortila and Szepesv{\'a}ri}{2021}]{weisz2021exponential}
\begin{binproceedings}[author]
\bauthor{\bsnm{Weisz},~\bfnm{Gell{\'e}rt}\binits{G.}}, \bauthor{\bsnm{Amortila},~\bfnm{Philip}\binits{P.}} \AND \bauthor{\bsnm{Szepesv{\'a}ri},~\bfnm{Csaba}\binits{C.}}
(\byear{2021}).
\btitle{Exponential lower bounds for planning in mdps with linearly-realizable optimal action-value functions}.
In \bbooktitle{Algorithmic Learning Theory}
\bpages{1237--1264}.
\bpublisher{PMLR}.
\end{binproceedings}
\endbibitem

\bibitem[\protect\citeauthoryear{Wu, Uehara and Sun}{2023}]{wu2023distributional}
\begin{binproceedings}[author]
\bauthor{\bsnm{Wu},~\bfnm{Runzhe}\binits{R.}}, \bauthor{\bsnm{Uehara},~\bfnm{Masatoshi}\binits{M.}} \AND \bauthor{\bsnm{Sun},~\bfnm{Wen}\binits{W.}}
(\byear{2023}).
\btitle{Distributional offline policy evaluation with predictive error guarantees}.
In \bbooktitle{International Conference on Machine Learning}
\bpages{37685--37712}.
\bpublisher{PMLR}.
\end{binproceedings}
\endbibitem

\bibitem[\protect\citeauthoryear{Xie et~al.}{2021}]{xie2021bellman}
\begin{barticle}[author]
\bauthor{\bsnm{Xie},~\bfnm{Tengyang}\binits{T.}}, \bauthor{\bsnm{Cheng},~\bfnm{Ching-An}\binits{C.-A.}}, \bauthor{\bsnm{Jiang},~\bfnm{Nan}\binits{N.}}, \bauthor{\bsnm{Mineiro},~\bfnm{Paul}\binits{P.}} \AND \bauthor{\bsnm{Agarwal},~\bfnm{Alekh}\binits{A.}}
(\byear{2021}).
\btitle{Bellman-consistent pessimism for offline reinforcement learning}.
\bjournal{Advances in neural information processing systems}
\bvolume{34}
\bpages{6683--6694}.
\end{barticle}
\endbibitem

\bibitem[\protect\citeauthoryear{Xie et~al.}{2023}]{xie2023the}
\begin{binproceedings}[author]
\bauthor{\bsnm{Xie},~\bfnm{Tengyang}\binits{T.}}, \bauthor{\bsnm{Foster},~\bfnm{Dylan~J}\binits{D.~J.}}, \bauthor{\bsnm{Bai},~\bfnm{Yu}\binits{Y.}}, \bauthor{\bsnm{Jiang},~\bfnm{Nan}\binits{N.}} \AND \bauthor{\bsnm{Kakade},~\bfnm{Sham~M.}\binits{S.~M.}}
(\byear{2023}).
\btitle{The Role of Coverage in Online Reinforcement Learning}.
In \bbooktitle{The Eleventh International Conference on Learning Representations}.
\end{binproceedings}
\endbibitem

\bibitem[\protect\citeauthoryear{Zhang}{2023}]{zhang_2023_ltbook}
\begin{bbook}[author]
\bauthor{\bsnm{Zhang},~\bfnm{Tong}\binits{T.}}
(\byear{2023}).
\btitle{Mathematical Analysis of Machine Learning Algorithms}.
\bpublisher{Cambridge University Press}.
\bdoi{10.1017/9781009093057}
\end{bbook}
\endbibitem

\bibitem[\protect\citeauthoryear{Zhang et~al.}{2023}]{zhang2023on}
\begin{binproceedings}[author]
\bauthor{\bsnm{Zhang},~\bfnm{Shuai}\binits{S.}}, \bauthor{\bsnm{Li},~\bfnm{Hongkang}\binits{H.}}, \bauthor{\bsnm{Wang},~\bfnm{Meng}\binits{M.}}, \bauthor{\bsnm{Liu},~\bfnm{Miao}\binits{M.}}, \bauthor{\bsnm{Chen},~\bfnm{Pin-Yu}\binits{P.-Y.}}, \bauthor{\bsnm{Lu},~\bfnm{Songtao}\binits{S.}}, \bauthor{\bsnm{Liu},~\bfnm{Sijia}\binits{S.}}, \bauthor{\bsnm{Murugesan},~\bfnm{Keerthiram}\binits{K.}} \AND \bauthor{\bsnm{Chaudhury},~\bfnm{Subhajit}\binits{S.}}
(\byear{2023}).
\btitle{On the Convergence and Sample Complexity Analysis of Deep Q-Networks with $\epsilon$-Greedy Exploration}.
In \bbooktitle{Thirty-seventh Conference on Neural Information Processing Systems}.
\end{binproceedings}
\endbibitem

\bibitem[\protect\citeauthoryear{Zhou et~al.}{2025}]{zhou2025q}
\begin{barticle}[author]
\bauthor{\bsnm{Zhou},~\bfnm{Jin~Peng}\binits{J.~P.}}, \bauthor{\bsnm{Wang},~\bfnm{Kaiwen}\binits{K.}}, \bauthor{\bsnm{Chang},~\bfnm{Jonathan}\binits{J.}}, \bauthor{\bsnm{Gao},~\bfnm{Zhaolin}\binits{Z.}}, \bauthor{\bsnm{Kallus},~\bfnm{Nathan}\binits{N.}}, \bauthor{\bsnm{Weinberger},~\bfnm{Kilian~Q}\binits{K.~Q.}}, \bauthor{\bsnm{Brantley},~\bfnm{Kiant{\'e}}\binits{K.}} \AND \bauthor{\bsnm{Sun},~\bfnm{Wen}\binits{W.}}
(\byear{2025}).
\btitle{$Q\sharp$: Provably Optimal Distributional RL for LLM Post-Training}.
\bjournal{arXiv preprint arXiv:2502.20548}.
\end{barticle}
\endbibitem

\end{thebibliography}

\clearpage
\edit{
\section{Appendix}
\subsection{Proof of Lemmas in \cref{sec:online-rl-squared-loss}}
\SqRLOptimismLemma*
\begin{proof}[Proof of \cref{lem:sq-rl-optimism}]
By standard martingale concentration via Freedman's inequality, \wpal $1-\delta$, for all $f,h$, we have
\begin{align}
    \textstyle\sum_{i=1}^n\Ecal^{\op{sq}}_h(f,\pi^i)&\textstyle\leq\ln(H|\Fcal|/\delta) + L_h^{\op{sq}}(f_h,f_{h+1},\Dcal_h)\notag
    \\&\textstyle-L_h^{\op{sq}}(\Tcal_h f_{h+1},f_{h+1},\Dcal_h).\label{eq:sq-rl-optimism-concentrate}
\end{align}
Let $g^f_h\in\argmin_{g_h\in\Gcal_h}L_h^\star(g_h,f_{h+1},\Dcal_h)$ denote the empirical risk minimizer, as used in the definition of $\Ccal_\beta^\star(\Dcal)$ (\cref{eq:rl-confidence-set}).
Under the BC premise,
\begin{align*}
\textstyle\sum_{i=1}^n\Ecal^{\op{sq}}_h(f,\pi^i)&\textstyle\leq\ln(H|\Fcal|/\delta)+L_h^{\op{sq}}(f_h,f_{h+1},\Dcal_h)
    \\&\textstyle-L_h^{\op{sq}}(g^f_h,f_{h+1},\Dcal_h).
\end{align*}
Thus, any $f\in\Ccal_\beta^{\op{sq}}(\Dcal)$ satisfies $\sum_{i=1}^n\Ecal^{\op{sq}}_h(f,\pi^i)\leq2\beta$, which proves Claim (a).
For Claim (b), we prove that $Q^\star\in\Ccal_\beta^{\op{sq}}(\Dcal)$: by \cref{eq:sq-rl-optimism-concentrate} and non-negativity of $\Ecal^{\op{sq}}$, we have $L_h^{\op{sq}}(\Tcal f_{h+1},f_{h+1},\Dcal_h)-L_h^{\op{sq}}(g^f_h,f_{h+1},\Dcal_h)\leq\ln(H|\Fcal|/\delta)=\beta$.
Then, setting $f=Q^\star$ and applying $Q^\star_h=\Tcal^\star_h Q^\star_{h+1}$ shows that $Q^\star$ satisfies the version space condition. Thus, $Q^\star\in\Ccal_\beta^{\op{sq}}(\Dcal)$ and Claim (b) follows by definition of $\hat f^{\op{op}}$.
\end{proof}
}

\edit{
The following is a proof for a stronger version of the pigeonhole lemma (\cref{lem:pigeonhole}). In partiular, \cref{lem:pigeonhole} follows when $\eps_0=\frac{1}{N}$.
\begin{lemma}\label{lem:pigeonhole-stronger}
Let $E:=\sup_{\nu\in\Mcal,\psi\in\Psi}\abs{\EE_p\psi}$.
Fix any $N\in\NN$, $\psi^{(1)},\dots,\psi^{(N)}\in\Psi$, and \edit{$\nu^{(1)},\dots,\nu^{(N)}\in\Mcal$}.
Let $\beta$ be a constant s.t. $\sum_{i<j}|\EE_{\nu^{(i)}}[\psi^{(j)}]|^q\leq \beta^q$ for all $j\in[N]$.
Then,
$\sum_{j=1}^N\abs*{\EE_{\nu^{(j)}}\psi^{(j)}}\leq\inf_{\eps_0\in(0,1)}\{N\eps_0+\EluDim_q(\Psi,\Mcal,\eps_0)\cdot(2E+\beta^q\ln(E\eps_0^{-1}))\}$.
\end{lemma}
\begin{proof}[Proof of \cref{lem:pigeonhole-stronger}]
Fix any $q\in\NN$; the proof will be for the $\ell_q$ eluder dimension.
We say a distribution $\nu\in\Mcal$ is $\eps$-independent of a subset $\Gamma\subset\Mcal$ if there exists $\psi\in\Psi$ s.t. $\abs{\EE_\nu\psi}>\eps$ but also $\sum_{\mu\in\Gamma}^n(\EE_{\mu}[\psi])^q\leq \eps^q$.
Conversely, we say $\nu$ is $\eps$-dependent on $\Gamma$ if for all $\psi\in\Psi$, we have $\abs{\EE_\nu[\psi]}\leq\eps$ or $\sum_{\mu\in\Gamma}^n(\EE_{\mu}[\psi])^q>\eps^q$.
For any $\Gamma\subset\Mcal$ and $\nu\in\Mcal$, we let $N(\nu,\Gamma,\eps_0)$ denote the largest number of disjoint subsets of $\Gamma$ that $\nu$ is $\eps_0$-dependent on.
We also use the shorthand $\mu^{(<j)}=\{\mu^{(1)},\mu^{(2)},\dots,\mu^{(j-1)}\}$.

\textbf{Claim 1: If $\abs*{\EE_{\mu^{(j)}}[\psi^{(j)}]}>\eps$, then $N(\mu^{(j)},\mu^{(<j)},\eps)\leq\beta^q\eps^{-1}$.}
By definition of $N:=N(\mu^{(j)},\mu^{(<j)},\eps)$, there are disjoint subsets $S^{(1)},\dots,S^{(N)}\subset \mu^{(<j)}$ s.t. each $S^{(i)}$ satisfies $\sum_{\mu\in S^{(i)}}\abs*{\EE_\mu[\psi^{(j)}]}>\eps$ since $\abs*{\EE_{\mu^{(j)}}[\psi^{(j)}]}>\eps$ by premise.
Thus, summing over all such subsets gives $N\eps<\sum_{i<j}\abs*{\EE_{\mu^{(i)}}[\psi^{(j)}]}^q\leq\beta^q$, proving Claim 1.

\textbf{Claim 2 (Pigeonhole): For any $\eps_0$ and any sequence $\mu^{(1)},\dots,\mu^{(N)}\in\Mcal$, there exists $j\leq N$ such that $N(\mu^{(j)},\mu^{(<j)},\eps_0)\geq \floor*{ \frac{(N-1)}{\EluDim_q(\Psi,\Mcal,\eps_0)} }$.}
Recall that if $\mu^{(1)},\dots,\mu^{(L)}\subset\Mcal$ satisfies for all $j\in[L]$, $\mu^{(j)}$ is $\eps_0$-independent of $p^{(<j)}$, then $L\leq\EluDim_q(\Psi,\Mcal,\eps_0)$ by definition.
To prove the claim, we maintain $J:=\floor*{ \frac{(k-1)}{\EluDim_q(\Psi,\Mcal,\eps_0)} }$ disjoint sequences $S^{(1)},\dots,S^{(J)}\subset \mu^{(<k)}$ s.t. each $S^{(i)}$ has the property that each element is $\eps$-independent of its precedessors.
We initialize $S^{(1)}=\cdots=S^{(J)}=\emptyset$ and iteratively add elements $\mu^{(1)},\dots,\mu^{(N)}$ until $\mu^{(j)}$ is $\eps_0$-dependent on all these disjoint subsequences, at which point the claim is proven. If there exists a subsequence which $\mu^{(j)}$ is $\eps_0$-independent of, we add $\mu^{(j)}$ to that subsequence, which preserves the invariant condition.
This process indeed terminates since otherwise one subsequence would have more elements than $\EluDim_q(\Psi,\Mcal,\eps_0)$, a contradiction.

\textbf{Claim 3: For any $\eps$, $\sum_{j=1}^N\II\bracks*{\abs*{\EE_{\mu^{(j)}}[\psi^{(j)}]}>\eps}\leq(\beta^q\eps^{-1}+1)\EluDim_q(\Psi,\Mcal,\eps)+1$.}
Let $\kappa$ denote the left hand sum and so let $i_1,\dots,i_\kappa$ be all the indices $j$ s.t. $\abs*{\EE_{\mu^{(j)}}[\psi^{(j)}]}>\eps$.
By Claim 2, there exists $j\leq\kappa$ s.t. $\floor*{ \frac{(\kappa-1)}{\EluDim_q(\Psi,\Mcal,\eps)} }\leq L(\mu^{(i_j)}, \mu^{(<i_j)},\eps)$.
Then by Claim 1, this is further upper bounded by $\beta^q\eps^{-1}$. Rearranging proves the claim.

\textbf{Concluding the proof.}
For any $\eps_0$, we have
\begin{align*}
    &\textstyle\sum_{j=1}^N\abs*{\EE_{\mu^{(j)}}[\psi^{(j)}]}
    =\textstyle\sum_{j=1}^N\int_0^E\II\bracks*{\abs*{\EE_{\mu^{(j)}}[\psi^{(j)}]}>y}\diff y
    \\&\leq\textstyle N\eps_0 +\sum_{j=1}^N\int_{\eps_0}^E\II\bracks*{\abs*{\EE_{\mu^{(j)}}\psi^{(j)}}>y}\diff y
    \\&\overset{(i)}{\leq}\textstyle N\eps_0 + \int_{\eps_0}^E\{ (\beta^qy^{-1}+1)\EluDim_q(\Psi,\Mcal,y)+1 \}\diff y
    \\&\overset{(ii)}{\leq}\textstyle N\eps_0 + \int_{\eps_0}^E\{ (\beta^qy^{-1}+1)\EluDim_q(\Psi,\Mcal,\eps_0)+1 \}\diff y
    \\&\overset{(iii)}{\leq}\textstyle N\eps_0 + \EluDim_q(\Psi,\Mcal,\eps_0)(2E+\beta^q\ln(E\eps_0^{-1})),
\end{align*}
where (i) is by Claim 3, (ii) is by monotonicity of the eluder dimension, and (iii) is by $\int_{\eps_0}^Ey^{-1}=\ln(E\eps_0^{-1})$.
\end{proof}

}

\edit{
 
\subsection{Proofs for Lemmas in \cref{sec:first-second-bounds-online-rl}}
\BCERLOptimism*
\begin{proof}[Proof of \cref{lem:bce-rl-optimism}]
By \cref{lem:logsumexp-symmetrization} extended on martingale sequences \citep{agarwal2020flambe},
\wpal $1-\delta$, for all $f\in\Fcal$, $h\in[H]$,
\begin{align}
    \textstyle\sum_{i=1}^n\Ecal^{\op{Ber}}_h(f,\pi^i)&\textstyle\leq\ln(H|\Fcal|/\delta) +\frac12 L_h^{\op{bce}}(f_h,f_{h+1},\Dcal_h)\notag
    \\&\textstyle-\frac12 L_h^{\op{bce}}(\Tcal_h^\star f_{h+1},f_{h+1},\Dcal_h).\label{eq:bce-rl-optimism-concentrate}
\end{align}
Let $g^f_h:=\argmin_{g_h\in\Gcal_h}L_h(g_h,f_{h+1},\Dcal_h)$ denote the empirical risk minimizer. Under the BC premise,
\begin{align*}
    \textstyle\sum_{i=1}^n\Ecal^{\op{Ber}}_h(f,\pi^i)&\textstyle\leq\ln(H|\Fcal|/\delta)+\frac12 L_h^{\op{bce}}(f_h,f_{h+1},\Dcal_h)
    \\&\textstyle-\frac12 L_h^{\op{bce}}(g^f_h,f_{h+1},\Dcal_h).
\end{align*}
Thus, any $f\in\Ccal_\beta(\Dcal)$ satisfies $\sum_{i=1}^n\Ecal^{\op{Ber}}_h(f,\pi^i)\leq \frac12\beta+\ln(H|\Fcal|/\delta)\leq2\beta$, which proves Claim (a).
For Claim (b), we prove that $Q^\star\in\Ccal_\beta(\Dcal)$.
By \cref{eq:bce-rl-optimism-concentrate} and non-negativity of $\Ecal^{\op{Ber}}$, we have $L_h^{\op{bce}}(\Tcal_h^\star f_{h+1},f_{h+1},\Dcal_h)-L_h^{\op{bce}}(g_h^f,f_{h+1},\Dcal_h)\leq 2\ln(H|\Fcal|/\delta)=\beta$.
Then, setting $f=Q^\star$ and noting that $Q_h^\star=\Tcal_h^\star Q^\star_{h+1}$ shows that $Q^\star$ satisfies the confidence set condition. Thus, $Q^\star\in\Ccal_\beta(\Dcal)$ and Claim (b) follows by definition of $\hat f^{\op{op}}$.
\end{proof}

\SelfBoundingBCELemma*
\begin{proof}[Proof of \cref{lem:self-bounding-bce}]
Fix any $f,\pi$. We use the shorthand $\delta_t(x,a)=\delta_t^{\op{Ber}}(f,\pi,x,a)$ to simplify notation.
The corollary follows from the main claim via $\EE_\pi[Q^\pi_h(x_h,a_h)]\leq V^\pi$, since costs are non-negative.
To prove the main claim, we establish the following claim by induction:
\begin{align}
    \textstyle f_h(x_h,a_h)\leq\sum_{t=h}^H&\textstyle(1+\frac1H)^{t-h}\EE_\pi[c_t\,+ \notag
    \\&\textstyle28H\delta_t(x_t,a_t)\mid x_h,\edit{a_h} ]. \label{eq:bce-self-bounding-induction}
\end{align}
The base case of $h=H+1$ holds since $f_{H+1}=0$.
For the induction step, fix any $h\in[H]$ and suppose that \cref{eq:bce-self-bounding-induction} is true for $h+1$.
By \cref{eq: first order} and AM-GM, we have
\begin{align*}
    \textstyle f_h(x_h,a_h)&\leq \textstyle(1+\frac{1}{H})\Tcal^\pi_h f_{h+1}(x_h,a_h)+28H\delta_h(x_h,a_h)
\end{align*}
By definition, $\Tcal^\pi_h f_{h+1}(x_h,a_h)=\EE_{\pi}[c_h + f_{h+1}(x_{h+1},a_{h+1})\mid x_h,a_h]$, so we can apply induction hypothesis to $f_{h+1}$.
This proves the inductive claim \cref{eq:bce-self-bounding-induction}.
Then, we prove the main claim by using the fact $(1+\frac1H)^H\leq e$.
The corollary then follows by $\EE_\pi[Q^\pi_h(x_h,a_h)]\leq V^\pi$ which holds due to the non-negativity of costs.
\end{proof}

\ImplicitFirstOrderIneq*
\begin{proof}[Proof of \cref{lem:implicit-first-order-ineq}]
By AM-GM, the premise implies $\sum_{k=1}^K (V^{\pi^k}-V^\star)\leq \frac12\sum_{k=1}^K V^{\pi^k} + \frac{3c^2}{2}$, which simplies to $\sum_{k=1}^K V^{\pi^k}\leq 2KV^\star + 3c^2$.
Hence, plugging this back into the premise yields the desired bound.
\end{proof}

\MLERLOptimism*
\begin{proof}[Proof of \cref{lem:mle-rl-optimism}]
The proof follows similarly as the bce case of \cref{lem:bce-rl-optimism}.
By \cref{lem:logsumexp-symmetrization} extended on martingale sequences \citep{agarwal2020flambe}, \wpal $1-\delta$, for all $p\in\Pcal$, $h\in[H]$,
\begin{align}
    \textstyle\sum_{i=1}^n\Ecal^{\op{dis}}_h(p,\pi^i)&\textstyle\leq\ln(H|\Pcal|/\delta) + \frac12 L_h^{\op{mle}}(p_h,p_{h+1},\Dcal_h)\notag
    \\&\textstyle-\frac12 L_h^{\op{mle}}(\Tcal_h^{\op{D},\star}p_{h+1},p_{h+1},\Dcal_h). \label{eq:dist-rl-optimism-concentrate}
\end{align}
Let $g^p_h:=\argmin_{g_h\in\Pcal_h}L_h^{\op{mle}}(g_h,p_{h+1},\Dcal_h)$ denote the empirical maximum likelihood estimate. Under the distributional BC premise, we have
\begin{align*}
    &\textstyle\sum_{i=1}^n\Ecal^{\op{dis}}_h(p,\pi^i)\leq\ln(H|\Pcal|/\delta)
    \\&\textstyle+\frac12 L_h^{\op{mle}}(p_h,p_{h+1},\Dcal_h)-\frac12 L_h^{\op{mle}}(g^p_h,p_{h+1},\Dcal_h).
\end{align*}
Thus, any $p\in\Ccal_{\beta}^{\op{mle}}(\Dcal)$ satisfies $\sum_{i=1}^n\Ecal^{\op{dis}}_h(p,\pi^i)\leq \frac12\beta+\ln(H|\Pcal|/\delta)\leq2\beta$, which proves Claim (a).
For Claim (b), we prove that $Z^\star\in\Ccal_\beta^{\op{mle}}(\Dcal)$.
By \cref{eq:dist-rl-optimism-concentrate} and non-negativity of $\Ecal^{\op{dis}}$, we have $L_h^{\op{mle}}(\Tcal_h^{\op{D},\star}p_{h+1},p_{h+1},\Dcal_h)-L_h^{\op{mle}}(g_h^p,p_{h+1},\Dcal_h)\leq 2\ln(H|\Pcal|/\delta)=\beta$.
Then, setting $p=Z^\star$ and noting that $Z_h^\star=\Tcal_h^{\op{D},\star} Z^\star_{h+1}$ shows that $Z^\star$ satisfies the confidence set condition. Thus, $Z^\star\in\Ccal_\beta^{\op{mle}}(\Dcal)$ and Claim (b) follows by definition of $\hat p^{\op{op}}$.
\end{proof}

\VarianceChangeOfMeasure*
\begin{proof}[Proof of \cref{lem:variance-change-of-measure}]
Fix any $p,\pi$. We use the shorthand $\delta_t(x,a)=\delta_t^{\op{dis}}(p,\pi,x,a)$ to simplify notation.
First, note that the corollary follows from the main claim since the law of total variance (LTV) implies $\EE[\sigma^2(Z^\pi_h(x_h,a_h))]\leq\sigma^2(Z^\pi)$, where recall the LTV states:
for any random variable $X,Y$:
\begin{align*}
    \textstyle\sigma^2(Y) = \EE[\sigma^2(Y\mid X)]+\sigma^2(\EE[Y\mid X]).
\end{align*}
We now establish the main claim.\\
\textbf{Step 1.} We first show the following claim by induction: for all $h$,
\begin{align}
    \textstyle\sigma^2(p_h(x_h,a_h))\leq&\textstyle\sum_{t=h}^H(1+\frac1H)^{t-h}\EE_\pi[ \,8H\delta_t(x_t,a_t) \notag
    \\&\textstyle 2\sigma^2(c_t+\bar p_{t+1}(x_{t+1},\pi(x_{t+1}))) \mid x_h,a_h ] \label{eq:dist-variance-change-measure-induction}
\end{align}
The base case $h=H+1$ is true since $\sigma^2(p_{H+1})=0$.
For the induction step, fix any $h\in[H]$ and suppose that the induction hypothesis (IH; \cref{eq:dist-variance-change-measure-induction}) is true for $h+1$.

By our second-order lemma for variance (\cref{eq:dtri-variance-inequality}),
\begin{align*}
\textstyle\sigma^2(p_h(x_h,a_h))\leq&\textstyle(1+\frac1H)\sigma^2(\Tcal^{\op{D},\pi}_hp_{h+1}(x_h,a_h))
    \\&+8H\delta_h(x_h,a_h).
\end{align*}
Then, we use LTV to condition on $c_h,x_{h+1}$ (\ie, the outer mean/variance are w.r.t. $c_h,x_{h+1}$, the inner mean/variance are w.r.t. $p_{h+1}$): $\sigma^2(\Tcal^{\op{D},\pi}_hp_{h+1}(x_h,a_h))$ is equal to
\begin{align*}
    &\textstyle\EE[ \sigma^2(p_{h+1}(x_{h+1},\pi(x_{h+1}))\mid c_h, x_{h+1}) ] \notag
    \\&\qquad\textstyle+\sigma^2(c_h+\bar p_{h+1}(x_{h+1},\pi(x_{h+1}))).
\end{align*}
We bound the first term by the IH, which completes the proof for \cref{eq:dist-variance-change-measure-induction}.

\noindent\textbf{Step 2.} By the above claim and $(1+\frac1H)^H\leq e$, we have
\begin{align*}
    \textstyle\sigma^2&\textstyle(p_h(x_h,a_h))\leq 8H\delta_{\op{dis}}^{\op{RL}}(p,\pi) +
    \\&\textstyle 2e\sum_{t=h}^H\EE_\pi[\sigma^2(c_t+\bar p_{t+1}(x_{t+1},\pi(x_{t+1})))\mid x_h,a_h].
\end{align*}

\noindent\textbf{Step 3.} Lastly, it suffices to convert the above variance term to $\sigma^2(c_t+V^\pi_{t+1}(x_{t+1}))$, since $\sigma^2(Z^\pi_h(x_h,a_h))=\sum_{t=h}^H\EE_\pi[\sigma^2(c_t+V^\pi_{t+1}(x_{t+1}))\mid x_h,a_h]$ by LTV.
To perform this switch in variance, observe that:
\begin{equation}
    \textstyle\abs{\bar p_h(x_h,\pi(x_h))-V^\pi(x_h)}\lesssim\sum_{t=h}^H\EE_\pi[ \sqrt{\delta_t(x_t,a_t)} ],\label{eq:crude-variance-switch-in-variance-change-proof}
\end{equation}
by the PDL and the second-order lemma (\cref{lemma: second order}). Also, recall that $\sigma^2(X)\leq2\sigma^2(Y)+2\sigma^2(X-Y)$. Thus, we have
\begin{align*}
    &\textstyle\sigma^2(c_t+\bar p_{t+1}(x_{t+1},\pi(x_{t+1})))
    \\&\textstyle\leq 2\sigma^2(c_t+V^\pi_{t+1}(x_{t+1}))
    \\&\textstyle+2\sigma^2(\bar p_{t+1}(x_{t+1},\pi(x_{t+1}))-V^\pi_{t+1}(x_{t+1}))
    \\&\textstyle\leq 2\sigma^2(c_t+V^\pi_{t+1}(x_{t+1})) + H\sum_{t=h}^H\EE_\pi[ \delta_t(x_t,a_t) ],
\end{align*}
where the last inequality used \cref{eq:crude-variance-switch-in-variance-change-proof} and Cauchy-Schwarz.
Thus we have shown that
\begin{align*}
    \textstyle\sigma^2&\textstyle(p_h(x_h,a_h))\lesssim H^2\delta_{\op{dis}}^{\op{RL}}(p,\pi) +
    \\&\textstyle 4e\sum_{t=h}^H\EE_\pi[\sigma^2(c_t+\bar p_{t+1}(x_{t+1},\pi(x_{t+1})))\mid x_h,a_h].
\end{align*}
\end{proof}

}

\end{document}